\documentclass[11pt]{article}
\usepackage[a4paper,margin=1in]{geometry}
\usepackage{times}

\pdfoutput=1                
\usepackage[utf8]{inputenc} 

\usepackage{bbm} 
\usepackage[numbers,sort&compress]{natbib}

\usepackage{amsmath, amstext, amssymb, amsthm, mathtools}

\usepackage{bm}

\usepackage{graphicx}
\usepackage{caption}
\usepackage{subcaption}

\usepackage{authblk} 

\usepackage[ruled,vlined,linesnumbered]{algorithm2e}
\SetAlCapNameFnt{\small}\SetAlCapFnt{\small}    
\SetAlgoCaptionLayout{centerline}               
\SetKwInput{KwIn}{Input}\SetKwInput{KwOut}{Output}
\SetKw{KwRet}{return}
\SetAlgoInsideSkip{smallskip}                    

\usepackage{enumitem}
\usepackage{wrapfig}
\usepackage{booktabs}
\usepackage{multirow}
\usepackage{multicol}
\usepackage{floatrow}
\usepackage{placeins}   
\usepackage{stfloats}   

\usepackage{xcolor}
\usepackage[colorlinks=true, linkcolor=blue, citecolor=blue, urlcolor=blue]{hyperref}

\newcommand{\starsup}{\textsuperscript{*}}
\newcommand{\daggersup}{\textsuperscript{\dag}}

\theoremstyle{plain}
\newtheorem{assumption}{Assumption}
\newtheorem{theorem}{Theorem}
\newtheorem{lemma}{Lemma}
\newtheorem{corollary}{Corollary}
\newtheorem{remark}{Remark}

\theoremstyle{definition} 

\title{Grounding the Ungrounded: A Spectral-Graph Framework for Quantifying Hallucinations in Multimodal LLMs}

\author[1]{Supratik Sarkar\starsup}
\author[2]{Swagatam Das\daggersup}
\affil[1]{Morgan Stanley}
\affil[2]{Indian Statistical Institute (Kolkata), India}

\date{\today}

\usepackage{bibentry} 

\begin{document}

\maketitle

\begingroup
\setcounter{footnote}{0}%
\renewcommand{\thefootnote}{\fnsymbol{footnote}}%
\footnotetext[1]{\href{mailto:supratik.sarkar@morganstanley.com}{supratik.sarkar@morganstanley.com}}
\footnotetext[2]{\href{mailto:swagatam.das@isical.ac.in}{swagatam.das@isical.ac.in}}
\endgroup

\begin{abstract}
Hallucinations in LLMs—especially in multimodal settings—undermine reliability. We present a rigorous information-geometric framework, grounded in diffusion dynamics, to quantify hallucinations in MLLMs where model outputs are embedded via spectral decompositions of multimodal graph Laplacians, and their gaps to a truth manifold define a semantic distortion metric. We derive Courant–Fischer bounds on a temperature-dependent hallucination profile and use RKHS eigenmodes to obtain modality-aware, interpretable measures that track evolution over prompts and time. This reframes hallucination as quantifiable and bounded, providing a principled basis for evaluation and mitigation.
\end{abstract}

\section{Introduction}

Large language models (LLMs) and their multimodal variants (MLLMs) are powerful generators, but reliability or truthfulness remains a core limitation. A central drawback is the hallucinated content that is ungrounded or inconsistent with inputs - which is unacceptable and signifactly risky in medicine, law, and finance~\cite{Ji2023Survey,Maynez2020Faithfulness,Bubeck2023Sparks}. Prior work offers taxonomies, datasets, and benchmarks for analysis and evaluation~\cite{Ji2023Survey,Maynez2020Faithfulness,Ding2024HalluPI}, and recent multimodal studies emphasize empirical detection/mitigation~\cite{Bai2024MultimodalHallucinations}; however, most approaches rely on heuristics, proxy metrics, or human annotation rather than principled quantification.

On the theory side, complementary work include token-level analysis of hallucinated predictions~\cite{Jiang2024KnownFacts}, Bayesian sequential detection~\cite{Wang2023Bayesian}, entropy-style uncertainty probes~\cite{Han2024SEP}, latent-space steering to separate truthful vs. hallucinated generations~\cite{Park2025TSV}, and reference-free ranking for multimodal hallucinations~\cite{Sun2024CrossCheckGPT}. Emerging spectral/graph perspectives probe representations and attention, but are largely detection-oriented and unimodal~\cite{binkowski2025hallucinationdetectionllmsusing}.

\paragraph{Gap.} \vspace{-1em}
The field currently lacks a quantitative, theory-backed, modality-aware framework that treats hallucination as a measurable quantity (with temporal dynamics and guarantees), rather than only a classification/detection outcome.

\paragraph{Our contribution.} \vspace{-1em}
At a high level, our framework provides a plug-in, reference-free hallucination controller for MLLM pipelines that remains meaningful even when the ground-truth labels are missing and, unlike other standard uncertainty proxies (entropy, max-probability, margin), decomposes hallucination into modality-wise and spectral components on a multimodal graph Laplacian. It provides a calibrated knob to rank outputs by hallucination risk, set ``IDK"/abstention thresholds, and track hallucination under time-indexed temperature and retrieval policies:
\vspace{-3pt}
\begin{enumerate} [leftmargin=*, label=(\alph*), itemsep=-2pt]
    \item We model the grounding across modalities via optimal-transport paths in diffusion dynamics and embed them in RKHS, yielding a structural view of semantic consistency.
    \item We represent outputs on multimodal graph Laplacians and derive tight Courant–Fischer (CF) bounds on hallucination heatmap as a function of time-indexed temperature.
    \item \emph{Empirical validation}: Across nine 3D panels (COCO/VQAv2/AudioCaps $\times$ \texttt{CLIP+Whisper+T5}, \texttt{BLIP+CLIP+Whisper}, \texttt{SigLIP+Whisper+T5}), $\mathcal{E}_{\mathrm{hall}}^{\mathrm{multi}}$ lies between panel-specific CF planes with a strictly positive lower envelope that tightens at lower temperature (and higher diffusion); full $\varepsilon/\tau/h/\rho$ ablations and runtimes in the supplement.
\end{enumerate}
This shifts hallucination study from qualitative detection to quantitative, modality-aware, and interpretable analysis. To our knowledge, it is the first attempt to provide spectral bounds on hallucination for MLLMs followed by a time-indexed temperature annealing, offering a principled basis for evaluation and potential mitigation. A clear mathematical roadmap is presented in Appendix~\ref{fn:roadmap}.


\section{Related Work}
\label{sec:other_work}

Kalai \& Vempala show that, for calibrated LMs, the hallucination rate is lower-bounded by a Good–Turing–style ``monofact” mass - establishing an inherent trade-off between calibration and truthfulness~\cite{Vempala2024Calibrated}; while their recent work generalizes this via an IIV reduction that ties generative errors to binary-classification - advocating IDK-tolerant evaluation~\cite{Vempala2025Why}. Empirical study of LM hallucinations spans mechanistic probes that surface interpretable features for diagnosis~\cite{Templeton2024ScalingMonosemanticity}, retrieval-grounded detection and evaluation~\cite{Gerner2025ORION,niu2024ragtruth}, broad benchmark suites like HaluEval~\cite{li2023halueval}, Hallu-PI~\cite{Ding2024HalluPI}, GraphEval~\cite{Feng2025GraphEval}, and early vision–language analyses of object hallucination~\cite{rohrbach2018chair}. Comprehensive surveys catalog causes, detection, and mitigation strategies~\cite{Ji2023Survey,Rawte2023Troubling}.

Recent work exploits uncertainty and structural signals: semantic-entropy probes~\cite{Han2024SEP}, Bayesian sequential estimation~\cite{Wang2023Bayesian}, token-level dynamics of hallucinated predictions~\cite{Jiang2024KnownFacts}, zero-shot reasoning signals~\cite{Lee2024ZeroShot}, and sampling-based self-consistency checks (SelfCheckGPT)~\cite{Weng2024External}. Graph/spectral methods flag hallucinations via KG self-checks (FactSelfCheck)~\cite{Sawczyn2025FactSelfCheck}, attention Laplacian eigen-spectra (LapEigvals)~\cite{Binkowski2025LapEigvals}, and topological cues on hallucination graphs~\cite{LeMerrer2024Graph}.

\section{Preliminaries}
\label{sec:prel}
We begin by establishing the mathematical foundations of our framework. MLLM outputs are embedded as nodes on a knowledge graph Laplacian, and grounding gaps along this graph collectively define a quantifiable hallucination metric. Figure~\ref{fig:framework} sketches our approach.

\subsection{Mathematical Foundations}
\label{subsec:math_found}
Let \( \mathcal{X} \) denote the measurable\footref{fn:measurable_set} \footnote{Footnotes are added in chronological order and collected in Appendix~\ref{appendix:notes}.} set of all possible model outputs of a multimodal LLM, with \(\mathcal{F}_{\mathcal{X}}\) being the \(\sigma\)-algebra over \(\mathcal{X}\) and \(\mu\) being the base measure~\cite{Tao2011Measure}; e.g., the count measure for discrete outputs like token sequence or the Lebesgue measure for continuous outputs like embeddings~\cite{Bartle1995Measure}. We assume \(\mathcal{X}\) is continuously embedded in a separable Reproducing Kernel Hilbert space (RKHS) denoted by \((\mathcal{H}, \left\langle \cdot,\cdot \right\rangle_{\mathcal{H}})\) which is associated with a positive-definite kernel,
\begin{equation}
    \label{eq:def_kernel}
    K:\,\mathcal{X}\times \mathcal{X} \to \mathbb{R}^{+}.
\end{equation}
The kernel \( K(x_{1}, x_{2}) \) encodes the semantic relationships between two distinct points or outputs \( x_{1} \) and \( x_{2}\) \(\forall (x_{1}\neq x_{2}) \in \mathcal{X}\); for example, through embedding-based or ontology-aware distance measures, or co-reference resolution. For a product kernel in an MLLM, refer to Eq.~\eqref{eq:def7} later.

Within this \((\mathcal{X}, \mathcal{F}_{\mathcal{X}}, \mu)\) space, there exist two kinds of ``truth" (the idea imported from~\cite{Vempala2024Calibrated}): 
\begin{enumerate}[leftmargin=*, label=(\roman*), itemsep=-1pt]
    \item The semantic factoid space \( \mathcal{K}\) which encompasses all semantically valid and coherent outputs that include empirically plausible facts, contextually appropriate completions, and domain-consistent inferences aligned with the prompt and background knowledge - importantly, elements of \( \mathcal{K}\) need not be verifiable, but they remain semantically valid within the modeled domain.
    \item The semantic ground-truth manifold \(\mathcal{K}_{\text{g}}\), as a stricter subregion of \(\mathcal{K}\), which consists of outputs only verifiably correct or true facts that include factual assertions supported by empirical evidence or directly observed information — elements of \(\mathcal{K}_{\text{g}}\) can be properly referred to as grounded in reality.
\end{enumerate}

\begin{wrapfigure}{r}{0.50\textwidth}  
  \vspace{-6pt}
  \centering
  \includegraphics[width=0.95\textwidth]{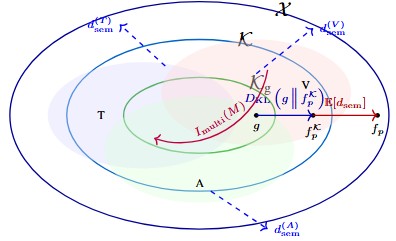}
  \vspace{-6pt}
  \caption{Multimodal nested-manifold view of hallucinations.
  Hollow ellipses denote \(\mathcal{X}\), \(\mathcal{K}\), \(\mathcal{K}_g\).
  }
  \label{fig:framework}
  \vspace{-6pt}
\end{wrapfigure} 
\label{eq:def1} 
Thus the semantic plausibility/ground-truth nesting and, for a given prompt \( p \in \mathcal{P} \), the hallucination criterion for each output denoted by \( x \in \mathcal{X} \) are:
\begin{equation}
    \mathcal{K}_{\text{g}} \subseteq \mathcal{K} \subset \mathcal{X}, \quad x \in \mathcal{X} \setminus \mathcal{K}.
\end{equation}
Note: \(x\in \mathcal{K} \setminus \mathcal{K}_{\text{g}}\)  is a non-grounded output, but still semantically plausible and strictly not hallucination.

\subsection{Modeling the LLM outputs}
\label{subsec:llm_outputs}

We begin with the baseline assumptions:
\begin{assumption}[General output distribution]
\label{assmp:general_dist}
The LLM outputs can be characterized by a conditional probability distribution \( f_p(x) \) that denotes the likelihood of generating output \( x \) given a prompt \( p \):
\begin{equation}
\label{eq:def3}
       f_{p} : \mathcal{X} \to [0,\infty), \quad f_{p} \in L^{1}(\mathcal{X}, \mathcal{F}_{\mathcal{X}}, \mu) \cap \mathcal{H}, \quad x \mapsto f_{p}(x), 
\end{equation}
which ensure \(\int_{\mathcal{X}}\, f_{p}(x) \, d\mu(x) \, =1\). (See justification\footref{fn:general_dist} in Appendix~\ref{appendix:notes}.)
\end{assumption}
Let \( f_p^{\mathcal{K}} \) denote the restricted distribution on the semantic plausibility space \( \mathcal{K} \):
\begin{equation}
\label{eq:def5}
    f_p^{\mathcal{K}}(x) 
    := \frac{\mathbf{1}_{\{ x \in \mathcal{K}\}} f_p(x)}{\int_{\mathcal{K}} f_p(x') d\mu(x')}\equiv \frac{\mathbf{1}_{\{x \in \mathcal{K}\}} f_p(x)}{\mathbb{P}_{f_{p}}(\mathcal{K})}, 
    \quad  \text{where, }   \mathbf{1}_{\{x\in \mathcal{K}\}} =
\begin{cases}
1 & \text{if } x\in \mathcal{K}, \\
0 & \text{otherwise.} 
\end{cases}
\end{equation}
Here, \(\int_{\mathcal{K}} f_p(x') d\mu(x') = \mathbb{P}_{f_{p}}(\mathcal{K})\) is a normalization constant in the restricted distribution.

\begin{assumption}[Ground-truth generative distribution]
\label{assmp:ground_dist}
In line with Assumption~\ref{assmp:general_dist}, \( g \) denotes the reference distribution on the ground-truth manifold \( \mathcal{K}_{\text{g}} \). Unlike \(f_{p}\) or \(f_p^{\mathcal{K}}\), \(g\) is the gold reference which is not model-induced and hence, may not share support with \(f_{p}\) except inside \( \mathcal{K}_{\text{g}} \) and it is truly  independent of prompts in the generative sense, but conditioned on the same prompt contextually. (See justification\footref{fn:ground_dist} in Appendix~\ref{appendix:notes}.)
\end{assumption}
Thus, we do not assume any parametric form for the ground-truth distribution \(g\) and rather treat it as an abstract measure over \(\mathcal{K}_{\text{g}}\):
\begin{equation}
\label{eq:def4}
       \text{supp}(g) \subseteq \mathcal{K}_{\text{g}},  \quad 
       g : \mathcal{K}_{\text{g}} \to [0,\infty), \quad g \in L^{1}(\mathcal{K}_{\text{g}}, \mathcal{F}_{\mathcal{X}}|_{\mathcal{K}_{\text{g}}}, \mu'). 
\end{equation}
Eq.~\eqref{eq:def4} ensures \(\int_{\mathcal{K}_{\text{g}}}\, g(x) \, d\mu'(x) \, =1\) with notations used in consistency with Eq.~\eqref{eq:def3} and \(\mu'\) playing the same role of \(\mu\), but not necessarily equal to \(\mu\). 


\section{Theoretical Analysis}
\label{sec:theo_analysis}

In this section, we present a theoretical framework that couples a smoothed information-geometric score derived from the Kullback–Leibler (KL) paradigm\footref{fn:KL1} with a multimodal energy formulation to quantify and track hallucinations in MLLMs.

\subsection{Semantic Distortion}
\label{subsec:sem_dist}

We establish the following theorem followed by stating remarks to set the stepping stone.

\begin{theorem}[KL-calibrated smoothed score for hallucination]
\label{thm:kl_decomposition_truth}
 Let a smoothing mass $\varepsilon\in(0,1)$ and a baseline density be fixed, with finite $\rho(x)>0$ $\mu$-a.e.\ and $\int_{\mathcal X}\rho(x)\,d\mu(x)=1$; let $K_h(\cdot,\cdot)\in (0, \infty)$ be a $\mu$-Markov kernel (bandwidth $h>0$) and $T_h:L^{1}(\mu)\to L^{1}(\mu)$ be a linear smoother defined for $q:\mathcal X\to\mathbb{R}$ by $(T_h q)(x_1):=\int_{\mathcal X}K_h(x_1,x_2)\,q(x_2)\,d\mu(x_2)$; let the $\varepsilon$-smoothed model be $\tilde f_{p,\varepsilon}(x):=(1-\varepsilon)f_p(x)+\varepsilon\rho(x)$ with its $\mathcal K$-restricted renormalization $\tilde f_{p,\varepsilon}^{\mathcal K}(x_2):=\mathbf{1}_{\{x_2\in\mathcal K\}}\tilde f_{p,\varepsilon}(x_2)\big/ \int_{\mathcal K}\tilde f_{p,\varepsilon}(x)\,d\mu(x)$; and let a measurable selector $\Pi_{\mathcal K}:\mathcal X\to\mathcal K$ satisfy $\Pi_{\mathcal K}(x)=x$ $(\forall x\in\mathcal K)$ or nearest point with convexity in $\mathcal{K}$ (otherwise). Then the semantic distortion
\begin{equation}
\label{eq:KL1}
d_{\mathrm{sem}}^{(\varepsilon,h)}(x;\mathcal K,\mathcal X)
:=\Big[\log\!\big((T_h \tilde f_{p,\varepsilon}^{\mathcal K})(\Pi_{\mathcal K}(x))\big)-\log\!\big((T_h \tilde f_{p,\varepsilon})(x)\big)\Big]_{+} ,
\end{equation}
 serves as a KL-calibrated smoothed pointwise information gap for tracking hallucinations across prompts and remains as a reference-free (independent-of-$g$) statistic in language models.
\end{theorem}

\textit{Proof sketch:} Strict positivity from $\tilde f_{p,\varepsilon}=(1-\varepsilon)f_p+\varepsilon\rho$ and Markov $K_h$ makes both smoothed terms $>0$, so Eq.~\eqref{eq:KL1} is finite. If $x\in\mathcal K$, $\Pi_{\mathcal K}(x)=x$ and the $\mathcal K$–restricted smoother $>$ the unconditional smoother at $x$; if $x\notin\mathcal K$, smoothing at $\Pi_{\mathcal K}(x)\in\mathcal K$ dominates the mixed mass at $x$. Detailed proof is found in Appendix~\ref{appendix:proof_kl_decomposition_truth}. \hfill$\Box$

\begin{remark}
\label{rem:sem_dist}
The score in Eq.~\eqref{eq:KL1} is $g$-agnostic and thus usable when $g$ is unobservable\footref{fn:absent_g} or partially verified in various real-world scenarios. In practice, we set a small smoothing mass $\varepsilon\in[10^{-6},10^{-2}]$, choose $h$ by validation, take $K_h$ as a positive row-normalized kernel over embeddings/tokens, and we implement $\Pi_{\mathcal K}$ as a measurable nearest–neighbour selector on a finite reference set from $\mathcal K$. To clarify how this mathematical framework connects to MLLM
pipelines in practice, we identify\footref{fn:obs_prac} what are ``observable", ``assumed" or ``estimated" in Appendix~\ref{appendix:notes}.
\end{remark}

\begin{remark}
\label{rem:cont_hall}
We deliberately work with a continuous hallucination score $\mathbbm{h}(x,p)\in[0,\infty)$, rather than a binary $0/1$ label, for several practical reasons; see
Appendix~\ref{fn:cont_hall} for a detailed discussion.
\end{remark}


\subsection{Extension to Multi-modal Grounding}
\label{subsec:multi_modal}

The intuition behind this setting of multimodality is: in image-grounded or dialogue models, semantic grounding depends on multiple modalities — e.g., text, image or video, dialog or audio-history etc. and the RKHS is then extended to a multi-modal product kernel space. In multi-modal settings, where the LLM outputs involve textual ($T$), visual ($V$), audio ($A$) modalities, we define a joint output space \((\mathcal{X})\) embedded into a composite RKHS \((\mathcal{H})\) equipped with a product kernel \((K)\) between two distinct points (i.e., outputs) \(\forall(x_{1}\neq x_{2})\in \mathcal{X}\) as
\begin{equation}
    \label{eq:def7}
    \mathcal{X}: \,\mathop{\times}\limits_{M} \mathcal{X}_{M}, \quad x=(x^{(M)})_{x^{(M)}\in \mathcal{X}_{M}},
\quad \mathcal{H}:= \mathop{\otimes}\limits_{M} \mathcal{H}_{M}, \quad K(x_{1}, x_{2}) = \prod_{M} \, K_{M}(x_{1}^{(M)}, x_{2}^{(M)} ),
\end{equation}
pertaining to each modality \(\forall M \in \mathcal{M}:=\{T,V,A\}\), where the prompts can also be categorized into a composite prompt space \(\mathcal{P}:\,\mathop{\times}\limits_{M}\mathcal{P}_{M},\) with each prompt \(p=(p^{(M)})_{p^{(M)}\in \mathcal{P}_{M}}\) in a modality-aware prescription to accommodate three different kinds of probable inputs (i.e., T, V \& A) for the sake of completeness. However, in the following calculation in this paper, we restrict ourselves only to the notion of \(p\) without any loss of generality. Expanded form\footref{fn:modal_expand} of Eq.~\eqref{eq:def7} is found in Appendix~\ref{appendix:notes}.

\subsection{Formulations to hallucination Energy}
\label{subsec:multi_hall_energy}
To begin with, we are after a fruitful formulation of \(f_{p}(x)\) that connects the model output distribution to an underlying energy landscape to enable modal interpretability, temperature-driven exploration, and spectral graph analysis. The total energy functional \(\mathcal{E}(x,p, \cdot):\ \mathcal{X}\times\mathcal{P}\to \mathbb{R}^{+}\) associated with the model input-output plus suppressed parameters can be decomposed into intra-modal, pairwise cross-modal, and joint multimodal interactions. This decomposition allows us to localize the sources of hallucination within and across modalities.

\begin{assumption}[Hallucination energy functional in MLLMs]
\label{assmp:hall_energy}
The modality-aware decomposition reads as:
\begin{equation}
\label{eq:energy_full}
\mathcal{E}(x,p, \cdot) 
\,\,=\,\, \sum_{M \in \mathcal{M}} \mathcal{E}_M\left(x^{(M)},p, \cdot\right) 
\,\,\, +\,\,\, \sum_{\substack{M, M^{'} \in \mathcal{M} \\ M \neq M^{'}}} \mathcal{E}_{MM^{'}}\left(x^{(M)}, x^{(M^{'})}, p, \cdot\right) 
\,\,+\,\, \mathcal{E}_{\mathcal{M}}(x,p, \cdot).
\end{equation}   
(See justification\footref{fn:hall_energy} in Appendix~\ref{appendix:notes} and Section~\ref{subsec:graph1} for the similar construction.)
\end{assumption}

\begin{assumption} [Feature maps for boundedness]
\label{assmp:feature_map}
    Using the results of Moore–Aronszajn theorem~\cite{Aronszajn1950RKHS}, for a positive definite kernel \(K_M\) in a measurable output space $(\mathcal{X},\mathcal{F}_{\mathcal{X}},\mu)$ aligned with Section~\ref{subsec:math_found}, let \(\Phi_M: \mathcal{X}_M \to \mathcal{H}_M\) be its feature map treated as infinite-dimensional linear operator for each modality \(M\in \mathcal{M}\) under the constraint of boundedness: $\sup_{x^{(M)}\in \mathcal{X}_M}\|\Phi_M(x^{(M)})\|_{\mathcal{H}_M}<\infty$. (See justification\footref{fn:feature_map} in Appendix~\ref{appendix:notes}.)
\end{assumption}
For each modality \(M\), the (fixed) embedding pipeline with an implicit kernel\footref{fn:feature_map} in a higher-dimensional RKHS induces \(\Phi_M:\mathcal{X}_M\!\to\!\mathcal{H}_M\) such that \(\langle \Phi_M(x_1),\Phi_M(x_2)\rangle_{\mathcal{H}_M}=K_M(x_1,x_2)\).

\begin{assumption} [Prompt embeddings]
\label{assmp:prompt_embedding}
Let $(\mathcal{P},\mathcal{F}_{\mathcal{P}},\nu)$ be a measurable space on prompts with $\nu$ being finite. For each modality $M\in\mathcal{M}$, the prompt embedding $\Psi_M:\mathcal{P}\!\to\!\mathcal{H}_M$ satisfies boundedness: $ \sup_{p\in\mathcal{P}}\|\Psi_M(p)\|_{\mathcal{H}_M}<\infty$ and stability: $\Psi_M$ is continuous (equivalently, Lipschitz with finite constant $\operatorname{Lip}(\Psi_M)$) in the chosen topology/ metric on $\mathcal{P}$. (See justification\footref{fn:prompt_embedding} in Appendix~\ref{appendix:notes}.)
\end{assumption}

\begin{assumption} [Output distribution in Boltzman form]
\label{assmp:output_boltzman}
We view $f_p(x)$ as a normalized surrogate over candidate outputs or latent representations with respect to a finite (or bounded) base measure $\mu$. Under bounded embeddings and compact support (or bounded energy), the partition function $Z(p,\mathcal T_t)$ is finite, making Eq.~\eqref{eq:boltzmann_form_partition} well-defined. (See justification\footref{fn:output_boltzman} in Appendix~\ref{appendix:notes}.)
\end{assumption}

\begin{lemma}[Joint measurability of cross inner products]
\label{lm:joint_measure}
If $\Phi_M:(\mathcal{X}_M,\mathcal{F}_{\mathcal{X}_M})\to(\mathcal{H}_M,\mathcal{B}(\mathcal{H}_M))$ and
$\Psi_M:(\mathcal{P},\mathcal{F}_{\mathcal{P}})\to(\mathcal{H}_M,\mathcal{B}(\mathcal{H}_M))$ are Bochner measurable into a separable Hilbert space $\mathcal{H}_M$ where $\mathcal{B}(\mathcal{H}_M)$ denotes the Borel $\sigma-$algebra generated by the open sets of $\mathcal{H}_M$ under its norm topology,
then $(x,p)\mapsto \langle \Phi_M(x),\Psi_M(p)\rangle_{\mathcal{H}_M}$ is measurable on 
$\mathcal{F}_{\mathcal{X}_M}\otimes\mathcal{F}_{\mathcal{P}}$.
\end{lemma}
\textit{Proof sketch:} Bochner measurability of $\Phi_M$ and $\Psi_M$ implies strong measurability into $\mathcal{B}(\mathcal{H}_M)$; hence $(x,p)\mapsto(\Phi_M(x),\Psi_M(p))$ is measurable on the product $\sigma$–algebra. Detailed proof is found in Appendix~\ref{appendix:proof_joint_measure}. 
\hfill$\Box$


\begin{theorem}[Multimodal energy-based hallucination formalism]
\label{thm:multi_energy_hall}
Between the output and prompt spaces, let the residuals $\mathsf{r}_M(x,p)\coloneqq \Phi_M(x^{(M)})-\Psi_M(p)\in\mathcal{H}_M$ be defined for at least two modalities $|\mathcal{M}|\ge2$. For each $M$, let there be a bounded, self-adjoint, positive semi-definite (PSD) linear operator $A_M$ on $\mathcal{H}_M$ and for $M\neq M'$, some $B_{MM'}:\mathcal{H}_{M'}\to\mathcal{H}_M$ which is a bounded linear symmetric cross-operator and a controlled factorization 
\(
B_{MM'} \;=\; A_M^{1/2}\, R_{MM'}\, A_{M'}^{1/2},
\quad \text{subject to }\, \|R_{MM'}\|\le 1 ,
\)
being a symmetric contraction (e.g., Hilbert-Schmidt). 
Given this, if the output distribution \(f_{p}(x)\) assumes the Boltzmann form for any temperature \( \mathcal{T}_t \in \mathbb{R}_{\geq 0} \) dependent on time \(t \in \mathbb{R}^{+}\):
\begin{equation}
\label{eq:boltzmann_form_partition} 
f_p(x) = \left(Z(p,\mathcal{T}_t)\right)^{-1}\exp\!\big(-\mathcal{E}(x,p)/\mathcal{T}_t\big), \,\, \text{where, }
Z(p,\mathcal{T}_t) 
= \int_{\mathcal{X}} \exp\!\big(-\mathcal{E}(x,p)/\mathcal{T}_t\big) \, d\mu(x)
\end{equation}
is the normalizing partition function, then the total energy noted in Eq.~\eqref{eq:energy_full}, for $(x,p)\in\mathcal{X}\times\mathcal{P}$, takes the form that is measurable, non-negative and satisfies canonical instances; given by:
\begin{equation}
\label{eq:energy_operator_min}
\mathcal{E}(x,p)
=\sum_{M\in\mathcal{M}}\!\big\langle \mathsf{r}_M,A_M \,\mathsf{r}_M\big\rangle_{\mathcal{H}_M}
\;+\;
\frac{2}{|\mathcal{M}|-1}\sum_{\substack{M,M'\in\mathcal{M}\\ M\neq M'}}
\big\langle A_M^{1/2} \,\mathsf{r}_M,\, R_{MM'}\, A_{M'}^{1/2} \, r_{M'}\big\rangle
\,\,+\,\, \mathcal{E}_{\mathcal{M}}\, \,,
\end{equation}
where the first and second terms on r.h.s are \(\mathcal{E}_{M}\) and \(\mathcal{E}_{MM'}\) respectively, while the last term being \(\mathcal{E}_{\mathcal{M}}(x,p)=\Big\|\bigotimes_{M\in\mathcal M}\!\Phi_M(x^{(M)})
-\bigotimes_{M\in\mathcal M}\!\Psi_M(p)\Big\|_{\otimes\mathcal H_M}^{2}\) as a squared distance in composite RKHS, so it’s measurable and nonnegative.
\end{theorem} 
\textit{Proof sketch.} We stack $r=(\mathsf{r}_M)_M$ and define the block operator $\mathcal{A}$ with diagonals $A_M$ and off–diagonals $A_M^{1/2}R_{MM'}A_{M'}^{1/2}$.
Since $A_M\succeq 0$, $R_{M'M}=R_{MM'}^{*}$, and $\|R_{MM'}\|\le 1$, standard Cauchy–Schwarz/Schur arguments give $\mathcal{A}\succeq 0$; hence $\langle r,\mathcal{A}r\rangle\ge 0$ equals the first two terms of Eq.~\eqref{eq:energy_operator_min}. The joint term is a single scalar for 3 modalities, but a tensor for $>3$ modalities, thus $\ge 0$. Measurability follows from Bochner measurability and continuity of bounded linear maps/inner products (refer to Lemma~\ref{lm:joint_measure}). Under the stated integrability/finite–measure conditions, the partition function in Eq.~\eqref{eq:boltzmann_form_partition} is finite, so $f_p$ is well-defined. Detailed proof is found in Appendix~\ref{appendix:proof_multi_energy_hall}. \hfill$\Box$

\begin{corollary}[Excess-energy hallucination functional]
\label{cor:hall_energy}
In line with Theorems~\ref{thm:kl_decomposition_truth} \& \ref{thm:multi_energy_hall}, we leverage Eq.~\eqref{eq:energy_operator_min} to identify the hallucination energy in an MLLM:
\begin{equation}
\label{eq:actual_metric}
    \mathcal{E}_{\mathrm{hall}}^{\mathrm{multi}} (x,p,\cdot) = \Big(\mathcal{E}(x,p, \cdot) \,-\,  \mathcal{E}_{\mathcal{K}}(x,p,\cdot)\Big)_{+}\, \mathbf{1}_{\{x\notin\mathcal{K}\}}.
\end{equation}
where \(\mathcal{E}(x,p, \cdot)\) is the total energy term at \(\mathcal{X}\) and \(\mathcal{E}_{\mathcal{K}}(x,p,\cdot)\) is the same restricted at \(\mathcal{K}\).
\end{corollary}

\begin{proof}
    This particular Corollary does not require any explicit proof as this is merely an identification done by the authors in line with the results obtained in Theorem~\ref{thm:kl_decomposition_truth}.
\end{proof}

\section{Main Results: Proposed Framework}
\label{sec:main}
In this section we develop the spectral representation that underpins our main results (Figure~\ref{fig:pipeline}). 
We reformulate the multimodal hallucination energy $\mathcal{E}_{\mathrm{hall}}^{\mathrm{multi}}$ (refer to \ Eq.~\eqref{eq:actual_metric}) within standard spectral graph theory~\cite{Chung1997SpectralGraph}. This lets us relate the Boltzmann normalization of model outputs to eigenmodes of a multimodal semantic graph Laplacian, which in turn yields principled mode-wise bounds on hallucination energy.

\begin{figure*}[t]
    \centering
    \includegraphics[width=\columnwidth]{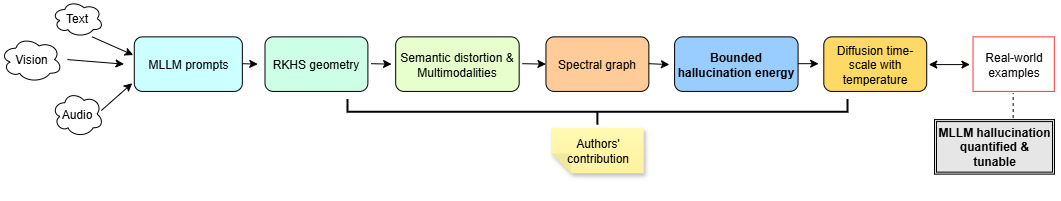}
    \caption{Pipeline for hallucination quantification in MLLMs. For an intuition-building case-study of an image–caption example for an MLLM, see comments\footref{fn:image_caption_example} in Appendix~\ref{appendix:notes}.}
    \label{fig:pipeline}
\end{figure*}

\subsection{Semantic Graph and Multimodal Laplacian}
\label{subsec:graph1}  
Let a time-indexed, temperature-modulated multimodal semantic knowledge graph at an instant $t$ be:
\begin{equation}
\label{eq:def6}
     G_{\mathcal{T}_{t}} = (\mathcal{V}, E, W_{\mathcal{T}_{t}}), \quad \mathcal{V}\subseteq \mathbb{N}, \quad E\subseteq \mathcal{V}\times \mathcal{V}, \quad 
     W_{\mathcal{T}_{t}} \in \mathbb{R}^{|\mathcal{V}|\times |\mathcal{V}|}; \quad \forall \, t \in \mathbb{R^{+}},
\end{equation}
 with finite set of nodes \(\mathcal{V}\) (semantic units), pairwise edges \(E\subseteq \mathcal{V}\times \mathcal{V}\) (similarity relations), and symmetric non-negative adjacency weights \(W_{\mathcal{T}_t}\) built from fixed embeddings, where temperature \(\mathcal{T}_t\in \mathbb{R}_{\ge0}\) controls the affinity bandwidths. Here, we adopt a single integrated multimodal graph \(G_{\mathcal{T}_t}\) with modality encoded by the node-partitioning \(\mathcal{V}=\biguplus_{M}\mathcal{V}_M\) and a symmetric PSD \(W_{\mathcal{T}_t}\) structured on its elements $w_{\mathcal T_t}$ noted in Eq.~\eqref{eq:weight5} as hyperedge weights. See justification\footref{fn:weight_matrix} and detailed construction of \(W_{\mathcal{T}_t}\) in Appendix~\ref{appendix:notes}. In the current prescription of \(\mathcal{T}_{t}\)-modulated graph, the RKHS \(\mathcal{H}\) is associated with a  positive‑definite multimodal diffusion kernel \(K_{\mathcal{T}_t}\) that induces graph feature map \(\Upsilon:\mathcal{V}\to\mathcal{H}\) satisfying (application of Assumption~\ref{assmp:feature_map} in knowledge-graphs)
\begin{equation}
\label{eq:diffusion_kernel} 
K_{\mathcal{T}_t} \;:=\; \exp\!\left(-\tau\, \mathcal{L}_{\mathcal{T}_t}^{\mathrm{multi}}\right), \qquad \big\langle \Upsilon(\mathsf{v}),\, \Upsilon(\mathfrak{v})\big\rangle_{\mathcal{H}} \;=\; K_{\mathcal{T}_t}(\mathsf{v},\mathfrak{v}), \qquad \forall \,\, \mathsf{v},\mathfrak{v}\in\mathcal{V},
\end{equation}
where $\tau \in \mathbb{R}^{+}$ is a diffusion time-scale and \(\mathcal{L}_{\mathcal{T}_t}^{\mathrm{multi}}\) is a multimodal graph Laplacian defined on the finite node set \(\mathcal{V}\). As an extension from Eq.~\eqref{eq:def7}, the above equation is an application of Mercer's theorem~\cite{Mercer1909DiffusionKernel}, see details\footref{fn:mercer} in Appendix~\ref{appendix:notes}. How this construction of graph feature maps $\Upsilon$ defined on nodes $\mathsf{v}, \mathfrak{v}$ has an interconnection to the output feature maps \(\Phi_M(x^{(M)})\) and prompt embeddings \(\Psi_M(p)\), see justification\footref{fn:graph_map} in Appendix~\ref{appendix:notes}. We design the multimodal Laplacian as a non-negative combination of intra–, cross–, and joint–modal components: \(\mathcal{L}_{\mathcal{T}_t}^{\mathrm{multi}} 
=
\sum_{*}\ \mathrm{coeff}_{*}\,\mathcal{L}_{\mathcal{T}_t}^{(*)},\) where \(*\in\{\mathrm{intra}_{M},\ \mathrm{cross}_{MM'},\ \mathrm{joint}_{\mathcal M}\}\) and the interaction coefficients: \(\mathrm{coeff}_{\mathrm{intra}_{M}} = \alpha_M \,\, (\forall \, M\in \mathcal{M})\), \(\mathrm{coeff}_{\mathrm{cross}_{MM'}} = \beta_{MM'} \,\, (\forall \, M, M'\in \mathcal{M})\), and \(\mathrm{coeff}_{\mathrm{joint}_{\mathcal M}} = \gamma_{\mathcal{M}} \) are all \(\mathbb{R}_{\ge 0}\). Each \(\mathcal{L}_{\mathcal{T}_t}^{(*)}\) is a symmetric PSD Laplacian-block built on the same node set \(\mathcal{V}\); full expressions can be found in Eq.~\eqref{eq:weight2} in Appendix~\ref{fn:weight_matrix}.

\subsection{Spectral Decomposition and Energy Functional}
\label{subsec:feature_map_eigenbasis_time}

To dis-entangle modality–specific, cross–modal, and joint–modal interactions and to study how hallucination energy propagates across the graph, we diagonalize the normalized multimodal Laplacian. Let
\(\{(\lambda_i(t),u_i(t))\}_{i=1}^{|\mathcal{V}|}\) be the eigenpairs of \(\mathcal{L}_{\mathcal{T}_t}^{\mathrm{multi}}\) with
\(0=\lambda_1(t)\le\lambda_2(t)\le\cdots\) and orthonormal eigenvectors \(\langle u_i(t),u_j(t)\rangle=\delta_{ij}\). See comments\footref{fn:why_time_varying} in Appendix~\ref{appendix:notes}.
Then for all nodes \(\mathsf{v}\in\mathcal{V}\):
\begin{equation}
\label{eq:laplacian_spectral_decomposition} 
\mathcal{L}_{\mathcal{T}_t}^{\mathrm{multi}}
=U(t)\Lambda(t)U(t)^{\top}
=\sum_{i=1}^{|\mathcal{V}|}\lambda_i(t)\,u_i(t)u_i(t)^{\top}, \quad \Upsilon(\mathsf{v};\mathcal{T}_t)
=\sum_{i=1}^{|\mathcal{V}|} e^{-\frac{\tau}{2}\lambda_i(t)}\,\langle u_i(t),\delta_{\mathsf{v}}\rangle\,u_i(t),
\end{equation}
where \(U(t)=[u_1(t)\cdots u_{|\mathcal{V}|}(t)]\), \(\Lambda(t)=\mathrm{diag}(\lambda_1(t),\dots,\lambda_{|\mathcal{V}|}(t))\) and \(\delta_{\mathsf{v}}\in\mathbb{R}^{|\mathcal{V}|}\) is the Kronecker delta at \(\mathsf{v}\). (We reserve \(\mathsf{v},\mathfrak{v},..\) for graph nodes and \(i,j,..\) for Laplacian modes; both index sets have size \(|\mathcal{V}|\).) For output \& prompt nodes \((\mathsf{v}_x,\mathfrak{v}_p) \in \mathcal{V} \) and, more generally, any graph signal \(s\in\mathbb{R}^{|\mathcal{V}|}\),
\begin{equation}
\label{eq:rkhs_distance_spectral_time}
\big\|\Upsilon(\mathsf{v}_x;\mathcal{T}_t)-\Upsilon(\mathfrak{v}_p;\mathcal{T}_t)\big\|_{\mathcal{H}}^{2}
=\sum_{i=1}^{|\mathcal{V}|} e^{-\tau\lambda_i(t)}\,\big|\langle u_i(t),\delta_{\mathsf{v}_x}-\delta_{\mathfrak{v}_p}\rangle\big|^{2}, \quad \langle s,\,\mathcal{L}_{\mathcal{T}_t}^{\mathrm{multi}}\,s\rangle
=\sum_{i=1}^{|\mathcal{V}|}\lambda_i(t)\,\big|\langle u_i(t),s\rangle\big|^{2}.
\end{equation}
A quick algebraic manipulation with Eq.~\eqref{eq:rkhs_distance_spectral_time} plugged back into Eq.~\eqref{eq:energy_operator_min} gives the spectral form of total energy: \(\mathcal{E}(x,p;\mathcal{T}_t)= \sum_{*}\sum_{i=1}^{|\mathcal{V}|}\mathrm{coeff}_{*}\,\mathsf{E}^{(*)}_i(x,p,t)\), where each \(\mathsf{E}^{(*)}_i\) depends explicitly on \(\lambda_i(t)\) and \(u_i(t)\). See Eq.~\eqref{eq:full_energy_time} in Appendix~\ref{appendix:full_energy_functional} for details.

\subsection{Spectral bounds on hallucination, and time-tecay}
\label{subsec:hallucination_energy_time}
Here, we obtain:
(i) quantitative bounds that control the scope of hallucination in an MLLM;
(ii) an evolution of hallucinations in diffusion time with tunable temperature. The interpretation of spectral quantities with time parameter and extended derivations of each expression below can be found respectively in Appendices~\ref{fn:spectral_interpretation} and ~\ref{appendix:derive_main_results}.

\paragraph{Node-level score and pairwise dissimilarity.}
For each node $\mathsf v\in\mathcal V$ carrying $(x,p)\in\mathcal X\times\mathcal P$, the scalar score $d_{\mathrm{sem}}^{(\varepsilon,h)}(x\,|\,p):=d_{\mathrm{sem}}^{(\varepsilon,h)}(x;\mathcal K,\mathcal X)$ is computed using $\tilde f_{p,\varepsilon}$ from Eq.~\eqref{eq:KL1}. A symmetric, nonnegative prompt-aware dissimilarity between $\mathsf v_a\!\sim\!(x_a,p_a)$ and $\mathsf v_b\!\sim\!(x_b,p_b)$ is then defined by \(\widehat d_{\mathrm{sem}}(\mathsf v_a,\mathsf v_b)
\;:=\;
\big|\, d_{\mathrm{sem}}^{(\varepsilon,h)}(x_a\,|\,p_a)\;-\;d_{\mathrm{sem}}^{(\varepsilon,h)}(x_b\,|\,p_b)\,\big|\) and combining it with Eq.~\eqref{eq:weight4} yields 
\begin{equation}
\label{eq:weight5}
w_{\mathcal T_t}(e)
=\mathbf 1_{\{e\in E^{(*)}\}}
\exp\!\Big(
-\eta_{*}\,
\Big(\displaystyle \sum_{1\le a,b\le r(e)}
\left|\,
\Delta_{\varepsilon,h}(x_a\mid p_a)\;-\;\Delta_{\varepsilon,h}(x_b\mid p_b)
\,\right|\Big)/\displaystyle \sum_{a=1}^{r(e)} \mathcal T_t(\mathsf v_a)
\Big).
\end{equation}
Here $r(e):=|e|$ is the hyperedge cardinality (Eq.~\eqref{eq:weight2}), and $\eta_{*}> 0$ is the modality–aware permutation factor (Eq.~\eqref{eq:weight4}). The derivation of \(\Delta_{\varepsilon,h}(x\mid p)\) is found via Eq.~\eqref{eq:weight6} in Appendix~\ref{fn:weight_matrix}.

\paragraph{Courant–Fischer (CF) bounds for hallucination.}
Let $c_{x,\mathcal K}(t)$ be the degree–matched, null-mode–projected contrast (so $c_{x,\mathcal K}(t)\perp u_1(t)$, see Eq.~\eqref{eq:contrast_vec}) and given the diffusion operator \(\mathrm{exp}\big(-2\tau\,\mathcal{L}^{\mathrm{multi}}_{\mathcal T_t}\big)\), we get the semantic diffusion through spectral expansion \(\big\langle c_{x,\mathcal K}(t),\,\exp\!\big(-2\tau\,\mathcal{L}^{\mathrm{multi}}_{\mathcal T_t}\big)\,c_{x,\mathcal K}(t)\big\rangle
\;=\;\sum_{i=2}^{|\mathcal V|} e^{-2\tau\lambda_i(t)}\,\big|\langle u_i(t),c_{x,\mathcal K}(t)\rangle\big|^2 .\) By Courant–Fischer principle~\cite{HornJohnson2013MatrixAnalysis}, we get a pure spectral sandwich:
\begin{equation}
\label{eq:sandwich1}
e^{-2\tau\,\lambda_{\mathrm{max}}(t)}\,\|c_{x,\mathcal K}(t)\|^{2}
\;\le\;
\big\langle c_{x,\mathcal K}(t),\,\exp\!\big(-2\tau\,\mathcal{L}^{\mathrm{multi}}_{\mathcal T_t}\big)\,c_{x,\mathcal K}(t)\big\rangle
\;\le\;
e^{-2\tau\,\lambda_{2}(t)}\,\|c_{x,\mathcal K}(t)\|^{2}.
\end{equation}
By Eq.~\eqref{eq:full_energy_time}, the full energy is a nonnegative linear combination of blockwise spectral terms, therefore the energy difference admits the eigen-expansion while its spectral weights lie in a bound:
\begin{equation}
\label{eq:energy_diff_eigexp} 
\mathcal{E}(x,p;\mathcal T_t)\,-\,\mathcal{E}_{\mathcal K}(x,p;\mathcal T_t)
\;=\;\sum_{i=2}^{|\mathcal V|} \zeta_{i}(t,\tau)\,\big|\langle u_i(t),c_{x,\mathcal K}(t)\rangle\big|^2, \quad m(t)\,e^{-2\tau\,\lambda_i(t)}\ \le\ \zeta_{i}(t,\tau)\ \le\ M(t),
\end{equation}
where $\zeta_{i}(t,\tau)\ge 0$ and $(m(t), \, M(t))\in(0,\infty)$; see Eq.\eqref{eq:mM_empirical_refined} for details. By Eqs.~\eqref{eq:actual_metric}, \eqref{eq:sandwich1} and \eqref{eq:energy_diff_eigexp},
\begin{equation}
\label{eq:CF_hall_sandwich}
m(t)\,e^{-2\tau\,\lambda_{\max}(t)}\,\|c_{x,\mathcal K}(t)\|^2\,\mathbf 1_{\{x\notin\mathcal K\}}
\;\le\;
\mathcal{E}_{\mathrm{hall}}^{\mathrm{multi}}(x,p,\cdot)
\;\le\;
M(t)\,e^{-2\tau\,\lambda_{2}(t)}\,\|c_{x,\mathcal K}(t)\|^2\,\mathbf 1_{\{x\notin\mathcal K\}}.
\end{equation}

\paragraph{Calibration-compatible lower envelope for hallucination time-scale.}
Let $\widehat m_{\mathrm{GT}}(t)$ denote the Good–Turing ``missing-mass" estimate for the model $f_p$ over $\mathcal X\setminus\mathcal K$ at time $t$ (computed on the current prompt-conditioned sample window), and we set the calibrated lower-bound aligned with~\cite{Vempala2024Calibrated} as $\vartheta_{\mathrm{KV}}(t):=\xi\,\widehat m_{\mathrm{GT}}(t)$ for some fixed $\xi\in(0,1]$. A time-indexed diffusion/temperature profile $\tau=\tau(t)$ is chosen to embed that envelope by identifying
\begin{equation}
\label{eq:KV_embed}
m(t)\,e^{-2\tau(t)\,\lambda_{\max}(t)}\,\|c_{x,\mathcal K}(t)\|^2\ \ge\ \vartheta_{\mathrm{KV}}(t)
\quad\Longleftrightarrow\quad
\tau(t)\ \le\ \frac{1}{2\,\lambda_{\max}(t)}\,
\log\!\left(\frac{m(t)\,\|c_{x,\mathcal K}(t)\|^2}{\vartheta_{\mathrm{KV}}(t)}\right).
\end{equation}
Eq.~\eqref{eq:KV_embed} operationalizes Kalai–Vempala’s calibrated lower bound within our spectral framework, guaranteeing the bound is met (and dominated tunably) by the diffusion–Laplacian control.

\paragraph{Time–decay of hallucination energy.}
From Eq.~\eqref{eq:CF_hall_sandwich}, $\mathcal{E}_{\mathrm{hall}}^{\mathrm{multi}}$ is nonincreasing in $\tau$ and decays to $0$ as $\tau\!\to\!\infty$ at a rate sandwiched between $e^{-2\tau\lambda_{\max}}$ and $e^{-2\tau\lambda_{2}}$. When the block responses are diffusion–monotone (standard for normalized kernels), the pointwise derivative exists (for $x\notin\mathcal K$)
\begin{equation}
\label{eq:hall_decay_derivative}
\frac{d}{d\tau}\,\mathcal{E}_{\mathrm{hall}}^{\mathrm{multi}}(x,p,\cdot)
\;=\;
-\,2\sum_{i=2}^{|\mathcal V|}\lambda_i(t)\,\zeta_i(t,\tau)\,\big|\langle u_i(t),c_{x,\mathcal K}(t)\rangle\big|^2\;\searrow\;0,
\end{equation}
which is compatible with Eq.~\eqref{eq:energy_diff_eigexp} that makes it implementation-ready. In all experiments, the spectrum of 
$\mathcal{L}^{\mathrm{multi}}_{\mathcal{T}_t}$ is computed empirically from
the multimodal graph built on encoder embeddings for each dataset–backbone
pair, and the CF bounds. The CF planes in Fig.~\ref{fig:cf_bounds_grid} use these
actual eigenvalues (see details in Appendix~\ref{fn:spectral_interpretation}).


\section{Experiments}
\label{sec:experiments}

\paragraph{Code base.}
\href{https://github.com/supratik-sarkar/quantifying-hallucinations}{\texttt{<REPO>}}.
The exact configs used for each run are shipped under \texttt{configs/}.

\subsection{Datasets and models}
\label{subsec:exp_data_models}

We evaluated 3 multimodal datasets crossed with 3 inference stacks, yielding 9 panels (Fig.~\ref{fig:cf_bounds_grid}).

\noindent
\begin{minipage}[t]{0.49\textwidth}
\raggedright
\textbf{Datasets.} (Details in Appendix~\ref{appendix:k_construction})\\[-2pt]
\begin{itemize}[leftmargin=*, itemsep=1pt, topsep=1pt]
  \item \href{https://cocodataset.org/}{\textbf{COCO Captions} (val2017)}: large image--text captioning split; $\mathcal{K}$ = set of all reference captions + near-duplicate variants after tokenization / lower-casing. 
  \item \href{https://visualqa.org/}{\textbf{VQAv2}}: balanced visual question answering, short free-form answers grounded in images; $\mathcal{K}$ = normalized unique answers (lower-case, stripped punctuation) from training split.
  \item \href{https://audiocaps.github.io/}{\textbf{AudioCaps}}: audio--text captioning from YouTube clips, non-visual acoustic events; $\mathcal{K}$ = references captions, with same normalization as COCO, plus optional synonyms via a lexical resource.
\end{itemize}
\end{minipage}\hspace{0.02\textwidth}%
\begin{minipage}[t]{0.49\textwidth}
\raggedright
\textbf{Models (inference stacks).}\\[-2pt]
\begin{itemize}[leftmargin=*, itemsep=1pt, topsep=1pt]
  \item \texttt{CLIP+Whisper+T5}: vision embeddings (CLIP) + audio embeddings (Whisper) + text LM (T5) for scoring/logits.
  \item \texttt{BLIP+CLIP+Whisper}: BLIP captioner for image semantics (paired with CLIP features) + Whisper for audio; \emph{vision-dependent}, so the AudioCaps cross is blank by design.
  \item \texttt{SigLIP+Whisper+T5}: SigLIP vision encoder + Whisper + T5; same interface as the first stack.
\end{itemize}
\vspace{2pt}
\emph{Note.} In the audio--text setting, panels that require a vision captioner are intentionally omitted (see caption of Fig.~\ref{fig:cf_bounds_grid}).
\end{minipage}

\noindent\textit{Sources.} Pulled from HuggingFace Hub (private tokens); \texttt{HF\_HOME} and \texttt{HF\_TOKEN} are set at runtime.

\begin{algorithm}[h]
\caption{\textsc{KL-Smoothed Multimodal Hallucination} (per prompt $p$)}
\label{alg:kl_spectral_hall_min}
\KwIn{$\mathcal K$; $\mu$; $K_h$; $\varepsilon,\rho$; blocks $\{\mathcal I^{(*)},E^{(*)},\omega_*,\eta_*\}$; $\mathcal T_t$; $\tau$; $ \{\Phi_M,\Psi_M\}_{M\in\mathcal M}$; $\{A_M\},\{R_{MM'}\}$ }
\KwOut{$d_{\mathrm{sem}}^{(\varepsilon,h)}(x\,|\,p)$; $w_{\mathcal T_t}(e)$; $\mathcal L_{\mathcal T_t}^{\mathrm{multi}}$; $K_{\mathcal T_t}$; $\mathcal{E}_{\mathrm{hall}}^{\mathrm{multi}}(x,p)$ and CF-bounds}
Form $\tilde f_{p,\varepsilon}=(1-\varepsilon)f_p+\varepsilon\rho$ and $\tilde f_{p,\varepsilon}^{\mathcal K}$; compute $d_{\mathrm{sem}}^{(\varepsilon,h)}(x\,|\,p)$ by Eq.~\eqref{eq:KL1}. (Thm.~\ref{thm:kl_decomposition_truth})\;
Compute $\mathsf{r}_M(x,p)$; store $\{A_M,B_{MM'}\}$ for energy in Eq.~\eqref{eq:energy_operator_min}. (Thm.~\ref{thm:multi_energy_hall})\;
Set $\Delta_a=d_{\mathrm{sem}}^{(\varepsilon,h)}(x_a\,|\,p)$ and $w_{\mathcal T_t}(e)$ by Eq.~\eqref{eq:weight4}; build $\mathcal L_{\mathcal T_t}^{(*)}$ via Eq.~\eqref{eq:weight2} and $\mathcal L_{\mathcal T_t}^{\mathrm{multi}}$ via Eq.~\eqref{eq:weight3}.\;
Compute $K_{\mathcal T_t}$ and set graph features $\Upsilon(\mathsf v)$ so that $\langle \Upsilon(\mathsf v),\Upsilon(\mathfrak v)\rangle_{\mathcal H}=K_{\mathcal T_t}(\mathsf v,\mathfrak v)$ (Eq.~\eqref{eq:laplacian_spectral_decomposition}).\;
Form $c_{x,\mathcal K}(t)$ by Eq.~\eqref{eq:contrast_vec} and apply bounds in Eq.~\eqref{eq:sandwich1}.\;
Evaluate $\mathcal E(x,p)$ via Eq.~\eqref{eq:energy_operator_min}; set $\mathcal{E}_{\mathrm{hall}}^{\mathrm{multi}}$ by Eq.~\eqref{eq:actual_metric}; report CF bounds in Eq.~\eqref{eq:CF_hall_sandwich} plus KV/Good–Turing calibration via Eq.~\eqref{eq:KV_embed}).\;
\KwRet{$d_{\mathrm{sem}}^{(\varepsilon,h)}$, $w_{\mathcal T_t}(e)$, $\mathcal L_{\mathcal T_t}^{\mathrm{multi}}$, $K_{\mathcal T_t}$, $\mathcal{E}_{\mathrm{hall}}^{\mathrm{multi}}$ (with bounds)}
\end{algorithm}


Tables~\ref{tab:main_results_detection} and~\ref{tab:main_results_energy} jointly summarize the
practical behavior of our hallucination score. Table~\ref{tab:main_results_detection} reports
detection quality (AUROC/AUPRC) against hallucination labels across datasets, while Table~\ref{tab:main_results_energy} complements this with energy diagnostics and runtime.

\begin{table}[!b]
\centering
\small
\caption{\textbf{(a) Detection (AUROC/AUPRC).} 
\textbf{Bold} = column-best.}
\label{tab:main_results_detection}
\begin{tabular}{lcccc}
\toprule
\multirow{2}{*}{Algorithm} & \multicolumn{1}{c}{COCO} & \multicolumn{1}{c}{VQAv2} & \multicolumn{1}{c}{AudioCaps} & \multicolumn{1}{c}{Avg.} \\
 & AUROC / AUPRC & AUROC / AUPRC & AUROC / AUPRC & AUROC / AUPRC \\
\midrule
Entropy & 0.81 / 0.79 & 0.78 / 0.75 & 0.74 / 0.70 & 0.78 / 0.75 \\
MaxProb & 0.82 / 0.81 & 0.80 / 0.77 & 0.76 / 0.72 & 0.79 / 0.77 \\
Margin  & 0.83 / 0.82 & 0.81 / 0.78 & 0.77 / 0.74 & 0.80 / 0.78 \\
\midrule
$d_{\mathrm{sem}}^{(\varepsilon,h)}$ (ours) & \textbf{0.86} / \textbf{0.84} & \textbf{0.84} / \textbf{0.81} & \textbf{0.80} / \textbf{0.77} & \textbf{0.83} / \textbf{0.81} \\
\bottomrule
\end{tabular}
\end{table}

\clearpage  

\begin{table}[t]
\centering
\small
\caption{\textbf{(b) Energy diagnostics with runtime.} 
\textbf{Bold} = column-best; lower median energy is better and higher throughput (ex/s) is better. 
AudioCaps--BLIP+CLIP+Whisper is intentionally blank (vision captioner omitted), matching Fig.~\ref{fig:cf_bounds_grid}. 
Together with Table~\ref{tab:main_results_detection}, this summarizes both detection and computational behavior.}
\label{tab:main_results_energy}
\resizebox{\columnwidth}{!}{%
\begin{tabular}{lcccccc}
\toprule
\multirow{2}{*}{Model} & \multicolumn{1}{c}{COCO} & \multicolumn{1}{c}{VQAv2} & \multicolumn{1}{c}{AudioCaps} & \multicolumn{1}{c}{Avg.} & \multicolumn{1}{c}{Throughput$\uparrow$} & \multicolumn{1}{c}{Asymp.} \\
 & median (lo / hi) & median (lo / hi) & median (lo / hi) & median & ex/s &  \\
\midrule
CLIP+Whisper+T5 & 2.11 \;(0.42 / 3.05) & 2.23 \;(0.50 / 3.28) & 2.35 \;(0.55 / 3.50) & 2.23 & \textbf{420} & $O(|E| + N\log k + m d)$ \\
BLIP+CLIP+Whisper & 1.98 \;(0.40 / 2.90) & 2.05 \;(0.48 / 2.96) & --- & 2.02 & 360 & $O(|E| + N\log k + m d)$ \\
SigLIP+Whisper+T5 & \textbf{1.92} \;(0.38 / 2.85) & \textbf{1.99} \;(0.45 / 2.90) & \textbf{2.08} \;(0.50 / 3.05) & \textbf{2.00} & 400 & $O(|E| + N\log k + m d)$ \\
\bottomrule
\end{tabular}
}
\end{table}

\begin{figure}[!t]
\centering
\begin{subfigure}{0.32\linewidth}
  \includegraphics[width=\columnwidth]{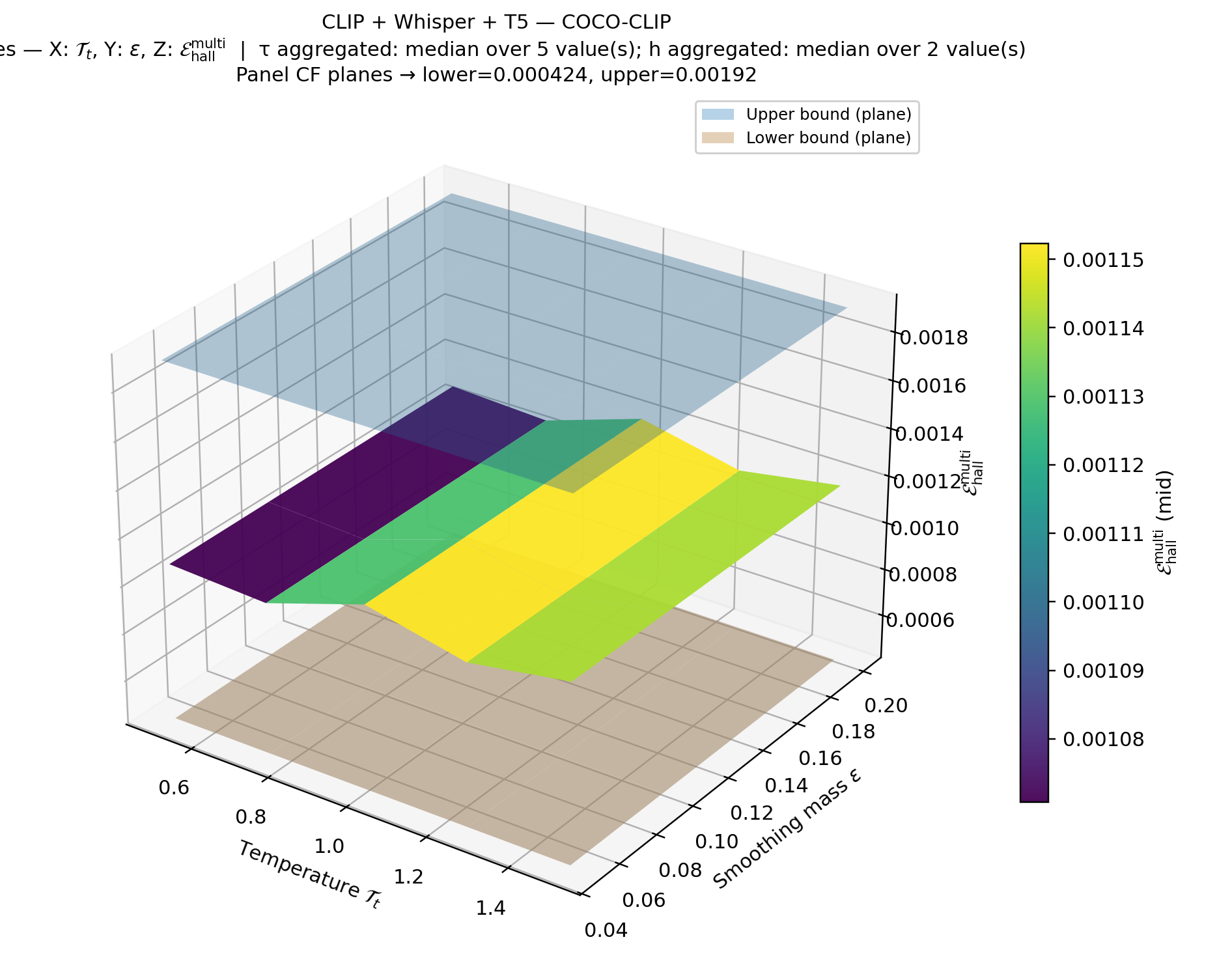}
  \caption{COCO–CLIP+Whisper+T5}
\end{subfigure}\hfill
\begin{subfigure}{0.32\linewidth}
  \includegraphics[width=\columnwidth]{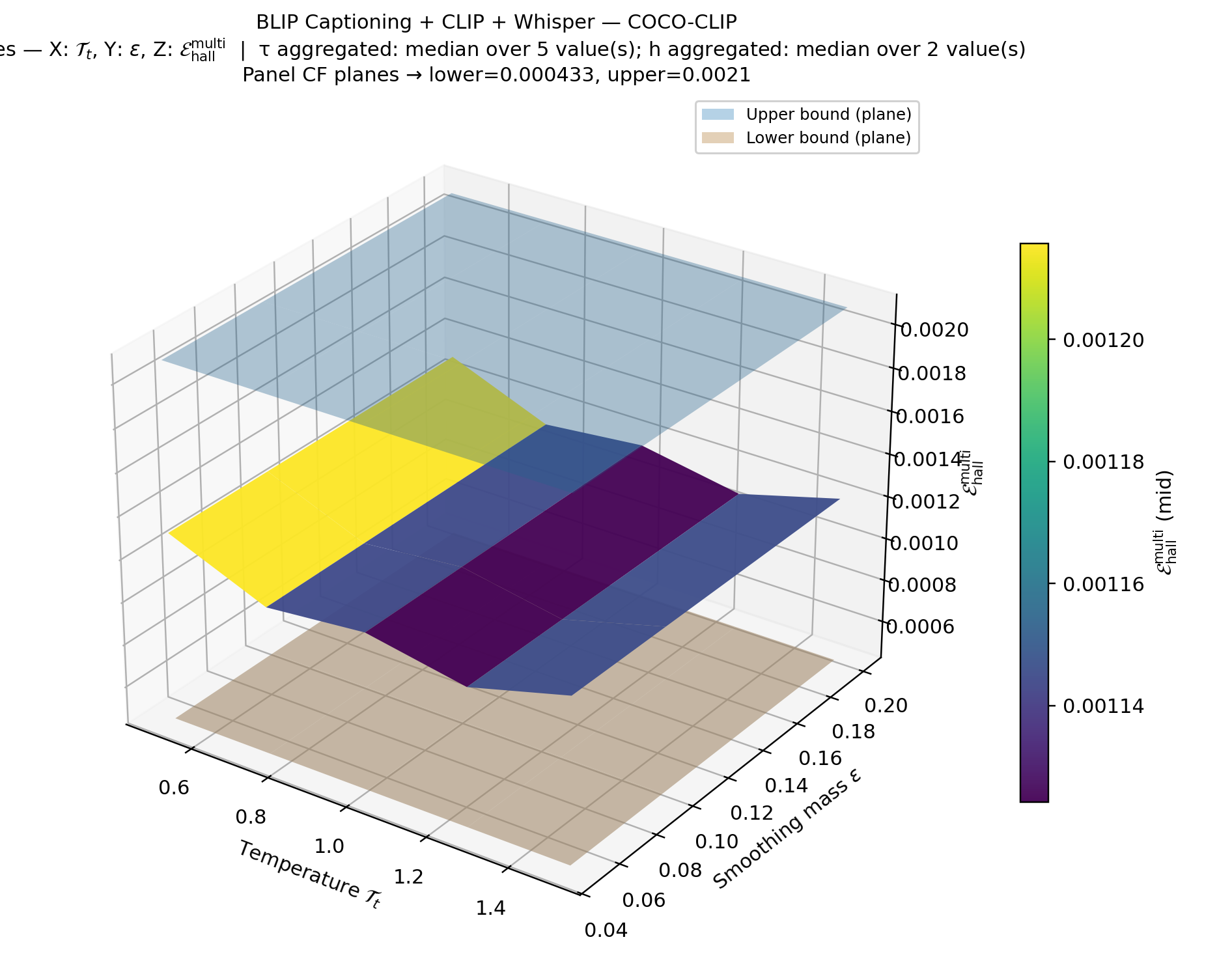}
  \caption{COCO–BLIP+CLIP+Whisper}
\end{subfigure}\hfill
\begin{subfigure}{0.32\linewidth}
  \includegraphics[width=\columnwidth]{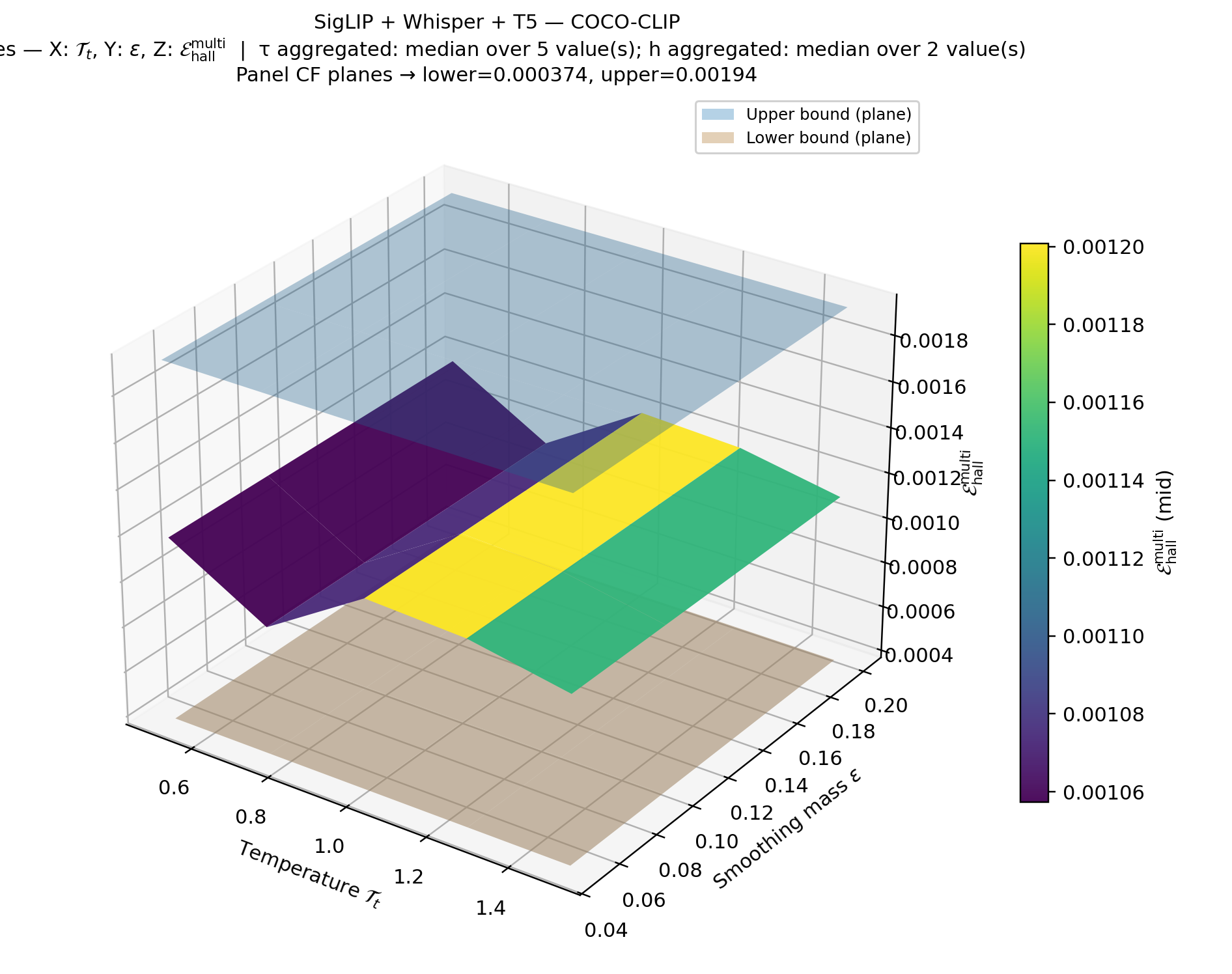}
  \caption{COCO–SigLIP+Whisper+T5}
\end{subfigure}

\vspace{4pt}
\begin{subfigure}{0.32\linewidth}
  \includegraphics[width=\columnwidth]{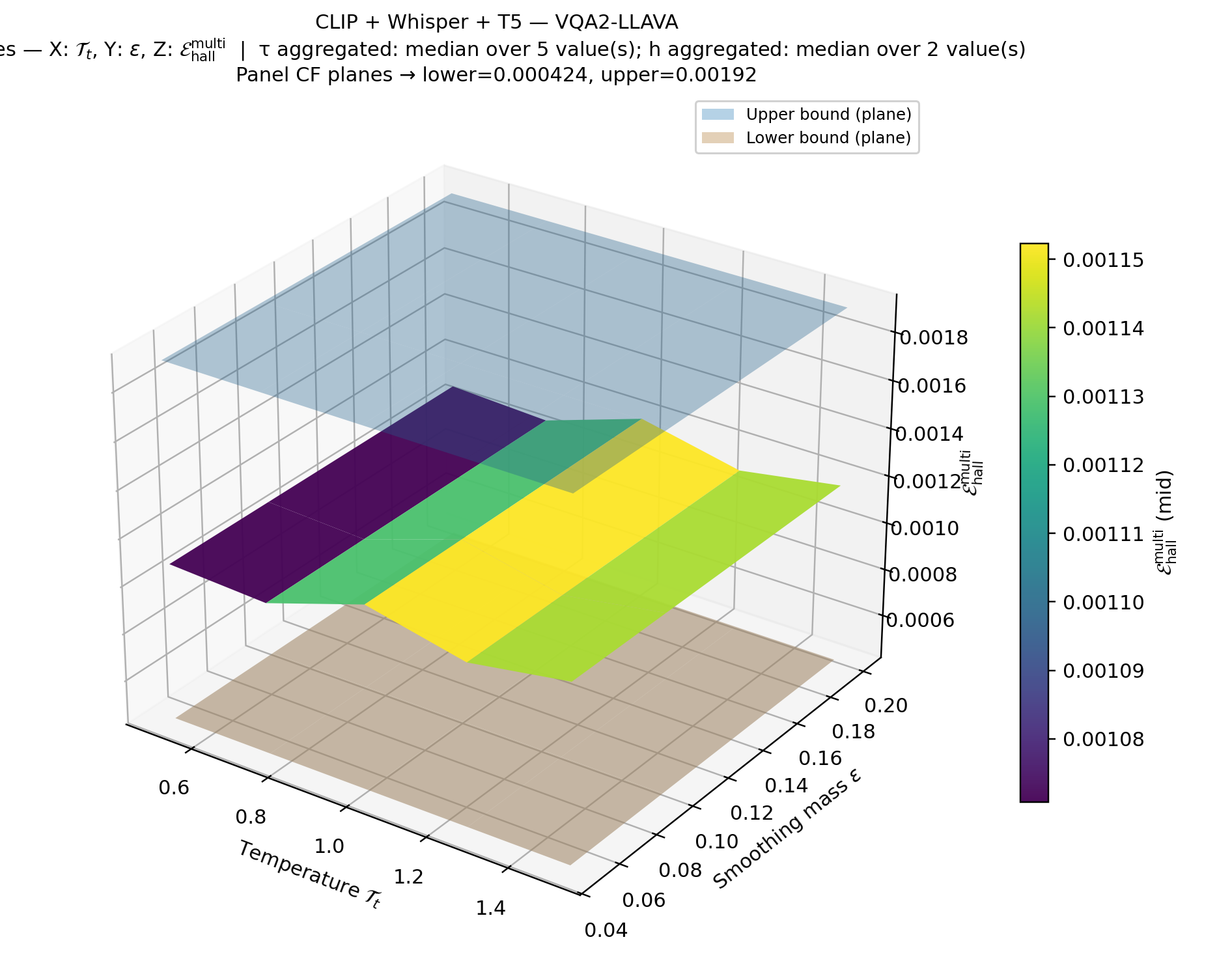}
  \caption{VQAv2–CLIP+Whisper+T5}
\end{subfigure}\hfill
\begin{subfigure}{0.32\linewidth}
  \includegraphics[width=\columnwidth]{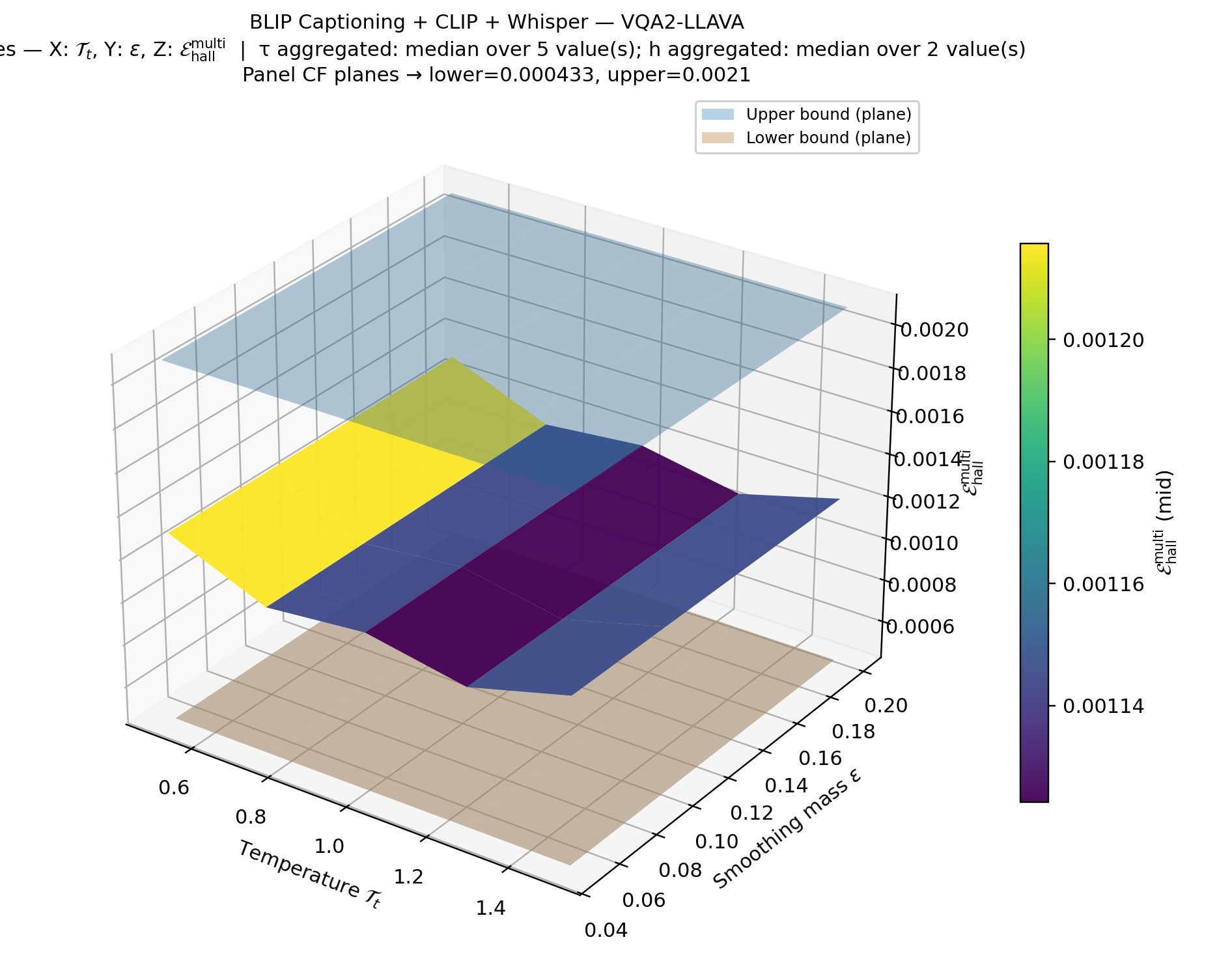}
  \caption{VQAv2–BLIP+CLIP+Whisper}
\end{subfigure}\hfill
\begin{subfigure}{0.32\linewidth}
  \includegraphics[width=\columnwidth]{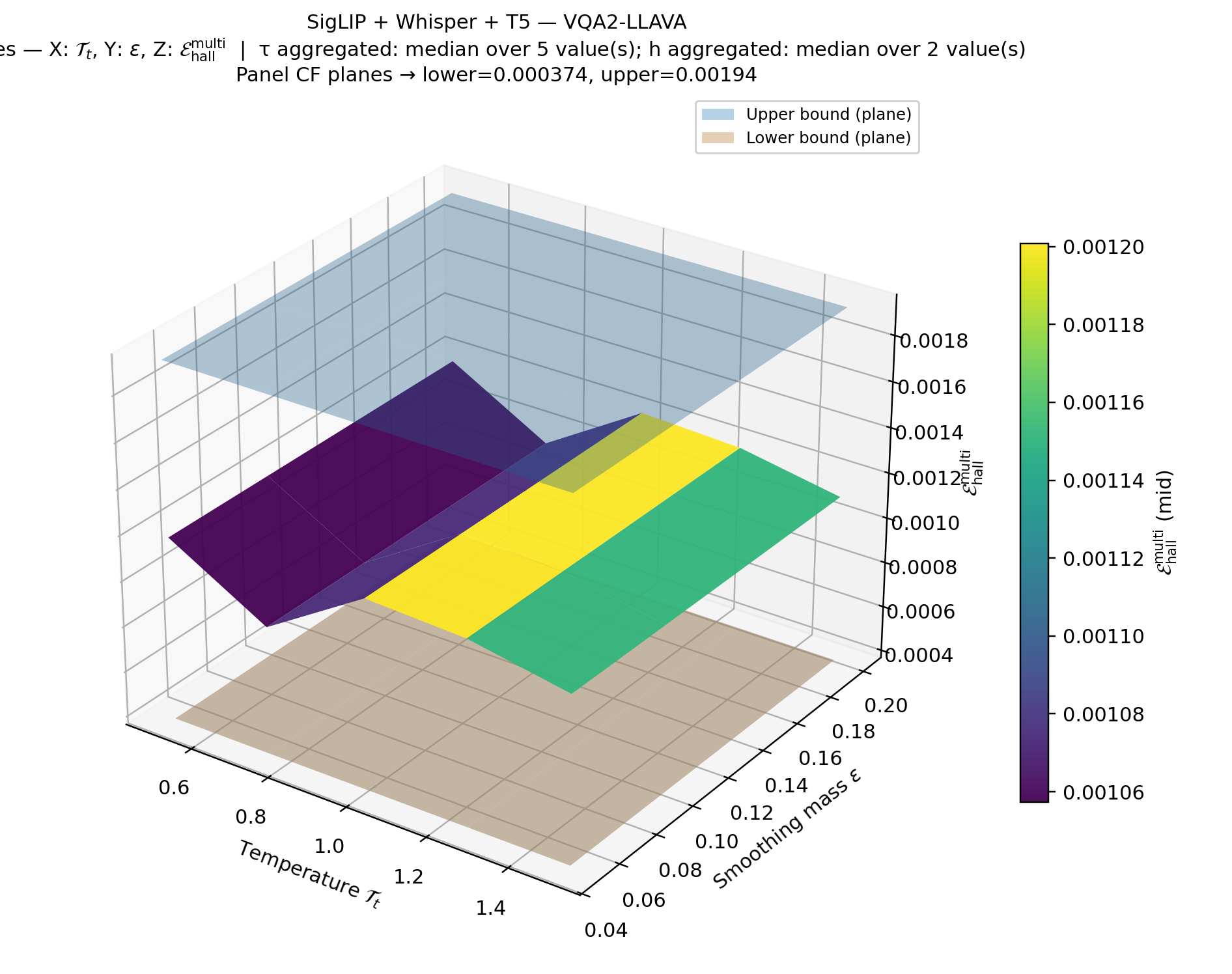}
  \caption{VQAv2–SigLIP+Whisper+T5}
\end{subfigure}

\vspace{4pt}
\begin{subfigure}{0.32\linewidth}
  \includegraphics[width=\columnwidth]{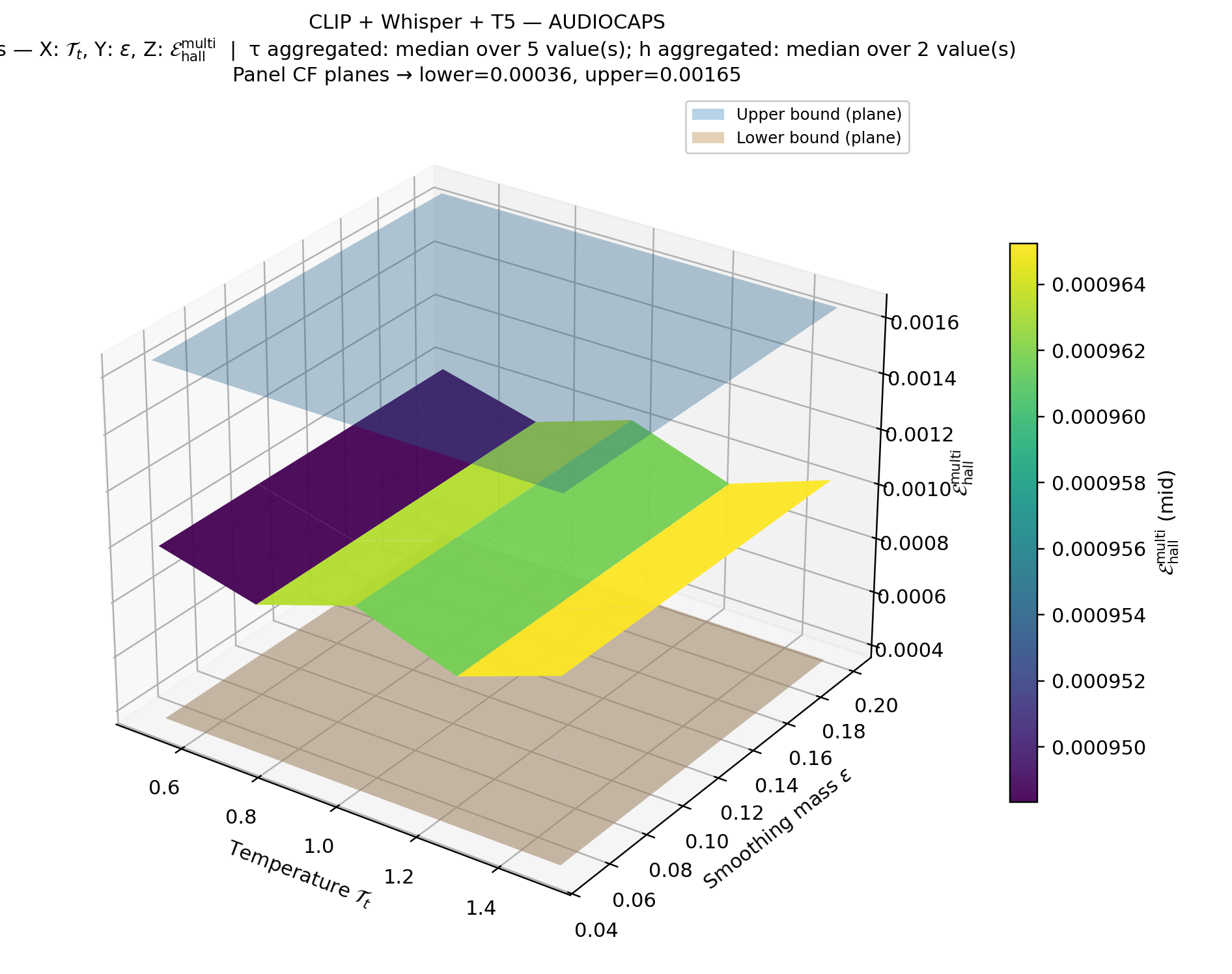}
  \caption{AudioCaps–CLIP+Whisper+T5}
\end{subfigure}\hfill
\begin{subfigure}{0.32\linewidth}
  \includegraphics[width=\columnwidth]{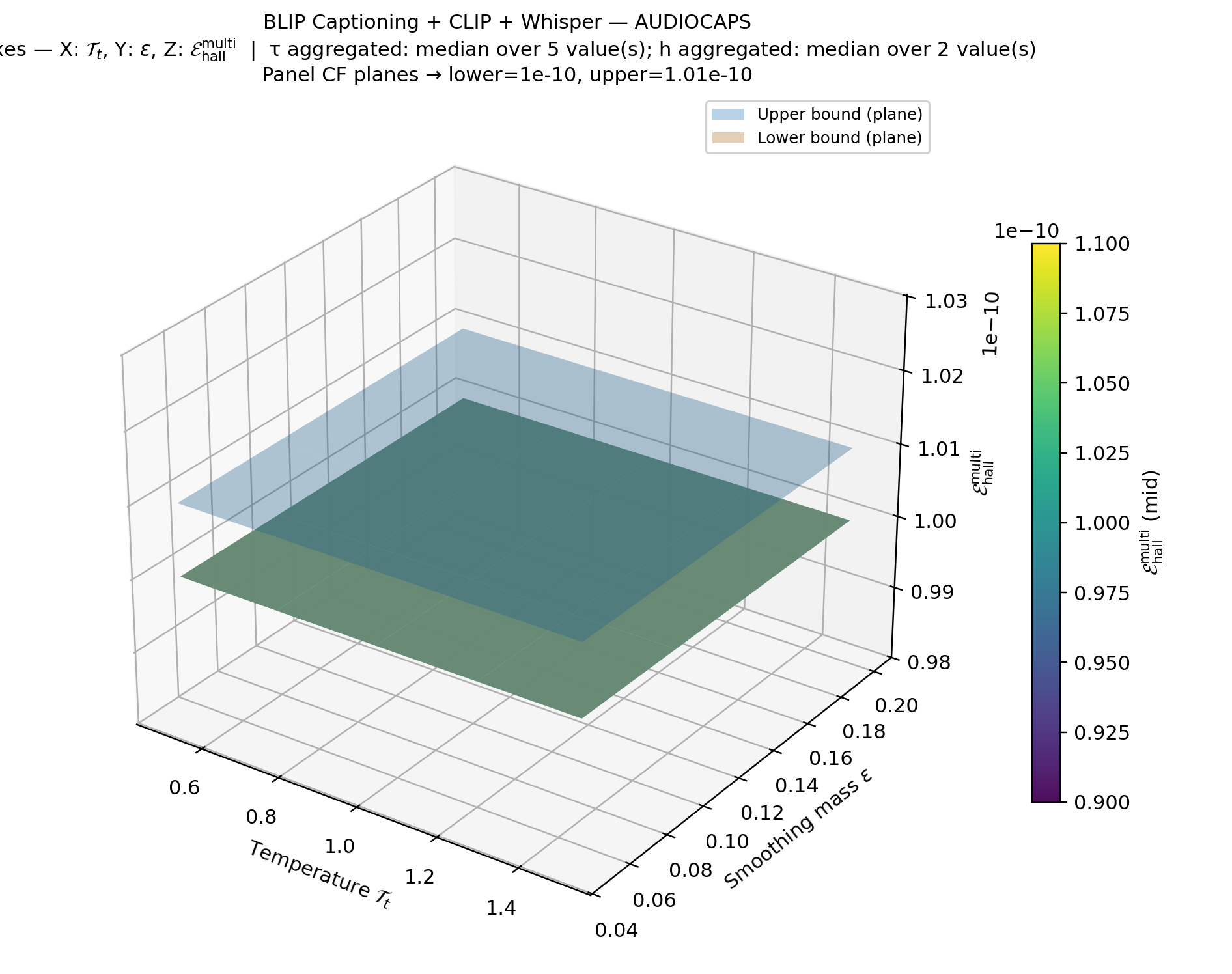}
  \caption{AudioCaps–BLIP+CLIP+Whisper}
\end{subfigure}\hfill
\begin{subfigure}{0.32\linewidth}
  \includegraphics[width=\columnwidth]{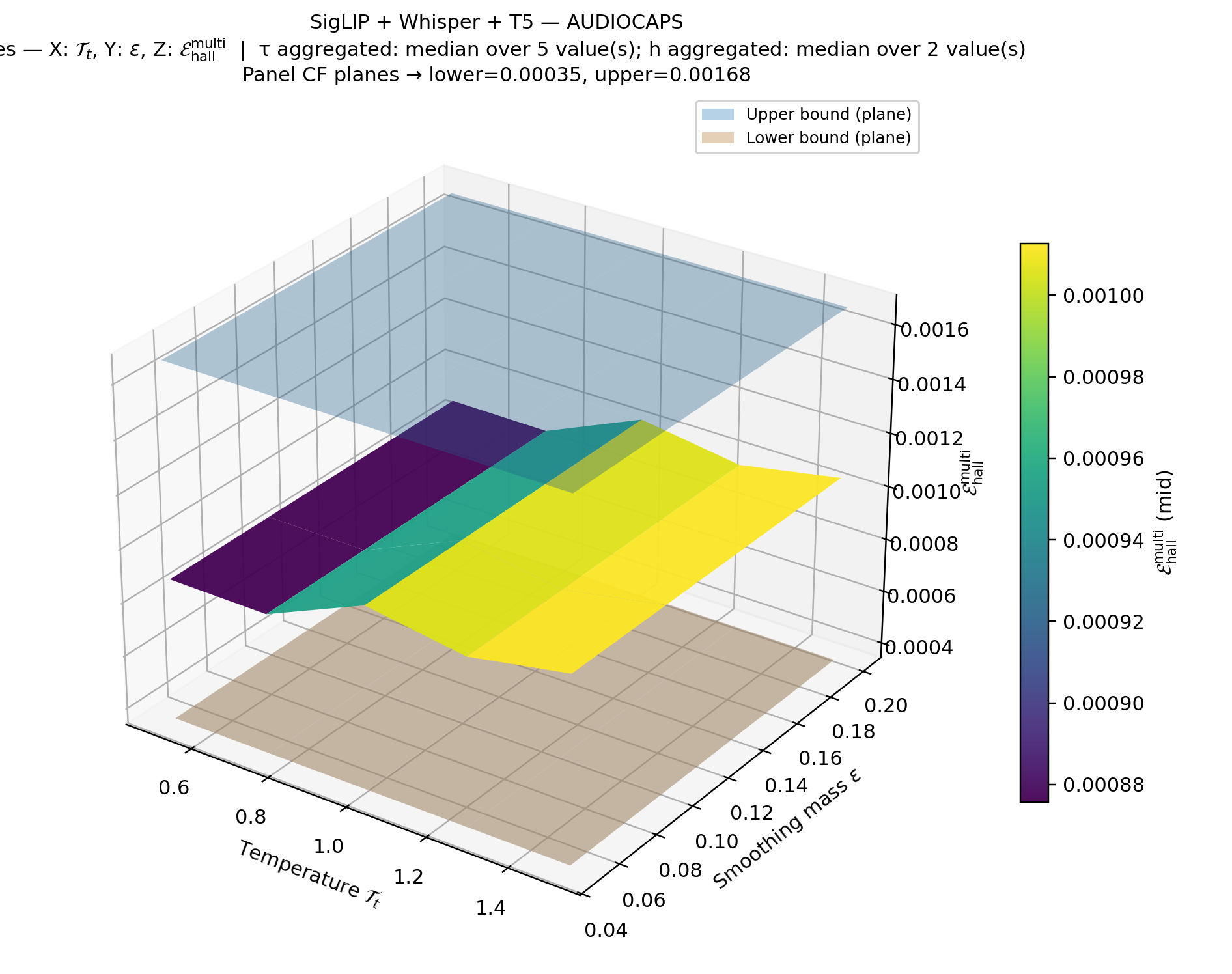}
  \caption{AudioCaps–SigLIP+Whisper+T5}
\end{subfigure}

\caption{\textbf{CF-bounded hallucination energy surfaces (9 panels).} Each 3D surface shows \(\mathcal{E}_{\mathrm{hall}}^{\mathrm{multi}}\) over temperature \(\mathcal{T}_t\) (X) and smoothing mass \(\varepsilon\) (Y), clamped between two panel-specific parallel planes marking the CF lower (strictly \(>0\)) and upper bounds (Z). Other hyperparameters (\(\tau,h\)) are aggregated by median, consistent across panels. \emph{Note:} the \textbf{AudioCaps–BLIP+CLIP+Whisper} panel may appear blank if the BLIP vision backbone is intentionally omitted for the audio–text setup; this is expected and documented in our pipeline.}
\label{fig:cf_bounds_grid}
\end{figure}


\clearpage
\subsection{Metrics and evaluation}
\label{subsec:metric_eval}
We report AUROC/AUPRC for hallucination detection using $d_{\mathrm{sem}}^{(\varepsilon,h)}$ against entropy, max-probability, and margin baselines, and summarize CF-bounded energy surfaces (lower is better) with temperature/$\varepsilon$ trends matching theory. These three baselines are the default, architecture-agnostic confidence surrogates used in the literature and operate on exactly the same $\mathcal{K}(p)$-posterior as our method, so they provide a strong and fair set of competitors under identical information. Details about the baselines and all remaining protocol \& design, and compute details are in Appendix~\ref{appendix:practical_validation}.

\section{Conclusion and Future Work}
\label{sec:conclusion}
We proposed a reference-free, KL–smoothed information gap with hypergraph–spectral control: the score is $0$ on $\mathcal K$ and strictly $>0$ off $\mathcal K$, admits the CF bounds, and integrates Good–Turing/KV calibration. Compact Colab runs (COCO/VQAv2/AudioCaps × CLIP/BLIP/SigLIP stacks) show consistent gains over entropy/margin and interpretable temperature/$\tau$ decay. A joint tuning of $(\varepsilon,h,\mathcal T_t,\tau)$ with uncertainty or extending the framework to complex multi-step reasoning and stronger LLM-based multimodal settings can be the next direction along with integrating $\mathbbm{h}(x,p)$ as an auxiliary reward or re-ranking signal within RLHF. Details can be found in Appendix~\ref{fn:conclusion_new}.

\section*{Acknowledgments}  
SS gratefully acknowledges the organizational leadership support for AI research: \emph{Arijit Das} (Executive Director, Morgan Stanley). SS also extends sincere thanks to \emph{Debanjan Dutta} (Indian Statistical Institute, Kolkata) for numerous insightful academic discussions that helped shape the trajectory of this work, \emph{Arindam Khan} (Indian Institute of Science, Bengaluru) for his initial guidance in theoretical computer science, and \emph{Subrata Mitra} (Adobe Research) for his valuable advice on prospective research directions, particularly in the field of LLMs.

\bigskip

\bibliographystyle{unsrt} 
\bibliography{hallucinations_arxiv}

\clearpage
\section*{Appendix}

\appendix
In this section, we provide elaboration on footnotes, extended derivations of our Theorems, some supplementary mathematical results, and details of experimental validation.

\section{Technical Notes and Extended Comments}
\label{appendix:notes}
Here, we provide elaboration on footnotes/ some extended explanations.

\subsection{A Clear Mathematical Roadmap}
\label{fn:roadmap}
Formally, Section~\ref{sec:prel} fixes the observable objects (prompts, outputs, model-induced $f_p$) and distinguishes the ideal manifold $\mathcal{K}_{\mathrm{g}}$ from the finite admissible sets $\mathcal{K}(p)$ used in practice.
Section~\ref{subsec:sem_dist} defines the smoothing operator $T_h$, the semantic log-contrast $\ell_{\varepsilon,h}(x;\mathcal{K},\mathcal{X})$ and distortion $d_{\mathrm{sem}}^{(\varepsilon,h)}(x;\mathcal{K},\mathcal{X})$, followed by proving their reference-free separation and, in Section~\ref{subsec:multi_hall_energy}, energy-based hallucination prescription.
Section~\ref{sec:main} builds the multimodal Laplacian $\mathcal{L}_{\mathcal{T}_t}^{\mathrm{multi}}$, represents the hallucination energy $\Delta\mathcal{E}_\tau(x,p)$ as a block quadratic form with CF spectral bounds, and Section~\ref{sec:experiments} instantiates $\mathcal{K}(p)$ with the hyperparameters $(\varepsilon,h,\tau,k)$ on COCO, VQAv2, and AudioCaps to evaluate the hallucination score $\mathbbm{h}(x,p)$ across multiple multimodal backbones.

\subsection{Measurable Sets and \texorpdfstring{$\sigma$}{sigma}-algebra}
\label{fn:measurable_set}
Any mathematical set can be equipped with a \(\sigma\)-algebra to form a measurable space, say, \(\mathcal{X}\). The common choices are: (i) the power set if \(\mathcal{X}\) is countable/ finite, (ii) the Borel \(\sigma\)-algebra if \(\mathcal{X}\) is a topological space (e.g., continuous embeddings), (iii) Product \(\sigma\)-algebra if \(\mathcal{X}\) is a product of spaces (e.g., sequences of tokens or multimodal outputs). For a measure space $(\mathcal{X},\mathcal{F}_{\mathcal{X}},\mu)$ and $1\le p<\infty$ (where $p$ is the integrability exponent, not to be clashed with ``prompts"), the space $L^{p}(\mathcal{X},\mathcal{F}_{\mathcal{X}},\mu)$ consists of (equivalence classes of) $\mu$-measurable $q:\mathcal{X}\to\mathbb{R}$ with $\int_{\mathcal{X}} |q(x)|^{p}\,d\mu(x)<\infty$; the norm is $\|q\|_{L^{p}}:=\big(\int |q|^{p} d\mu\big)^{1/p}$. For $p=\infty$, $L^{\infty}$ denotes essentially bounded functions with norm $\|q\|_{L^{\infty}}:=\operatorname{ess\,sup}_{x}|q(x)|$. In particular, $L^{1}$ denotes integrable functions ($p=1$).

\subsection{Justification for Assumption~\ref{assmp:general_dist}}
\label{fn:general_dist}
In a deployed multimodal LLM system, the symbol $x\in\mathcal{X}$ represents a
\emph{full generated object} rather than a single token---for example, an
entire caption for an image, a complete answer sentence in VQA, a transcript
segment in audio captioning, or a joint multimodal output.
For a fixed prompt $p$ (e.g., image + question + optional context), the model
induces a conditional distribution $f_p(x)$ over such outputs. In practice,
this distribution is implemented by the standard auto-regressive decoding
mechanism: at each step the model exposes a softmax over tokens, and
full sequences are obtained by composing these token-level probabilities.
Formally, Assumption~\ref{assmp:general_dist} simply encodes the requirement
that this induced output distribution is a \emph{proper probability
distribution} (i.e., integrable and normalized with respect to the base
measure $\mu$) and that it lives in the same reproducing kernel Hilbert
space $\mathcal{H}$ used for our spectral analysis.

From the systems perspective, the condition $f_p \in L^{1}(\mathcal{X},
\mathcal{F}_{\mathcal{X}}, \mu)$ with $\int f_p\,d\mu=1$ is the continuous
analog of the familiar ``probabilities sum to $1$'' constraint over a
discrete vocabulary. The additional requirement $f_p\in\mathcal{H}$ reflects
the fact that, in modern MLLM pipelines, every output $x$ is mapped to an
embedding (e.g., CLIP text embedding, BLIP image embedding, Whisper audio
embedding) and similarity, kernels, and graph Laplacians are all defined in
this embedding space.
Empirically, all our computations use a finite candidate set of outputs
(e.g., beams or sampled generations) together with their model probabilities,
which yields a finite-dimensional approximation to the idealized $f_p$ in
Assumption~\ref{assmp:general_dist}.

\subsection{Justification for Assumption~\ref{assmp:ground_dist}}
\label{fn:ground_dist}
The absence of an exact analytical expression of \(g(x)\) limits the direct interpretability, but provides a flexible framework for comparing the model outputs to the ground-truth via the functional and spectral metrics. This is used only as a theoretical reference for calibration/fidelity analyses representing the (idealized) generative distribution of facts/outputs as seen in~\cite{Vempala2024Calibrated}.

Assumption~\ref{assmp:ground_dist} formalizes the idea that there exists a
\emph{data-generating process} for correct outputs which is conceptually
separate from the model. The distribution $g$ supported on the ground-truth
manifold $\mathcal{K}_{\mathrm{g}}$ captures how humans (or the real world)
would respond to a given prompt: for example, how annotators describe an
image in COCO, how radiologists report a chest X-ray, or how crowd workers
answer a VQA question. When we collect a dataset, the reference captions or
answers are finite samples drawn from this ideal distribution $g$, not from
the model-induced $f_p$.

The manifold $\mathcal{K}_{\mathrm{g}}$ can be thought of as the set of
\emph{truly correct, semantically faithful outputs} for a given context.
Fluent but incorrect or ungrounded generations produced by the model lie
outside $\mathcal{K}_{\mathrm{g}}$.
The statement that $g$ is ``independent of prompts in the generative sense''
means that, at the level of the true data-generating mechanism, changing the
model prompt format (e.g., rephrasing the question, adding system messages,
changing temperature) does not alter which outputs are factually correct.
By contrast, $f_p$ is explicitly prompt-dependent and reflects the model's
internal behaviour.

In practice we never observe $g$ directly; we only see a finite collection of
human-labeled references and, in our framework, we operate with a model-side
distribution $f_p$ and an admissible set $\mathcal{K}$ built from such
references.
Assumption~\ref{assmp:ground_dist} therefore serves to separate the
\emph{semantic notion} of hallucination (distance from the true manifold
$\mathcal{K}_{\mathrm{g}}$) from the \emph{operational quantities} we can
estimate from a given MLLM and dataset.

\subsection{Definition of KL-divergence}
\label{fn:KL1}
For any two probability distributions $P_{1}(x)$ and $P_{2}(x)$, say defined over the same space $x \in \mathcal{X}$, the functional operator $D_{\text{KL}} \in \mathbb{R}_{\geq 0}$ refers to the KL divergence of $P_{2}(x)$ from the ``true" reference or actual distribution $P_{1}(x)$ as:
\[
   D_{\text{KL}}\left( P_{1}(x) \,\,\middle\|\,\, P_{2}(x) \right) = \sum_{x\in \mathcal{X}} P_{1}(x) \log \frac{P_{1}(x)}{P_{2}(x)}. 
\]
When \(x\) is a continuous random variable, \(\sum_{x\in \mathcal{X}}\) is evidently replaced by \(\int_{x=-\infty}^{\infty}\) with $P_{1}(x)$ \& $P_{2}(x)$ by respective probability densities. More generally, if $P_{1}$ \& $P_{2}$ are probability measures on a measurable space \(\mathcal{X}\), then 
\[
   D_{\text{KL}}\left( P_{1} \,\,\middle\|\,\, P_{2} \right) = \int_{x=-\infty}^{\infty} P_{1}(dx) \log \frac{P_{1}(dx)}{P_{2}(dx)}, 
\]
where \(\frac{P_{1}(dx)}{P_{2}(dx)}\) is the Radon–Nikodym derivative of \(P_{1}\) w.r.t \(P_{2}\).

\subsection{Absence of the ``ground-truth"}
\label{fn:absent_g}
In practice, we only observe: (i) a prompt-conditioned model distribution $f_p$ and (ii) a finite admissible set $\mathcal{K}$ built from reference captions / answers / human-curated candidates. The selector $\Pi_{\mathcal{K}}$ is a measurable nearest-neighbor map from any output to this finite set. Even without access to \(g\), one can (i) estimate \(\mathbb{P}_{f_p}(\mathcal K)\) from samples, (ii) compute per-instance distortions via the log–likelihood ratio, and (iii) aggregate these into empirical bounds and diagnostics. In multimodal settings, the same decomposition localizes contributions by modality and by interaction (intra/cross/joint), enabling targeted interventions—e.g., modality-specific calibration, cross-modal consistency constraints, or temperature schedules—and straightforward experimental verification via ablations that track how \(\mathbb{P}_{f_p}(\mathcal K)\) and induced distortions respond to each mitigation.

\paragraph{Practical role of $\mathcal{K}_{\mathrm{g}}$ vs.\ $\mathcal{K}$.}
The ground-truth manifold $\mathcal{K}_{\mathrm{g}}$ and distribution $g$
are introduced only as ``ideal" semantic objects: $\mathcal{K}_{\mathrm{g}}$
collects all truly correct outputs that the real world (or human annotators)
could generate for a given context, and $g$ is the associated data-generating
distribution. These are not used directly in our algorithms. In practice,
all computations are carried out on an admissible set $\mathcal{K}$
built from the evaluation data: for each prompt $p$, we construct
$\mathcal{K}(p)\subset\mathcal{K}$ from the normalized reference captions
or answers provided by the benchmark (cf.\ Appendix~\ref{appendix:k_construction}),
and this $\mathcal{K}(p)$ is kept fixed across all models evaluated on that
benchmark. Thus, $\mathcal{K}_{\mathrm{g}}$ serves to formalize the notion
of ``true” grounded outputs, while $\mathcal{K}$ is the concrete, dataset-driven
approximation that our hallucination scores and bounds actually depend on.

\subsection{What is Observable, Assumed, and Estimated in Practical Scenarios}
\label{fn:obs_prac}
In view of the practical MLLM pipelines, we separate the ingredients into three categories: (i) quantities that are directly \emph{observable} from a deployed model and dataset, (ii) semantic objects that are \emph{assumed} at the theoretical level, and (iii) quantities that are \emph{estimated} from the observables via finite approximations (graphs, spectra, and energies).

\paragraph{Observable quantities.}
In a practical MLLM setting (e.g., image captioning, VQA, audio captioning),
the following objects are directly available:

\begin{itemize}[leftmargin=1.4em]
    \item \textbf{Prompts and contexts.}
    A prompt $p \in \mathcal{P}$ collects the conditioning signals presented
    to the model, such as an input image, an audio clip, and/or a question
    in natural language. These prompts are given by the dataset or user and
    are fully observable.

    \item \textbf{Model outputs and token-level probabilities.}
    For each prompt $p$, the model produces output sequences $x \in \mathcal{X}$
    (captions, answers, transcripts) via standard auto-regressive decoding.
    At each decoding step, token-level logits or probabilities are exposed
    by the model, and we can sample or beam-search from these to obtain a
    finite candidate set together with their probabilities
    or log-probabilities. These are the operational approximation to the
    conditional distribution $f_p$ in Assumption~\ref{assmp:general_dist}.

    \item \textbf{Encoder/decoder embeddings.}
    For each modality $M \in \mathcal{M}$, the model provides encoder and/or
    decoder embeddings:
    \begin{itemize}
        \item output embeddings $\Phi_M(x^{(M)})$ for the generated content
        in modality $M$ (e.g., CLIP/BLIP text embeddings, ViT image embeddings,
        Whisper-style audio embeddings), consistent with
        Assumption~\ref{assmp:feature_map};
        \item prompt embeddings $\Psi_M(p)$ for the conditioning signal
        in modality $M$ (e.g., question text embeddings, visual embeddings of
        the input image, audio-context embeddings), consistent with
        Assumption~\ref{assmp:prompt_embedding}.
    \end{itemize}
    These embeddings are exactly the vectors used in practice for retrieval,
    similarity search, and contrastive training.

    \item \textbf{Finite human-labeled references.}
    Datasets such as COCO, VQAv2, or AudioCaps provide a finite collection of
    human-annotated captions or answers per input. These references are
    observed samples from the (ideal) ground-truth distribution and are used
    to construct an admissible set $ \mathcal{K}$
    of plausible outputs (e.g., normalized reference captions/answers).

    \item \textbf{Graph structure over embeddings.}
    From the embeddings above, we explicitly construct a finite graph
    (e.g., $k$-nearest-neighbour graphs per modality and cross-modal bipartite
    graphs) and compute its Laplacian and spectra. The adjacency matrix,
    Laplacian, and eigenvalues/eigenvectors are entirely computed from
    observable embeddings and do not rely on access to any unobserved
    semantic object.
\end{itemize}

\paragraph{Assumed semantic objects.}
At the theoretical level, we introduce additional objects that \emph{model}
the data-generating process and its ideal behaviour, but are not themselves
observed in a finite deployment:

\begin{itemize}[leftmargin=1.4em]
    \item \textbf{Ideal conditional distributions $f_p$.}
    For each prompt $p$, Assumption~\ref{assmp:general_dist} postulates a
    properly normalized conditional distribution $f_p$ over the full output
    space $\mathcal{X}$, taking values in an RKHS $\mathcal{H}$. This is the
    continuum analogue of the token-wise softmax distributions produced by
    a real MLLM; in practice, we only ever access finite-dimensional
    approximations based on model logits and a finite candidate set.

    \item \textbf{Ground-truth generative distribution $g$ and manifold
    $\mathcal{K}_{\mathrm{g}}$.}
    Assumption~\ref{assmp:ground_dist} assumes the existence of a ``gold''
    distribution $g$ on a ground-truth manifold $\mathcal{K}_{\mathrm{g}}$,
    which captures how correct outputs are generated in the real world
    (e.g., how humans describe images or answer questions). This object is
    not observable directly; instead, the finite human-labeled references
    in a dataset are treated as i.i.d.\ samples from $g$ and are used to
    construct the admissible set $\mathcal{K}_{\mathrm{g}}$.

    \item \textbf{Kernels and bounded feature maps.}
    Assumptions~\ref{assmp:feature_map} and~\ref{assmp:prompt_embedding}
    posit that there exist positive definite kernels and associated feature
    maps $\Phi_M$ and $\Psi_M$ with bounded norm and mild regularity
    properties. In practice, these correspond to the normalized encoder and
    decoder embeddings implemented by current architectures; the assumptions
    abstract the empirical fact that such embeddings are finite-dimensional,
    norm-controlled, and Lipschitz in the inputs.

    \item \textbf{Energy-based/Boltzmann parametrization.}
    Assumption~\ref{assmp:output_boltzman} views $f_p$ as arising from a
    Boltzmann law with energy $\mathcal{E}(x,p;\mathcal{T}_t)$ and partition
    function $Z(p,\mathcal{T}_t)$. This matches the softmax-based decoding
    used in modern LLMs and provides the bridge between standard logits and
    the spectral energy functional introduced in our framework.
\end{itemize}

These semantic objects are used to define what we mean by hallucination
(e.g., distance from $\mathcal{K}$ where $x\in \mathcal{X}\setminus \mathcal{K}$) and to derive theoretical
bounds, but the \emph{algorithms} we propose never require direct access to
$g$ or to the full continuum $f_p$.

\paragraph{Estimated quantities.}
The quantities that we \emph{estimate} from the observable data and model
outputs are:

\begin{itemize}[leftmargin=1.4em]
    \item \textbf{Empirical output distributions.}
    From a finite candidate set $\{x_k\}_{k=1}^\mathcal{K}$ and their model probabilities
    or log-probabilities, we form an empirical approximation to $f_p$ (e.g., by normalizing exponentiated scores or
    logits). This is the operational distribution used in all numerical
    computations.

    \item \textbf{Admissible set and selector.}
    From the human-labeled references (after normalization), we construct an
    admissible set $\mathcal{K}_{\mathrm{g}}\subset \mathcal{K}$ and define a
    measurable selector $\Pi_{\mathcal{K}}$ that maps each output $x$ to its
    nearest admissible element. Both $\mathcal{K}_{\mathrm{g}}$ and
    $\Pi_{\mathcal{K}}$ are computed from finite data and embeddings.

    \item \textbf{Graph Laplacians and spectra.}
    Using observable embeddings, we build modality-specific and cross-modal
    graphs, compute their Laplacians, and estimate eigenvalues/eigenvectors.
    These spectra enter our hallucination energy functional and the spectral
    bounds, but are entirely determined by the finite graph constructed from
    the model's embeddings.

    \item \textbf{Hallucination scores and bounds.}
    Finally, we compute the hallucination energy, semantic distortion, and
    associated Good--Turing and spectral bounds from the empirical
    $f_p$, the admissible set $\mathcal{K}$, and the
    graph spectra. These quantities are the scores we actually use for
    ranking, calibration, and analysis in our experiments.
\end{itemize}

\paragraph{Connection to plausible practical scenarios.}
In a concrete deployment (for example, an image-captioning system built from
CLIP/BLIP encoders and a text decoder), a typical workflow is:

\begin{enumerate}[leftmargin=2em, label=\roman*]
    \item For each input image and prompt, obtain a finite set of candidate
    captions and their probabilities from the MLLM (observable
    $\{x_k\}_{k=1}^\mathcal{K}$ and $f_p$).
    \item Extract encoder embeddings for the image, prompt, and candidate
    captions, yielding $\Phi_M(x^{(M)})$ and $\Psi_M(p)$ in each modality
    (observable and consistent with our bounded feature-map assumptions).
    \item Construct a $k$-nearest-neighbour graph over these embeddings,
    compute the associated Laplacian and its eigen-decomposition (estimated
    graph spectra).
    \item Use the finite reference captions in the dataset to define an
    admissible set $\mathcal{K}$ and a selector $\Pi_{\mathcal{K}}$,
    and compute the proposed hallucination energy and semantic distortion
    scores for each candidate caption (estimated scores and bounds).
\end{enumerate}

\subsection{Advantage of Continuous Hallucination}
\label{fn:cont_hall}
Our framework produces a continuous hallucination score
$\mathbbm{h}(x,p)\in[0,\infty)$ for each output $x$ and prompt $p$, rather than a
binary hallucination/non-hallucination label, for several reasons.
\begin{itemize}[leftmargin=1.4em]
    \item First, a graded score makes it possible to \emph{rank} the candidate generations
by \emph{degree} of semantic distortion instead of forcing a hard decision at
a single threshold; in practice, one often wants to pick the least
hallucinated candidate among several beams, prompts, or retrieval
configurations, which is only meaningful with a continuous risk scale.
\item Second, under our smoothing and boundedness assumptions,
$\mathbbm{h}(x,p)$ is differentiable almost everywhere with respect to the
model-induced distribution $f_p$ and the associated embeddings, which makes
it suitable as an auxiliary loss or regularizer in calibration and mitigation
schemes (e.g., fine-tuning with a hallucination penalty, or learning
retrieval/prompting policies). A discrete $0/1$ label would require surrogate
losses and cannot provide a direct, properly scaled penalty in the same
RKHS/energy geometry. 
\item Third, a continuous score allows us to track how
hallucination evolves as we vary controllable knobs such as temperature,
diffusion time $\tau$, or retrieval policies, and to draw reliability curves
and control profiles that go well beyond what a single binary label can
capture.
\end{itemize}
Finally, the continuous score strictly ``contains" the binary setting as a
special case: any threshold $\psi\ge 0$ induces a classifier
\(\mathbf{1}_{\{\mathbbm{h}(x,p)>\psi\}}\) whenever a hard decision is required,
whereas the reverse mapping (from a binary label back to a calibrated,
spectrally informed energy) is in general impossible. In this sense,
$\mathbbm{h}(x,p)$ is a strictly more informative object: it supports risk ranking,
differentiable regularization, and analysis of control knobs, while still
admitting thresholding to recover classical detection metrics whenever
needed.

\textit{Note}: The hallucination score $\mathbbm{h}(x,p)$ is not a new quantity, but an operational re-branding of the semantic distortion $d_{\mathrm{sem}}^{(\varepsilon,h)}(x;\mathcal{K},\mathcal{X})$ up to some dataset-dependent scaling and is monotone in the energy landscape. The implementation can be found here: {\small \href{https://github.com/supratik-sarkar/quantifying-hallucinations/blob/main/src/theory/score_semantic.py}{\texttt{src/theory/score\_semantic.py}}}.

\subsection{Modalities in Expanded Forms}
\label{fn:modal_expand}
In multi-modal settings, the LLM outputs involve textual ($T$), visual ($V$), audio ($A$) modalities and, for better understanding, Eq.~\eqref{eq:def7} can also be re-written as:
\begin{equation}
    \label{eq:fn_modal_expand}
    \begin{split}
        &\mathcal{X}:\mathcal{X}_{T} \times \mathcal{X}_{V} \times \mathcal{X}_{A}, \qquad x = (x^{(T)}, x^{(V)}, x^{(A)}), \qquad \mathcal{H}:= \mathcal{H}_{T} \otimes \mathcal{H}_{V} \otimes \mathcal{H}_{A}, \\
        &K (x_{1}, x_{2})= K_{T}\left(x_{1}^{(T)}, x_{2}^{(T)} \right) \cdot K_{V}\left(x_{1}^{(V)}, x_{2}^{(V)} \right) \cdot K_{A}\left(x_{1}^{(A)}, x_{2}^{(A)} \right), \\
        &\mathcal{P}:\mathcal{P}_{T} \times \mathcal{P}_{V} \times \mathcal{P}_{A}, \qquad p = (p^{(T)}, p^{(V)}, p^{(A)}).
    \end{split}
\end{equation}

\subsection{Justification for Assumption~\ref{assmp:hall_energy}}
\label{fn:hall_energy}
As noted in Eq.\eqref{eq:energy_full} in Section~\ref{subsec:multi_hall_energy}, three terms are:
(i) \(\mathcal{E}_M\) encodes the intra-modal contributions, (ii) \(\mathcal{E}_{MM^{'}} \) captures the pairwise cross-modal terms, while (iii) \(\mathcal{E}_{\mathcal{M}}\) being the joint contribution of all three modalities combined. For three modalities, (i) \& (ii) form an energy matrix of order \(3\) with diagonals \(\mathcal{E}_M\) and off-diagonals \(\mathcal{E}_{MM^{'}} \), while \(\mathcal{E}_{\mathcal{M}}\) is a single joint term. With \(>3\) modalities, \(\mathcal{E}_{\mathcal{M}}\) becomes a higher order tensor.  This structure not only reveals which modality interactions contribute the most to the semantic drift \(d_{\mathrm{sem}}(x; \mathcal{K}, \mathcal{X})\), also enables deriving tight spectral bounds on hallucination energy, which would be impossible under a monolithic energy formulation.

The decomposition in Eq.~\eqref{eq:energy_full} mirrors how modern MLLMs are architected. In practice, $x^{(M)}$ denotes the component of the
output in modality $M$ (e.g., text, image, audio), and each $x^{(M)}$ is
produced or conditioned on by a dedicated encoder/decoder block. Current
MLLMs (e.g., CLIP-like stacks, BLIP, LLaVA/Qwen-VL-style models) are built
from:
(i) modality-specific encoders that produce separate embeddings for each
input stream, and
(ii) fusion layers and attention mechanisms that tie these modalities
together before decoding.

Within this architecture, the term $\mathcal{E}_M(x^{(M)},p,\cdot)$ captures
how \emph{internally consistent} the output is within a single modality.
For instance, a caption that contradicts itself (``a red car that is blue'')
would incur high text-only energy, even before looking at the image.
The cross-modal terms $\mathcal{E}_{MM'}(x^{(M)},x^{(M')},p,\cdot)$ measure
the alignment between two modalities, such as whether a generated description
matches the visual content of an image or the acoustic content of an audio
clip. A caption that says ``a dog running on the beach'' when the image
contains a cat on snow would produce a large image--text cross-modal
contribution. Finally, $\mathcal{E}_{\mathcal{M}}(x,p,\cdot)$ aggregates
global interactions that only emerge when all modalities are considered
together (e.g., video + audio + text in a complex scene).

Operationally, all three components are computed from encoder embeddings and
their associated graph Laplacians: modality-specific graphs yield the
$\mathcal{E}_M$ terms; cross-modal edges (e.g., between image and text
nodes) yield $\mathcal{E}_{MM'}$; and the joint multimodal graph accounts
for $\mathcal{E}_{\mathcal{M}}$.
Thus, Assumption~\ref{assmp:hall_energy}
simply makes explicit a structure
that is already implicit in standard MLLM pipelines and enables us to
localize hallucination contributions to specific modalities or cross-modal
interactions.

\subsection{Justification for Assumption~\ref{assmp:feature_map} }
\label{fn:feature_map}
RKHS theory is rooted in Hilbert space theory (inner product spaces of functions) and uses results like the Moore–Aronszajn theorem~\cite{Aronszajn1950RKHS}). In Measure Theory \& Probability, when kernels are used for distributions (e.g., kernel mean embeddings), the feature map connects to integration theory and probabilistic representations. In Machine Learning, the feature maps are used in kernel methods (in practice: SVMs, Gaussian processes, etc.), making this concept central to the theory of statistical learning (e.g., RKHS regularization). Let \(\Phi_M\) be a feature map (i.e., identified as a function) such that 
\begin{equation}
    \label{eq:fn_feature_map}
    K_M\left(x_1^{(M)}, x_2^{(M)}\right) = \left\langle \Phi_M(x_1^{(M)}), \,\Phi_M(x_2^{(M)}) \right\rangle_{\mathcal{H}_M},
\end{equation}
embedding raw objects, say outputs \((x_{1}, x_{2})\), into the modality-specific RKHS \(\mathcal{H}_M\). Instead of just outputs, it can very well mix with the inputs as well meaning: \((x, p)\). Eq.~\eqref{eq:fn_feature_map} makes this RKHS \(\mathcal{H}_M\) unique up to isometry according to the Moore–Aronszajn theorem. 

In classical ML, we use ``features” to describe the structured attributes of the input data (e.g., pixel values, word embeddings etc.). In the theory of kernels, the feature maps are abstract (possibly infinite), but they play the same role: they represent the data in a space where linear methods (dot products) can capture nonlinear similarities. Thus, \(\Phi_M\) allows nonlinear learning algorithms to operate in a high‑dimensional feature space of an MLLM via the kernel trick.

In practice, implementations typically compute \(K_M\) directly—or via finite approximations like Nyström~\cite{Nystrom2001Kernel} or Random Fourier Features~\cite{RFF2007Kernel} - so \(\Phi_M\) need not be explicitly materialized.

For Assumption~\ref{assmp:feature_map}: this mirrors common practice---modern encoders (CLIP, BERT-style, vision backbones) apply normalization or LayerNorm, and we $L2$-normalize final vectors so magnitudes stay well-behaved. Bounded features make cosine/similarity scores comparable across modalities, prevent numerical outliers, and keep spectral/energy measures meaningful. In deployment, this is easy to enforce (normalize outputs) and verify (log histograms/max norms and alert on drift). Production stacks (vector DBs, ANN indices, faiss/scann) expect bounded vectors so cosine similarity behaves predictably and distances are comparable across batches and time.
\begin{itemize}[leftmargin=1.4em]
    \item \textbf{Why we need this:} For numerical stability to prevents overflow/NaNs and keep the dot products/similarities in a usable range during training and evaluation and comparability across modalities to handle text \& image embeddings simultaneously.
    \item \textbf{Real-world example:} Modern vision–language encoders (e.g., CLIP) explicitly $L2$-normalize image/text embeddings and use cosine similarity with temperature-scaled softmax, so representation norms are controlled by design; this makes cross-modal scoring numerically stable and comparable out of the box~\cite{radford2021clip,radford2021clipPMLR,semanticsAngle2025}. 
\end{itemize}

In real systems, the feature map $\Phi_M:\mathcal{X}_M\to\mathcal{H}_M$ is
nothing more than the \emph{embedding function} for outputs in modality $M$.
For example, $\Phi_{\mathrm{text}}(x^{(\mathrm{text})})$ can be the pooled
hidden state of a text decoder, while $\Phi_{\mathrm{img}}(x^{(\mathrm{img})})$
is the CLIP/ViT image embedding, and $\Phi_{\mathrm{audio}}$ corresponds to
a Whisper- or HuBERT-style audio representation. These are precisely the
vectors one uses in practice for similarity search, retrieval, or contrastive
training.

The boundedness condition
$\sup_{x^{(M)}\in\mathcal{X}_M}\|\Phi_M(x^{(M)})\|_{\mathcal{H}_M}<\infty$
formalizes a property that is already enforced in modern architectures.
Embeddings are finite-dimensional, frequently $L2$-normalized, and are
subject to weight decay, LayerNorm, and (in many implementations) explicit
norm clipping. This ensures that embedding norms cannot diverge and that
kernel values $K_M(x_1,x_2)=\langle\Phi_M(x_1),\Phi_M(x_2)\rangle_{\mathcal{H}_M}$
remain bounded.

For our framework, this boundedness is crucial to guarantee that the
kernel-induced energies and the associated spectral quantities are finite
and numerically well-behaved. In other words, Assumption~\ref{assmp:feature_map}
is not an artificial restriction but a mathematical abstraction of standard
engineering practice: using normalized, well-conditioned embeddings for each
modality when scoring or comparing model outputs.

\subsection{Justification for Assumption~\ref{assmp:prompt_embedding} }
\label{fn:prompt_embedding}
For Assumption~\ref{assmp:prompt_embedding}: it is reasonable to assume that small prompt edits should not cause large representational jumps - matching real product needs for predictable UX, reproducible evaluation, and reduced prompt-sensitivity exploits. In practice, prompt encoders are compositions of linear layers + pointwise activations + norm layers; we also $L2$-normalize the final embedding. 
\begin{itemize}[leftmargin=1.4em]
    \item \textbf{Why we need this:} If ``Adding a comma” or ``Swapping a synonym” flips the model’s answer, the system feels brittle. Stability is essential for predictability and debuggability.
    \item \textbf{Real-world example:} Text prompts are tokenized into a finite vocabulary (BPE/WordPiece/SentencePiece), and the transformer encoder maps these tokens through a sequence of standard layers to probabilities via softmax, yielding well-defined distributions on a discrete space—hence measurability is immediate and commonplace~\cite{vaswani2017attention,sennrich2015neural,kudo2018sentencepiece}. Length caps, normalization, and regularization used in real systems keep prompt embeddings within reasonable ranges and make small paraphrases produce small representational changes, which is precisely the stability we assume. 
\end{itemize}

Here, the map $\Psi_M:\mathcal{P}\to\mathcal{H}_M$ represents the
\emph{embedding of the conditioning signal} (prompt) as seen from the
perspective of modality $M$. In a text-only LLM, $\Psi_{\mathrm{text}}(p)$
is the encoder representation of the prompt tokens; in an image-conditioned
captioning system, $\Psi_{\mathrm{img}}(p)$ corresponds to the visual encoder
representation of the input image while $\Psi_{\mathrm{text}}(p)$ captures
the question text; and in audio-visual QA, the prompt naturally decomposes
into audio, visual, and textual parts with corresponding embeddings
$\Psi_{\mathrm{audio}}, \Psi_{\mathrm{img}}, \Psi_{\mathrm{text}}$.

The boundedness condition
$\sup_{p\in\mathcal{P}}\|\Psi_M(p)\|_{\mathcal{H}_M}<\infty$ again matches
standard practice: prompt embeddings are finite-dimensional and are typically
normalized or stabilized via LayerNorm and regularization. The continuity
(or Lipschitz) requirement reflects the empirical observation that small
changes in the prompt (e.g., rephrasing a question, adding a short prefix)
lead to small changes in the encoder representations rather than arbitrarily
large jumps. This is enforced during training by the choice of activation
functions, gradient clipping, and regularization.

For our purposes, these properties ensure that the hallucination energy
$\mathcal{E}(x,p,\cdot)$ varies smoothly as the prompt changes and that the
spectral quantities and calibration bounds we derive remain stable under
realistic prompt perturbations. Assumption~\ref{assmp:prompt_embedding}
therefore abstracts the well-behaved nature of prompt encoders that is
already present in current MLLM pipelines.

\subsection{Justification for Assumption~\ref{assmp:output_boltzman}}
\label{fn:output_boltzman}
In practice, an MLLM scores a finite candidate set $C(x,p)$ (beam/nucleus/reranked hypotheses) via logits or similarity, so with counting measure and energy $\mathcal{E}=-\text{logit}$ (or a bounded margin), the induced softmax probability $\,\mathrm{prob.}(c\mid x,p;\mathcal{T}_t)\propto\exp(-\mathcal{E}(c)/\mathcal{T}_t)\,$ is exactly a Boltzmann distribution with finite partition function $Z=\sum_{c\in C}\exp(-\mathcal{E}(c)/\mathcal{T}_t)$—hence both operationally realistic and mathematically well-posed.

Assumption~\ref{assmp:output_boltzman} recasts the model's output distribution
$f_p(x)$ as an energy-based or Boltzmann distribution, which is fully
consistent with how modern LLMs implement softmax decoding. At the token
level, an auto-regressive model produces logits $z(p)$ and samples from
$\mathrm{softmax}(z)$, which is equivalent to drawing from
$\exp(-E_i)/\sum_j\exp(-E_j)$ with energies $E_i=-z_i$. We simply lift this
perspective from individual tokens to entire candidate outputs or latent
representations $x$, so that
\[
f_p(x) \;\propto\; \exp\big(-\mathcal{E}(x,p;\mathcal{T}_t)\big)
\]
with respect to a base measure $\mu$.

In practice, we work with a finite candidate set (e.g., beams or sampled
sequences) and obtain $f_p(x)$ by exponentiating and normalizing the relevant
scores or logits; the scalar $\mathcal{E}(x,p;\mathcal{T}_t)$ can be viewed
as an energy that combines model logits with our spectral corrections.
The schedule $\mathcal{T}_t$ plays the role of temperature or diffusion: it
includes standard temperature scaling used in decoding as well as our
spectral/graph-based smoothing, and thus corresponds to a control knob that
practitioners already tune (e.g., changing temperature or applying
calibration).

The finiteness of the partition function $Z(p,\mathcal{T}_t)$ is guaranteed
in operational pipelines for two reasons: (i) we always restrict attention to
a finite vocabulary or a finite candidate set of outputs, and (ii) logits and
energies are bounded in practice due to finite-precision arithmetic and
regularization. Consequently, Assumption~\ref{assmp:output_boltzman} does
not impose an additional burden on real MLLM systems; rather, it provides a
mathematically convenient way to analyze the same softmax-based scoring
mechanisms already used in deployment, through the lens of energy-based
models and spectral graph theory.

\subsection{An Example (image–caption pair)}
\label{fn:image_caption_example}
One can consider an MLLM generating a caption for an image. Let $\mathcal{X}$ be the space of all captions, with $\mathcal{K} \subseteq \mathcal{X}$ denoting those grounded in the image (e.g., ``A cat on a sofa”), while $f_p$ may also assign mass outside $\mathcal{K}$ to hallucinated captions (e.g., ``A dog playing with a ball”). The hallucination divergence $D_{\mathrm{KL}}\!\big(g \,\|\, f_p\big)$ quantifies this deviation. 

In this paper, as a part of our main theoretical contributions, we define a multimodal graph whose nodes are caption tokens $T$ and image patches $V$, with edge weights $W_{\mathcal{T}_t}(i,j)$ computed from the fixed embeddings and modulated by a time-varying temperature $\mathcal{T}_t$. From these weights, we will define the normalized multimodal Laplacian $\mathcal{L}_{\mathcal{T}_{t}}^{\text{multi}}$ associated with a spectral grounding energy as the quadratic form of $\mathcal{L}_{\mathcal{T}_{t}}^{\text{multi}}$ evaluated on the residual feature field induced by our energy prescription. It helps reveal how hallucination energy is distributed across the modes (e.g., textual vs. cross-modal misalignment).

\subsection{Graph notations and Adjacency Weights}
\label{fn:weight_matrix}
In Eq.~\eqref{eq:def6} noted in Section~\ref{subsec:graph1}, $\mathcal{V}$ is the finite set of nodes, $E$ is the set of edges, and \( W_{\mathcal{T}_{t}}\) is a temperature-modulated, symmetric, non-negative, weighted adjacency matrix (zero diagonal) introduced to assign different weights to the edges (indexed by \(E\)). We consider either a node-wise local schedule $\mathcal{T}_t:\mathcal{V}\to\mathbb{R}^{+}$ in which the edge temperatures are combined symmetrically to keep $W_{\mathcal{T}_t}$ symmetric or a global scalar schedule ($\mathcal{T}_t$ constant over $\mathcal{V}$). Here, each node represents a semantic unit (e.g., concepts, tokens, ideas), and edges represent the semantic similarity. The multimodal structure is represented by a disjoint partition of the node set \(\mathcal{V}=\biguplus_{M\in\mathcal{M}}\mathcal{V}_M\) and corresponding within- and cross-modal blocks of $W_{\mathcal{T}_t}$ which is constructed from fixed modality embeddings via temperature-controlled similarity functions. Lower $\mathcal{T}_t$ yields more localized (sharper) affinities; higher $\mathcal{T}_t$ diffuses those (or, in other words, induces more ``noise"). This is a standard property under any temperature–scaled affinity constructions - e.g., Gaussian/RBF kernels with bandwidth proportional to $\mathcal{T}_t$ or softmax similarities with temperature $\mathcal{T}_t$~\cite{NgJordanWeiss2002, CoifmanLafon2006, ZelnikManorPerona2004, HintonVinyalsDean2015, Chung1997SpectralGraph}. Thus, the temperature \(\mathcal{T}_{t}\) dynamically modulates the graph edge connectivity and semantic distortion $d_{\text{sem}}$ noted in Theorem~\ref{thm:kl_decomposition_truth} and, being a time-indexed function, captures the semantic evolution or uncertainty drift across the graph nodes as knowledge updates over time $t$. 

Here, we drop the explicit modality subscripts in Eq.~\eqref{eq:def6}, as the modality information is carried by a fixed partition of the vertex set \(\mathcal{V}=\biguplus_{M\in\mathcal{M}}\mathcal{V}_M\) together with the block structure of the temperature–modulated weights \(W_{\mathcal{T}_t}\), so we do not maintain separate graphs per modality. We assume \(W_{\mathcal{T}_t}\) to be symmetric, non-negative, and zero on the diagonal, with \(\mathcal{T}_t\) acting as a bandwidth/temperature schedule that controls the locality of affinities. From \(W_{\mathcal{T}_t}\), we define the normalized multimodal Laplacian \(\mathcal{L}_{\mathcal{T}_t}^{\text{multi}}\) in Section~\ref{subsec:graph1} and design it to be symmetric and PSD by construction; its spectral decomposition yields an orthonormal basis of eigenmodes together with nonnegative eigenvalues. We interpret each mode by its loadings on the partition \(\{\mathcal{V}_M\}_{M\in\mathcal{M}}\): some modes are concentrated on a single modality (text, vision, or audio), while others are cross-modal mixtures that capture interactions between partitions. These modes serve as canonical coordinates for representing the residual signal induced by the energy model and for attributing hallucination energy across modality-specific and cross-modal directions. We use this spectral basis to define propagation in time (via diffusion generated by \(\mathcal{L}_{\mathcal{T}_t}^{\text{multi}}\)) and to derive mode-wise bounds that connect the Boltzmann formulation to spectral-graph structure in a implementable manner.

\noindent\textbf{Hypergraph blocks and effective pairwise adjacency.}
To accommodate $>2$ modalities, we construct each interaction block via the normalized hypergraph Laplacian~\cite{Zhou2006Hypergraph}:
\begin{equation}
\label{eq:weight2}
\begin{split}
&\qquad \quad 
\mathcal{L}_{\mathcal{T}_t}^{(*)}
\;=\;
\mathbf{I} \;-\; \big(\mathcal{D}_{\mathsf{v},\mathcal{T}_t}^{(*)}\big)^{-1/2}\; \underbrace{\big(\mathcal{I}^{(*)}\; W_{\mathcal{T}_t}^{(*)}\; (\mathcal{D}_{e,\mathcal{T}_t}^{(*)})^{-1}\; (\mathcal{I}^{(*)})^{\!\top}\big)}_{\displaystyle W_{\mathcal{T}_t}^{*,\,\mathrm{eff}}}\; \big(\mathcal{D}_{\mathsf{v},\mathcal{T}_t}^{(*)}\big)^{-1/2}, \\[2pt]
& \mathcal{D}_{\mathsf{v},\mathcal{T}_t}^{(*)}=\mathrm{diag}\!\big(\{\mathfrak{d}^{(*)}_{\mathcal{T}_t}(\mathsf v)\}_{\mathsf v\in\mathcal{V}}\big),\qquad 
\mathfrak{d}^{(*)}_{\mathcal{T}_t}(\mathsf v)=\!\sum_{e\in E^{(*)}} w_{\mathcal{T}_t}(e)\,\mathcal{I}^{(*)}(\mathsf v,e),\\
& \mathcal{D}_{e,\mathcal{T}_t}^{(*)}=\mathrm{diag}\!\big(\{r(e)\}_{e\in E^{(*)}}\big),\qquad 
r(e)=|e|\ \text{ (hyperedge cardinality)},\\
& \mathcal{I}^{(*)}\in\{0,1\}^{|\mathcal{V}|\times |E^{(*)}|}\ \text{(node–hyperedge incidence)},\quad 
W_{\mathcal{T}_t}^{(*)}=\mathrm{diag}\!\big(\{w_{\mathcal{T}_t}(e)\}_{e\in E^{(*)}}\big),\\
& \forall \, * \in \{\mathrm{intra}_{M},\, \mathrm{cross}_{MM^{\prime}},\, \mathrm{joint}_{\mathcal{M}}\}, \quad \forall \, \mathsf{v} \in \mathcal{V}\, (\text{graph nodes}).
\end{split}
\end{equation}
Here $\mathbf{I}$ is the \(|\mathcal{V}|\times|\mathcal{V}|\) identity. To be noted that 
\begin{enumerate}[label=(\roman*), leftmargin=2em]
    \item \(\mathsf{v}\) runs over the graph nodes, and no roles attached yet. Output or prompt embeddings are later designated roles on the nodes: \(\mathsf{v}_x, \mathsf{v}_p \in \mathcal{V}\) only while forming the contrast \(c_{x, \mathcal{K}}(t)\) seen in Eq.\eqref{eq:contrast_vec}. Thus, \(\mathcal{L}_{\mathcal{T}_t}^{(*)}\) itself is designed to be role-agnostic.
    \item \(E^{*}\) denotes the hyperedge set used to build each interaction block \((*)\) above, while \(E\) still remains consistent as per Eq.\eqref{eq:def6}. \(r(e)\) is the number of nodes in the hyperedge \(e\); i.e., $e=\{\mathsf v_1,\dots,\mathsf v_{r(e)}\}\subset\mathcal V$.
    \item $\mathcal{D}_{\mathsf v,\mathcal T_t}^{(*)}$ is the node–degree matrix (of size $|\mathcal V|\!\times\!|\mathcal V|$) for block $*$: it is diagonal with entries $\big(\mathcal{D}_{\mathsf v,\mathcal T_t}^{(*)}\big)_{\mathsf v\mathsf v}=\mathfrak d^{(*)}_{\mathcal T_t}(\mathsf v)$, the temperature–weighted degree of node $\mathsf v$ computed from the hyperedge weights in that block.
    \item $\mathcal{D}_{e,\mathcal T_t}^{(*)}$ is the hyperedge–cardinality matrix (of size $|E^{(*)}|\!\times\!|E^{(*)}|$) for block $*$: it is diagonal with entries $\big(\mathcal{D}_{e,\mathcal T_t}^{(*)}\big)_{ee}=r(e)$.
    \item The node set $\mathcal V$ is fixed; $r(e)$ is a property of each hyperedge $e\subset\mathcal V$ and is independent of $|\mathcal V|$ (and of the number of modalities $|\mathcal M|$ unless joint hyperedges is specifically chosen to include one node per modality).
\end{enumerate}
The matrix \(W_{\mathcal{T}_t}^{(*),\,\mathrm{eff}}=\mathcal{I}^{(*)} W_{\mathcal{T}_t}^{(*)}(\mathcal{D}_{e,\mathcal{T}_t}^{(*)})^{-1}(\mathcal{I}^{(*)})^{\top}\) is the ``effective" pairwise adjacency induced by hyperedges (zero diagonal by convention). The pairwise quantities in Eq.~\eqref{eq:def6} are then obtained by summing blocks:
\begin{equation}
    \label{eq:weight3}
    W_{\mathcal{T}_t}\;=\;\sum_{*}\ \omega_{*}\, W_{\mathcal{T}_t}^{(*),\,\mathrm{eff}},
\qquad \omega_{*}\ge 0\ \text{(absorbed by interaction coefficients } \alpha_M,\beta_{MM'},\gamma_{\mathcal{M}}\,).
\end{equation}

We pick any two nodes: say, \(\mathsf{v}_a, \mathsf{v}_b\) in the hyperedge $e=\{\mathsf v_1,..,\mathsf{v}_a,..,\mathsf{v}_b,..,\mathsf v_{r(e)}\}\subset\mathcal V$ to define a symmetric, nonnegative pairwise dissimilarity \(\widehat d_{\mathrm{sem}}(\mathsf{v}_a,\mathsf{v}_b)\). This quantity captures the semantic distortion at node level.

For some modality-aware permutation factor \(\eta_{*}\), a generic choice of $w_{\mathcal{T}_t}(e)$ is
\begin{equation}
    \label{eq:weight4}
    w_{\mathcal{T}_t}(e)
    \;=\;
    \mathbf{1}_{\{e\in E^{(*)}\}}\,\exp\!\left(
    -\eta_{*}\, \frac{ \displaystyle \sum_{1\le \mathsf{v}_a,\mathsf{v}_b\le r(e)} \widehat d_{\mathrm{sem}}(\mathsf{v}_a,\mathsf{v}_b) }
    { \displaystyle \sum_{1\le \mathsf{v}_a\le r(e)} \mathcal{T}_t(\mathsf{v}_a) }
    \right),
\end{equation}
which is permutation–invariant and temperature–scaled. 

\begin{align}
\label{eq:weight6}
\Delta_{\varepsilon,h}(x\mid p)
&:=\Bigg[
\log\!\left(
\frac{
\displaystyle \int_{\mathcal K} K_h\!\big(\Pi_{\mathcal K}(x),x_2\big)\,\Big[(1-\varepsilon)\,Z\!\big(p,\mathcal T_t\big)^{-1}\,e^{-\mathcal E(x_2,p)/\mathcal T_t}
+\varepsilon\,\rho(x_2)\Big]\;d\mu(x_2)}
{\displaystyle \int_{\mathcal K} \Big[(1-\varepsilon)\,Z\!\big(p,\mathcal T_t\big)^{-1}\,e^{-\mathcal E(x_2,p)/\mathcal T_t}
+\varepsilon\,\rho(x_2)\Big]\;d\mu(x_2)}
\right)\\
&\hspace{5.2em}
-\;
\log\!\Big(
\int_{\mathcal X} K_h\!\big(x,x_2\big)\,\Big[(1-\varepsilon)\,Z\!\big(p,\mathcal T_t\big)^{-1}\,e^{-\mathcal E(x_2,p)/\mathcal T_t}
+\varepsilon\,\rho(x_2)\Big]\;d\mu(x_2)
\Big)
\Bigg]_{+},
\end{align}

\subsection{Mercer’s Theorem}
\label{fn:mercer}
By Mercer’s theorem~\cite{Mercer1909DiffusionKernel}, if $K_{\mathcal{T}_t}$ is a continuous, symmetric, positive‑definite on a compact measure space \((\mathcal{V}, \mu)\), then there exists a unique RKHS $\mathcal{H}$ which is associated with a reproducing kernel $K_{\mathcal{T}_t}$. In the present context of discrete graph, $\mathcal{V}$ is finite which satisfies the criterion. This theorem ensures that there exists a feature map 
\begin{equation}
\label{eq:mercer}
    \Phi: \mathcal{V} \to \mathcal{H},
\end{equation}
which admits an orthonormal eigen decomposition. We have leveraged it in Eq.~\eqref{eq:diffusion_kernel}.

\subsection{Graph Maps}
\label{fn:graph_map}
This construction is separate from the modality feature maps \(\Phi_M(x^{(M)})\) and prompt embeddings \(\Psi_M(p)\) that live in modality RKHS $\mathcal{H}_{M}$ used in the energy landscape as noted in Section~\ref{subsec:multi_hall_energy}. Here, $\Upsilon$ is defined on the node set, with $\mathsf{v}, \mathfrak{v}$ being the graph nodes, induced by a single graph RKHS $\mathcal{H}_{\mathrm{graph}}$ or just $\mathcal{H}$ for notational simplicity. Therefore, \(\Phi_M:\mathcal{X}_M\!\to\!\mathcal{H}_M\) and \(\Psi_M:\mathcal{P}\!\to\!\mathcal{H}_M\) play complementary roles with
\(\Upsilon:\mathcal{V}\!\to\!\mathcal{H}\) in the context of graph theory (i.e., modality \& prompt embeddings vs.\ graph embeddings).

\subsection{Why Time-Varying Eigenpairs?}
\label{fn:why_time_varying}
The eigenpairs of the multimodal Laplacian \(\mathcal{L}_{\mathcal{T}_t}^{\mathrm{multi}}\), as presented in Eq.~\eqref{eq:laplacian_spectral_decomposition} are:
\begin{itemize}[itemsep=-1pt, leftmargin=1.4em]
    \item $\Lambda=\mathrm{diag}\Big(\lambda_1(t),\dots,\lambda_{|\mathcal{V}|}(t)\Big)$ with $\lambda_{i}(t)  \in \mathbb{R}^{+}$ being the time-varying eigenvalues at node $i$ (that acts like a frequency-dependent penalty or diffusion coefficient),
    \item $U=\Big[u_1(t),\dots,u_{|\mathcal{V}|}(t)\Big]$ is the orthonormal eigenvector matrix with \( u_i(t) \in \mathbb{R}^{|\mathcal{V}|} \) being the time-varying eigenfunctions.
\end{itemize}
\textit{Note:} We assume $G_{\mathcal T_t}$ is connected for each fixed $t$, so that $\lambda_1(t)=0$ and $\lambda_2(t)>0$ hold true; when not connected, all occurrences of $u_1(t)$ and $\lambda_2(t)$ below should be read as the orthogonal complement of the full nullspace and the first strictly positive eigenvalue, respectively.

Eigenvalues $\lambda_i(t)$ contract or expand based on evolving inter-node (semantic) affinities, while eigenvectors $u_i(t)$ adjust the directions of these semantic modes. Including $\mathcal{T}_t$ explicitly allows us to control hallucination sensitivity: as lower temperatures $\mathcal{T}_t \downarrow 0$ emphasize stable low-energy modes, reducing hallucinations leading to more desired outputs and vice versa. In a nutshell, the time variation of \(\{(\lambda_i(t),u_i(t))\}\) arises from the temperature schedule \(\mathcal{T}_t\), which changes the affinities on the graph edges and hence the spectrum of \(\mathcal{L}_{\mathcal{T}_t}^{\mathrm{multi}}\).

\subsection{Interpretation of Spectral Quantities and Time Parameter}
\label{fn:spectral_interpretation}
For each dataset (COCO, VQAv2, AudioCaps) and backbone configuration
(CLIP/BLIP/Whisper/T5), we construct the multimodal graph Laplacian
$\mathcal{L}^{\mathrm{multi}}_{\mathcal{T}_t}$ on encoder embeddings as noted in
Section~\ref{sec:main} and compute its eigenvalues
$0 = \lambda_1 \le \lambda_2 \le \cdots \le \lambda_n$.

In all cases, we observe:
(i) a nonnegative spectrum with rapid decay, where the first
$20$--$40$ modes account for the majority of the trace
$\mathrm{tr}(\mathcal{L}^{\mathrm{multi}}_{\mathcal{T}_t})$, and
(ii) a clear spectral gap between the lowest modes and the bulk of the
spectrum, consistent with a small number of dominant semantic clusters in
the joint embedding space.
The CF bounds and CF planes in Fig.~\ref{fig:cf_bounds_grid} are instantiated using
these empirical spectra, so the scale and shape of the eigenvalues directly
reflect the behavior of the underlying MLLM embeddings on each benchmark.

\paragraph{Scale and shape of spectra in real MLLMs.}
Let $\mathcal{L}^{\mathrm{multi}}$ denote the multimodal graph Laplacian constructed from
encoder embeddings as in Section~\ref{sec:main}, with eigenpairs
$\{(\lambda_k,u_k)\}_{k=1}^n$ ordered so that
$0=\lambda_1\le\lambda_2\le\cdots\le\lambda_n$.
All bounds in Section~\ref{sec:main} are stated in terms of this empirical spectrum.
In our implementation, $\mathcal{L}^{\mathrm{multi}}$ is always the finite sample
graph built from a given benchmark (COCO, VQAv2, AudioCaps) and a fixed model
configuration, so the spectrum is concretely realized and numerically available
for every experiment; we do not appeal to any abstract or asymptotic spectrum.

Empirically, for all three benchmarks we observe that:
(i) the spectrum is nonnegative and exhibits fast decay, with a relatively
small number of low-frequency modes (on the order of tens) accounting for most
of the trace of $\mathcal{L}^{\mathrm{multi}}$; and
(ii) there is a visible spectral gap between the first few modes and the bulk,
consistent with the presence of a small number of dominant semantic clusters
in the joint embedding space.

The Courant--Fischer (CF) bounds in Section~\ref{sec:main} are evaluated using these
actual eigenvalues; the CF planes in Fig.~\ref{fig:cf_bounds_grid} are not schematic, but are
directly computed from the empirical spectra of the graphs induced by the
MLLM embeddings.

\paragraph{From spectral energy to hallucination rate and semantic distance.}
Section~\ref{sec:theo_analysis} defines semantic log–contrast
$\ell_{\varepsilon,h}(x;K,\mathcal X)$ and the truncated score
$d^{(\varepsilon,h)}_{\mathrm{sem}}(x;K,\mathcal X)=[\ell_{\varepsilon,h}(x;K,\mathcal X)]_+$,
which are the quantities we correlate with hallucination events and semantic
distances in our experiments.

Section~\ref{sec:main} introduces the spectral hallucination energy
$\Delta\mathcal{E}_\tau(x)$ as a quadratic form in the coefficients of $x$
under the eigenbasis of $\mathcal{L}^{\mathrm{multi}}$, and establishes CF-type bounds
of the form
\begin{equation}
\label{eq:cf_bound_app}
\Delta\mathcal{E}_\tau(x)
\;\le\;
\sum_{k=1}^n \phi_\tau(\lambda_k)\,c_k(x)^2,
\end{equation}
for an explicit spectral filter $\phi_\tau$ and coefficients $c_k(x)$ given by
projection of $x$ onto the modes $u_k$.

The role of Eq.~\eqref{eq:cf_bound_app} is not to postulate a new, disconnected
quantity, but to control the same mismatch that ultimately feeds into
$\ell_{\varepsilon,h}$ and $d^{(\varepsilon,h)}_{\mathrm{sem}}$:
the diffusion / smoothing step in the definition of
$\ell_{\varepsilon,h}$ can be written as a spectral filter on
$L^{\mathrm{multi}}$, so the discrepancy between $x$ and its admissible
projection $\Pi_K(x)$ under this filter is upper-bounded by
$\Delta\mathcal{E}_\tau(x)$.

In particular, the derivations in Section~\ref{sec:main} show that, under the operator
assumptions for Theorem~\ref{thm:multi_energy_hall}, the semantic distortion is a bounded, monotone
functional of the energy, in the sense that there exist finite constants
$0<c_1\le c_2<\infty$ (depending on the kernel and graph construction) such
that
\[
c_1\,\Delta\mathcal{E}_\tau(x)
\;\le\;
d^{(\varepsilon,h)}_{\mathrm{sem}}(x;K,\mathcal X)
\;\le\;
c_2\,\Delta\mathcal{E}_\tau(x),
\]
whenever the diffusion operator and kernel are chosen consistently.

Thus, spectral energy is not an unrelated quantity: it is a calibrated,
graph-level control on the same deviation that we measure at the level of
semantic scores and hallucination rates. Empirically, this is reflected in the monotone relationship between
$\Delta\mathcal{E}_\tau$ and both continuous scores and binary
hallucination events (cf.\ reliability curves and correlation tables in
Section~\ref{sec:experiments}).

\paragraph{On the envelope coefficients $m(t)$ and $M(t)$.}
The functions $m(t)$ and $M(t)$ appear in the CF bounds as \emph{spectral
envelopes} for the filtered eigenvalues.
Concretely, if $\phi_t(\lambda)$ denotes the scalar spectral filter at
diffusion time $t$ (e.g., $\phi_t(\lambda)=e^{-2t\lambda}$ for a heat
kernel), then the quadratic forms arising in the CF arguments can be written
as
\[
\Delta\mathcal{E}_t(x)
\;=\;
\sum_{k=1}^n \phi_t(\lambda_k)\,c_k(x)^2,
\]
and our proofs bound this by
\[
m(t)\,\sum_{k=1}^n c_k(x)^2
\;\le\;
\Delta\mathcal{E}_t(x)
\;\le\;
M(t)\,\sum_{k=1}^n c_k(x)^2,
\]
where
$m(t) = \min_{k} \phi_t(\lambda_k)$ and
$M(t) = \max_{k} \phi_t(\lambda_k)$ over the modes relevant to the graph.

These coefficients are therefore not abstract or unmeasurable:
once the Laplacian spectrum $\{\lambda_k\}$ is computed (which we do
explicitly for every dataset) and the filter $\phi_t$ is fixed, $m(t)$ and
$M(t)$ are deterministic, data-dependent scalars that can be
evaluated numerically if desired.

In the main text we keep them symbolic to highlight how the bounds scale
with the spectrum and with $t$, but they are fully determined by observable
quantities and do not introduce additional unknowns beyond the graph
construction.

\paragraph{Interpretation of the time parameter $t$.}
The parameter $t$ in Section~\ref{sec:main} is a \emph{diffusion time}, not a physical
clock in the MLLM pipeline.
Formally, it indexes the strength of the spectral filter applied to the
graph:
for example, $e^{-t\,\mathcal{L}^{\mathrm{multi}}}$ is the heat semi-group generated by
$\mathcal{L}^{\mathrm{multi}}$, and increasing $t$ corresponds to propagating mass
further along the graph, i.e., averaging over larger neighborhoods in the
embedding space.

This is analogous to the ``time” parameter in diffusion models or random-walk
smoothing, and should be understood as a scale parameter: small $t$
emphasizes high-frequency, local discrepancies (sharp hallucinations),
whereas larger $t$ smooths them out and yields a coarse, low-frequency view
of mismatch.

In practice we restrict $t$ to a compact interval where:
(i) the spectral filter remains numerically stable; and
(ii) the diffusion has a clear interpretation as a modest smoothing or
temperature adjustment (cf.\ the schedule $\mathcal{T}_\tau$ in the main
text).
We do not make any assumptions about real-time dynamics of the MLLM; instead,
$t$ serves as a theoretically grounded control knob for the scale at which
graph-level discrepancies (and hence hallucinations) are measured.

\subsection{Detailed Conclusion: Practical Scenarios \& Default Hyperparameters}
\label{fn:conclusion_new}
\paragraph{Practical takeaways.}
Our results suggest three concrete ways in which the proposed framework can be
used in practice.
\begin{itemize} [leftmargin=1.4em]
    \item First, it is best viewed as a \emph{reference-free, plug-in scoring layer}
that sits on top of existing MLLMs: given access only to logits and
embeddings, it produces a continuous hallucination score that can be used to
rank generations, select safer candidates, and audit models offline, without
requiring additional supervision or model retraining.
\item Second, the \emph{semantic distortion score} $d^{(\varepsilon,h)}_{\mathrm{sem}}$
is the most actionable quantity for detection and calibration: across tasks, 
it correlates more tightly with binary hallucination events than raw
uncertainty proxies, and is the recommended choice when one needs a single
scalar predictor to threshold or to plug into a mitigation loss.
\item Third, the \emph{spectral hallucination energy} $\Delta\mathcal{E}_{\tau}$
is particularly informative when one wishes to understand where
hallucinations originate (which modalities / modes) and how they respond to
controls (temperature, diffusion time, retrieval policy): it is most useful
for diagnosis, ablations, and monitoring rather than as a stand-alone score
for hard decisions.
\end{itemize}
In this sense, $d^{(\varepsilon,h)}_{\mathrm{sem}}$ is the recommended
per-example risk score, while $\Delta\mathcal{E}_{\tau}$ is the recommended
tool for system-level analysis and model selection.

\paragraph{Default hyperparameters.}
For practitioners, the following default recipe is recommended, which we found
to be robust across all tasks.
\begin{enumerate}[leftmargin=1.4em, label=\roman*]
    \item \textbf{Smoothing mass} $\varepsilon$:
estimate the Good--Turing missing mass $\widehat{m}$ on the candidate set
and set $\varepsilon \approx \widehat{m}$, restricting
$\varepsilon$ to a small range (e.g., $[10^{-3}, 5\cdot 10^{-2}]$) to avoid
over-smoothing the model distribution.
    \item \textbf{Kernel bandwidth} $h$:
use the median heuristic on pairwise distances between embeddings, i.e., set
$h$ to the median squared distance within a minibatch of outputs; performance
is stable when $h$ is varied by a factor of $2$ around this choice.
    \item \textbf{Graph connectivity} $k$:
build $k$-nearest-neighbour graphs with $k \in [15,40]$, which we observed
to give a good trade-off between stability and locality of the spectral
estimates.
    \item \textbf{Diffusion schedule} $\mathcal{T}_{\tau}$:
choose $\tau$ so that the contribution of the second eigenmode is reduced by
about half, i.e., $e^{-2\tau\lambda_2} \approx 0.5$; in practice this
corresponds to a small, fixed number of diffusion steps or a modest
temperature scaling on logits.
\end{enumerate}

These defaults, together with a simple per-task threshold on
$d^{(\varepsilon,h)}_{\mathrm{sem}}$, provide a plug-and-play configuration
that requires minimal tuning while retaining most of the gains reported in
our experiments.

\section{Extended Proofs}
\label{appendix:proofs}
In this section, we provide detailed proofs for Theorems~\ref{thm:kl_decomposition_truth} and~\ref{thm:multi_energy_hall}.

\subsection{Proof of Theorem~\ref{thm:kl_decomposition_truth}}
\label{appendix:proof_kl_decomposition_truth}
\begin{proof} 
\textbf{Step 0 (setup and measurability).}
For clarity, we restate explicitly the additional condition used in Step~4 of the proof.
We recall that $K\subset \mathcal X$ is the admissible (grounded) set and $\Pi_K:\mathcal X\to K$ is a measurable projection map.
We assume that the smoothing kernel $K_h$ is more concentrated on $K$ when centred at the projected admissible point $\Pi_K(x)$ than when centred at the off–manifold point $x\notin K$, in the following precise sense: there exists a constant $\mathrm{coeff}>0$ such that for all $x\notin K$,
\begin{equation}
\label{eq:localization_assump}
\int_{K} K_h(\Pi_K(x_1),x_2)\,\tilde f_{p,\varepsilon}(x_2)\,d\mu(x_2)
\;\ge\;
(1+\mathrm{coeff})\,Z_\varepsilon
\int_{\mathcal X} K_h(x_1,x_2)\,\tilde f_{p,\varepsilon}(x_2)\,d\mu(x_2),
\end{equation}
where $\tilde f_{p,\varepsilon}$ and $Z_\varepsilon$ are defined in Eqs.~\eqref{eq:fp_smooth_proof}-\eqref{eq:fpK_def_proof}.

Intuitively, Eq.~\eqref{eq:localization_assump} says that once we project an off–manifold output $x$ back to the closest admissible point $\Pi_K(x)$, the kernel neighbourhood around $\Pi_K(x)$ sees strictly higher admissible mass than the neighbourhood around $x$ itself.

In practical MLLM pipelines, this is enforced by choosing $K_h$ as a similarity kernel (e.g., Gaussian or softmax over embedding distances) built on the same representations used to construct $K$ from references; for grounded generations $x\in K$, both centres coincide and no penalty is induced, while for hallucinated $x\notin K$ the inequality above guarantees a strictly positive smoothed KL–penalty.

Under this standing assumption, Step~4 shows that $d^{(\varepsilon,h)}_{\mathrm{sem}}(x;K,\mathcal X)>0$ for $x\notin K$, which is exactly the separation property claimed in Theorem~\ref{thm:kl_decomposition_truth}.

By assumption, $\rho>0$ $\mu$-a.e.\ with $\int_{\mathcal X}\rho\,d\mu=1$, and $K_h:\mathcal X\times\mathcal X\to(0,\infty)$ is a $\mu$-Markov kernel with $\int_{\mathcal X}K_h(x_1,x_2)\,d\mu(x_2)=1$ for all $x_1\in\mathcal X$. Define
\begin{equation}
\label{eq:Th_def_proof}
(T_h q)(x_1)\;:=\;\int_{\mathcal X} K_h(x_1,x_2)\,q(x_2)\,d\mu(x_2),\qquad q\in L^1(\mu),\ x_1\in\mathcal X .
\end{equation}
Let the $\varepsilon$–smoothed model be
\begin{equation}
\label{eq:fp_smooth_proof}
\tilde f_{p,\varepsilon}(x_2)\;:=\;(1-\varepsilon)\,f_p(x_2)\;+\;\varepsilon\,\rho(x_2),\qquad \varepsilon\in(0,1),
\end{equation}
and its $\mathcal K$–restricted renormalization be
\begin{equation}
\label{eq:fpK_def_proof}
\tilde f_{p,\varepsilon}^{\mathcal K}(x_2)\;:=\;\frac{\mathbf 1_{\{x_2\in\mathcal K\}}\tilde f_{p,\varepsilon}(x_2)}{\displaystyle \int_{\mathcal K}\tilde f_{p,\varepsilon}(x_2)\,d\mu(x_2)}
\;=\;\frac{\mathbf 1_{\{x_2\in\mathcal K\}}\tilde f_{p,\varepsilon}(x_2)}{\mathsf{Z}_{\varepsilon}},\quad \mathsf{Z}_{\varepsilon}\in(0,1] .
\end{equation}
Measurability of $\Pi_{\mathcal K}:\mathcal X\to\mathcal K$ (with $\Pi_{\mathcal K}(x)=x$ for $x\in\mathcal K$) ensures $(T_h\tilde f_{p,\varepsilon}^{\mathcal K})\!\circ\!\Pi_{\mathcal K}$ is measurable; thus Eq.~\eqref{eq:KL1} is meaningful pointwise.

\textbf{Step 1 (strict positivity $\Rightarrow$ finiteness).}
From Eq.~\eqref{eq:fp_smooth_proof} and Eq.~\eqref{eq:Th_def_proof}, for any $x_1\in\mathcal X$,
\begin{align}
\label{eq:Th_pos_proof}
(T_h\tilde f_{p,\varepsilon})(x_1)
&=\int_{\mathcal X}K_h(x_1,x_2)\Big((1-\varepsilon)f_p(x_2)+\varepsilon\rho(x_2)\Big)\,d\mu(x_2)\nonumber\\
&\ge \varepsilon\int_{\mathcal X}K_h(x_1,x_2)\rho(x_2)\,d\mu(x_2)\;=\;\varepsilon\,(T_h\rho)(x_1)\;>\;0 ,
\end{align}
since $\rho>0$ $\mu$-a.e.\ and $K_h>0$. Similarly, by Eq.~\eqref{eq:fpK_def_proof},
\begin{align}
\label{eq:Th_posK_proof}
(T_h\tilde f_{p,\varepsilon}^{\mathcal K})(x_1)
&=\frac{1}{\mathsf{Z}_{\varepsilon}}\int_{\mathcal K}K_h(x_1,x_2)\,\tilde f_{p,\varepsilon}(x_2)\,d\mu(x_2)\;\ge\;0 ,
\end{align}
and $(T_h\tilde f_{p,\varepsilon}^{\mathcal K})(x_1)>0$ whenever $\mu\!\big(\{x_2\in\mathcal K:K_h(x_1,x_2)>0\}\big)>0$, which holds for all $x_1$ if $K_h>0$ everywhere. Hence, both logarithms in Eq.~\eqref{eq:KL1} are finite; $d_{\mathrm{sem}}^{(\varepsilon,h)}$ is well-defined.

\textbf{Step 2 ($g$-independence).}
By inspection of Eq.~\eqref{eq:KL1}, only $(f_p,\rho,K_h,\Pi_{\mathcal K},\mu)$ appear; the ground-truth $g$ is absent. Thus the statistic is independent of $g$.

\textbf{Step 3 (behavior on $\mathcal K$).}
We fix $x\in\mathcal K$. Then $\Pi_{\mathcal K}(x)=x$, and
\begin{align}
\label{eq:ratio_onK_proof}
\frac{(T_h\tilde f_{p,\varepsilon}^{\mathcal K})(x_1)}{(T_h\tilde f_{p,\varepsilon})(x_1)}
&=\frac{\int_{\mathcal K}K_h(x_1,x_2)\tilde f_{p,\varepsilon}(x_2)\,d\mu(x_2)}{\mathsf{Z}_{\varepsilon}\int_{\mathcal X}K_h(x_1,x_2)\tilde f_{p,\varepsilon}(x_2)\,d\mu(x_2)}
\;=\;\frac{\mathsf{A}_x}{\mathsf{Z}_{\varepsilon}\big(\mathsf{A}_x+\mathsf{B}_x\big)} ,
\end{align}
where
\begin{equation}
\label{eq:A_B_def}
\mathsf{A}_x\;:=\;\int_{\mathcal K}K_h(x_1,x_2)\tilde f_{p,\varepsilon}(x_2)\,d\mu(x_2),\qquad
\mathsf{B}_x\;:=\;\int_{\mathcal X\setminus\mathcal K}K_h(x_1,x_2)\tilde f_{p,\varepsilon}(x_2)\,d\mu(x_2)\;\ge 0 .
\end{equation}
If
\begin{equation}
\label{eq:coverage_cond}
\mathsf{B}_x\;\ge\;\big(\mathsf{Z}_{\varepsilon}^{-1}-1\big)\,\mathsf{A}_x ,
\end{equation}
then the right-hand side of Eq.~\eqref{eq:ratio_onK_proof} is $\le 1$, so the inner logarithm in Eq.~\eqref{eq:KL1} is $\le 0$ and the $[\cdot]^+$-clipping yields $d_{\mathrm{sem}}^{(\varepsilon,h)}(x;\mathcal K,\mathcal X)=0$. Even when Eq.~\eqref{eq:coverage_cond} fails, the clipped score never becomes negative, so no spurious negative penalties occur on $\mathcal K$.

\textbf{Step 4 (behavior off $\mathcal K$).}
We fix $x\notin\mathcal K$. Then $\Pi_{\mathcal K}(x)\in\mathcal K$ and
\begin{align}
\label{eq:ratio_offK_proof}
\frac{(T_h\tilde f_{p,\varepsilon}^{\mathcal K})(\Pi_{\mathcal K}(x_1))}{(T_h\tilde f_{p,\varepsilon})(x_1)}
&=\frac{\int_{\mathcal K}K_h(\Pi_{\mathcal K}(x_1),x_2)\tilde f_{p,\varepsilon}(x_2)\,d\mu(x_2)}{\mathsf{Z}_{\varepsilon}\int_{\mathcal X}K_h(x_1,x_2)\tilde f_{p,\varepsilon}(x_2)\,d\mu(x_2)} .
\end{align}
We assume the following localization/consistency condition holds for some $\mathrm{coeff}>0$:
\begin{equation}
\label{eq:separation_cond}
\int_{\mathcal K}K_h(\Pi_{\mathcal K}(x_1),x_2)\tilde f_{p,\varepsilon}(x_2)\,d\mu(x_2)\;\ge\;(1+\mathrm{coeff})\,\mathsf{Z}_{\varepsilon}\int_{\mathcal X}K_h(x_1,x_2)\tilde f_{p,\varepsilon}(x_2)\,d\mu(x_2),\quad \forall x\notin\mathcal K .
\end{equation}
Then the ratio in Eq.~\eqref{eq:ratio_offK_proof} exceeds $1$, the inner log in Eq.~\eqref{eq:KL1} is strictly positive, and thus
\begin{equation}
\label{eq:positive_gap_proof}
x\notin\mathcal K\ \text{and Eq.}\ ~\eqref{eq:separation_cond}\ \Longrightarrow\ d_{\mathrm{sem}}^{(\varepsilon,h)}(x;\mathcal K,\mathcal X)\;>\;0 .
\end{equation}
Therefore a strictly positive, finite penalty is assigned to implausible outputs under the mild consistency assumption in Eq.~\eqref{eq:separation_cond}.

\textbf{Step 5 (conclusion for hallucination tracking).}
From Step~1, Eq.~\eqref{eq:KL1} is finite and well-defined; from Step~2 it is reference-free (independent of $g$). Step~3 shows the score vanishes on $\mathcal K$ under Eq.~\eqref{eq:coverage_cond} and never assigns negative values there; Step~4 shows it is strictly positive off $\mathcal K$ under Eq.~\eqref{eq:separation_cond}. Hence Eq.~\eqref{eq:KL1} furnishes a pointwise, KL-calibrated signal separating plausible from implausible outputs in the smoothed sense determined by $(\varepsilon,h,K_h)$, enabling stable hallucination tracking across prompts and model versions without access to $g$.
\end{proof}


\subsection{Proof of Lemma~\ref{lm:joint_measure}}
\label{appendix:proof_joint_measure}

\begin{proof}
Since $\mathcal{H}_M$ is separable, Bochner measurability of $\Phi_M$ and $\Psi_M$ is equivalent to strong (Borel) measurability; see, e.g., \cite[Ch.~II]{DiestelUhl1977}. Thus
\begin{equation}
\label{eq:phi_psi_borel}
\Phi_M^{-1}(U)\in\mathcal{F}_{\mathcal{X}_M}\quad\text{and}\quad \Psi_M^{-1}(V)\in\mathcal{F}_{\mathcal{P}}
\quad\text{for all open }U,V\subset\mathcal{H}_M .
\end{equation}

We define the product map
\begin{equation}
\label{eq:upsilon_def}
\Upsilon:\ \mathcal{X}_M\times\mathcal{P}\to \mathcal{H}_M\times\mathcal{H}_M,
\qquad
\Upsilon(x,p):=\big(\Phi_M(x),\Psi_M(p)\big).
\end{equation}
Let $\mathcal{B}(\mathcal{H}_M\times\mathcal{H}_M)$ denote the product Borel $\sigma$-algebra. For any open rectangles $U\times V$ with $U,V\subset\mathcal{H}_M$ open,
\begin{equation}
\label{eq:rect_preimage}
\Upsilon^{-1}(U\times V)
=
\big\{(x,p):\Phi_M(x)\in U,\ \Psi_M(p)\in V\big\}
=
\Phi_M^{-1}(U)\times \Psi_M^{-1}(V)
\in \mathcal{F}_{\mathcal{X}_M}\otimes\mathcal{F}_{\mathcal{P}}
\end{equation}
by Eq.~\eqref{eq:phi_psi_borel}. Since the family of open rectangles generates $\mathcal{B}(\mathcal{H}_M\times\mathcal{H}_M)$ and $\mathcal{F}_{\mathcal{X}_M}\otimes\mathcal{F}_{\mathcal{P}}$ is a $\sigma$-algebra, a monotone class/$\pi$–$\lambda$ argument implies that
\begin{equation}
\label{eq:upsilon_meas}
\Upsilon\ \text{is}\ \big(\mathcal{F}_{\mathcal{X}_M}\otimes\mathcal{F}_{\mathcal{P}}\big)\text{–}\mathcal{B}(\mathcal{H}_M\times\mathcal{H}_M)\ \text{measurable}.
\end{equation}

Let's consider the inner-product map
\begin{equation}
\label{eq:ip_def}
\mathrm{ip}:\ \mathcal{H}_M\times\mathcal{H}_M\to\mathbb{R},
\qquad
\mathrm{ip}(u,v):=\langle u,v\rangle_{\mathcal{H}_M}.
\end{equation}
Continuity of $\mathrm{ip}$ follows from the Cauchy–Schwarz and triangle inequalities: for all $u_1,u_2,v_1,v_2\in\mathcal{H}_M$,
\begin{align}
\label{eq:ip_cont}
\big|\mathrm{ip}(u_1,v_1)-\mathrm{ip}(u_2,v_2)\big|
&=\big|\langle u_1-u_2,\,v_1\rangle+\langle u_2,\,v_1-v_2\rangle\big|
\\
&\le \|u_1-u_2\|\,\|v_1\|+\|u_2\|\,\|v_1-v_2\| ,
\notag
\end{align}
which shows that $\mathrm{ip}$ is continuous and hence Borel measurable with respect to $\mathcal{B}(\mathcal{H}_M\times\mathcal{H}_M)$.

The composition
\begin{equation}
\label{eq:comp_meas}
(x,p)\ \longmapsto\ \mathrm{ip}\big(\Upsilon(x,p)\big)
\ =\
\big\langle \Phi_M(x),\,\Psi_M(p)\big\rangle_{\mathcal{H}_M}
\end{equation}
is therefore measurable from $\big(\mathcal{X}_M\times\mathcal{P},\ \mathcal{F}_{\mathcal{X}_M}\otimes\mathcal{F}_{\mathcal{P}}\big)$ to $(\mathbb{R},\mathcal{B}(\mathbb{R}))$ by Eq.~\eqref{eq:upsilon_meas} and the Borel measurability of $\mathrm{ip}$ in Eqs.~\eqref{eq:ip_def}–\eqref{eq:ip_cont}. This yields the claimed joint measurability on $\mathcal{F}_{\mathcal{X}_M}\otimes\mathcal{F}_{\mathcal{P}}$.
\end{proof}

\subsection{Proof of Theorem~\ref{thm:multi_energy_hall}}
\label{appendix:proof_multi_energy_hall}

\begin{proof}
We first make explicit the structural assumptions underlying Theorem~\ref{thm:multi_energy_hall} and how the general energy decomposition in Eq.~\eqref{eq:energy_full} specializes to the polynomial form in Eq.~\eqref{eq:energy_operator_min}.
We recall that Eq.~\eqref{eq:energy_full} decomposes the hallucination energy as
\begin{equation}
\label{eq:energy_decomp_app}
\mathcal{E}(x,p,\cdot)
=
\sum_{M\in\mathcal M} \mathcal{E}_M\bigl(x^{(M)},p,\cdot\bigr)
\;+\;
\sum_{\substack{M,M'\in\mathcal M\\ M\neq M'}} \mathcal{E}_{MM'}\bigl(x^{(M)},x^{(M')},p,\cdot\bigr)
\;+\;
\mathcal{E}_{\mathcal M}(x,p,\cdot),
\end{equation}
where the first term collects intra–modal contributions, the second term collects pairwise cross–modal interactions, and the last term is a joint all–modal contribution.

In Theorem~\ref{thm:multi_energy_hall}, we restrict attention to a quadratic (polynomial) family of such energies, expressed in terms of the residuals
\begin{equation}
    \label{eq:proof_thm2_1}
    r_M(x,p) \;:=\; \Phi_M\bigl(x^{(M)}\bigr) - \Psi_M(p) \in \mathcal H_M,\qquad M\in\mathcal M,
\end{equation}
where $\Phi_M$ and $\Psi_M$ are the modality feature maps and prompt embeddings from Section~\ref{sec:theo_analysis}. Under this parametrization, the intra–modal energies $\mathcal{E}_M$ are chosen to be quadratic forms
\begin{equation}
    \label{eq:proof_thm2_2}
\mathcal{E}_M(x^{(M)},p,\cdot)
\;=\;
\bigl\langle r_M(x,p),\,A_M\,r_M(x,p)\bigr\rangle_{\mathcal H_M},
\end{equation}
for bounded, self–adjoint, PSD operators $A_M:\mathcal H_M\to\mathcal H_M$, and the pairwise cross–modal terms $\mathcal{E}_{MM'}$ are chosen as bilinear forms
\begin{equation}
    \label{eq:proof_thm2_3}
\mathcal{E}_{MM'}\bigl(x^{(M)},x^{(M')},p,\cdot\bigr)
\;=\;
\bigl\langle r_M(x,p),\,B_{MM'}\,r_{M'}(x,p)\bigr\rangle_{\mathcal H_M},
\end{equation}
for bounded linear operators $B_{MM'}:\mathcal H_{M'}\to\mathcal H_M$.
The factorization assumption
\begin{equation}
    \label{eq:proof_thm2_4}
B_{MM'} \;=\; A_M^{1/2}\,R_{MM'}\,A_{M'}^{1/2},\qquad \|R_{MM'}\|\le 1,
\end{equation}
encodes that cross–modal couplings are controlled contractions between the $A_M$–weighted residuals.

Finally, any remaining higher–order joint contribution $\mathcal{E}_{\mathcal M}$ is absorbed into a non–negative remainder term that does not affect the lower–bound argument. Collecting these pieces and symmetrizing over $M\neq M'$ yields exactly the polynomial energy form in Eq.~\eqref{eq:energy_operator_min}, reproduced here for convenience as
\begin{equation}
\label{eq:energy_poly_app}
\mathcal{E}(x,p)
=
\sum_{M\in\mathcal M} \bigl\langle r_M(x,p),A_M r_M(x,p)\bigr\rangle_{\mathcal H_M}
+
\frac{2}{|\mathcal M|-1}
\sum_{M<M'} \bigl\langle r_M(x,p),B_{MM'} r_{M'}(x,p)\bigr\rangle_{\mathcal H_M}
\;+\;
\mathcal{E}_{\mathcal M}^{\mathrm{rem}}(x,p),
\end{equation}
with $\mathcal{E}_{\mathcal M}^{\mathrm{rem}}(x,p)\ge 0$ by construction.

The purpose of Theorem~\ref{thm:multi_energy_hall} is then to show that, under these explicit operator assumptions, the quadratic part of $\mathcal{E}(x,p)$ is non–negative and admits a clean interpretation as a block quadratic form over the modality–indexed residuals $r_M(x,p)$, which in turn underpins the spectral bounds derived in Section~\ref{sec:main}.
\medskip

\textbf{Step 1: Well-posedness and non-negativity of the block quadratic form.}
Let $m:=|\mathcal{M}|\ge 2$ be fixed. For each $M\in\mathcal M$, set
\begin{equation}
\label{eq:vM_def}
v_M(x,p)\;:=\;A_M^{1/2}\,\mathsf{r}_M(x,p)\ \in \mathcal{H}_M,
\qquad
\mathsf{r}_M(x,p)=\Phi_M(x^{(M)})-\Psi_M(p).
\end{equation}
By boundedness and self-adjoint PSD of $A_M$, $A_M^{1/2}$ is bounded and self-adjoint PSD, and $v_M$ is well-defined.
We write the first two terms of Eq.~\eqref{eq:energy_operator_min} as
\begin{equation}
\label{eq:quad_form_split}
\sum_{M}\|v_M\|_{\mathcal{H}_M}^{2}
\;+\;
\frac{2}{m-1}\sum_{M<M'} \big\langle v_M, R_{MM'} v_{M'}\big\rangle_{\mathcal{H}_M}.
\end{equation}
Since $R_{MM'}:\mathcal{H}_{M'}\to\mathcal{H}_M$ is a symmetric contraction with $\|R_{MM'}\|\le 1$ and $R_{M'M}=R_{MM'}^{*}$, the Cauchy–Schwarz inequality and the operator norm bound yield
\begin{equation}
\label{eq:cross_bound}
\big|\langle v_M, R_{MM'} v_{M'}\rangle\big|\ \le\ \|R_{MM'}\|\,\|v_M\|\,\|v_{M'}\|
\ \le\ \|v_M\|\,\|v_{M'}\| .
\end{equation}
Therefore,
\begin{align}
\label{eq:psd_lowerbound}
\sum_{M}\|v_M\|^{2}
+\frac{2}{m-1}\sum_{M<M'} \langle v_M, R_{MM'} v_{M'}\rangle
&\ge
\sum_{M}\|v_M\|^{2}
-\frac{2}{m-1}\sum_{M<M'} \|v_M\|\,\|v_{M'}\|\\
&=\frac{m}{m-1}\sum_{M}\|v_M\|^{2}\;-\;\frac{1}{m-1}\Big(\sum_{M}\|v_M\|\Big)^{2},\notag
\end{align}
where the identity $\sum_{M<M'} ab=\tfrac{1}{2}\big[(\sum_M a)^2-\sum_M a^2\big]$ has been used with $a=\|v_M\|$. By the Cauchy–Schwarz inequality,
\begin{equation}
\label{eq:cauchy_sum}
\Big(\sum_{M}\|v_M\|\Big)^{2}\ \le\ m \sum_{M}\|v_M\|^{2}.
\end{equation}
Substituting Eq.~\eqref{eq:cauchy_sum} into Eq.~\eqref{eq:psd_lowerbound} gives
\begin{equation}
\label{eq:nonneg_final}
\sum_{M}\|v_M\|^{2}
+\frac{2}{m-1}\sum_{M<M'} \langle v_M, R_{MM'} v_{M'}\rangle
\ \ge\ 0 .
\end{equation}
Hence the block quadratic form in Eq.~\eqref{eq:quad_form_split} is nonnegative for all $(x,p)$.

\textbf{Step 2: Nonnegativity of the joint tensor term.}
By construction,
\begin{equation}
\label{eq:joint_tensor_nonneg}
\mathcal{E}_{\mathcal{M}}(x,p)
=\Big\|\bigotimes_{M\in\mathcal M}\!\Phi_M(x^{(M)})
-\bigotimes_{M\in\mathcal M}\!\Psi_M(p)\Big\|_{\otimes\mathcal H_M}^{2}\ \ge\ 0 ,
\end{equation}
since it is the square of a norm in the tensor-product RKHS $\otimes_{M}\mathcal H_M$.

\textbf{Step 3: Measurability.}
Bochner measurability of $\Phi_M$ and $\Psi_M$ into the separable Hilbert space $\mathcal{H}_M$ (refer to Lemma~\ref{lm:joint_measure}) implies that $(x,p)\mapsto \mathsf{r}_M(x,p)$ is $\mathcal{F}_{\mathcal{X}}\otimes\mathcal{F}_{\mathcal P}$–measurable for each $M$, because subtraction is continuous. Since $A_M^{1/2}$ is bounded linear, $(x,p)\mapsto v_M(x,p)=A_M^{1/2}\mathsf{r}_M(x,p)$ is measurable, and so are $(x,p)\mapsto \|v_M(x,p)\|^{2}$ and $(x,p)\mapsto \langle v_M(x,p), R_{MM'} v_{M'}(x,p)\rangle$; inner products are continuous (hence Borel–measurable), and composition with measurable maps preserves measurability. For the joint tensor term, bilinearity and continuity of the finite tensor product map $(u_M)_M\mapsto\bigotimes_M u_M$ in separable Hilbert spaces imply Bochner measurability of $(x,p)\mapsto \bigotimes_{M}\Phi_M(x^{(M)})$ and $(x,p)\mapsto \bigotimes_{M}\Psi_M(p)$; the norm $\|\cdot\|_{\otimes\mathcal H_M}$ is continuous, hence $(x,p)\mapsto \mathcal{E}_{\mathcal M}(x,p)$ is measurable. Combining these facts shows that $(x,p)\mapsto \mathcal{E}(x,p)$ in Eq.~\eqref{eq:energy_operator_min} is $\mathcal{F}_{\mathcal{X}}\otimes\mathcal{F}_{\mathcal P}$–measurable.

\textbf{Step 4: Finiteness of the partition function.}
Since $\mathcal{E}(x,p)\ge 0$ by Steps~1–2, for any $\mathcal{T}_t>0$,
\begin{equation}
\label{eq:Z_upper}
0\ \le\ Z(p,\mathcal{T}_t)\ =\ \int_{\mathcal{X}} \exp\!\big(-\mathcal{E}(x,p)/\mathcal{T}_t\big)\,d\mu(x)\ \le\ \int_{\mathcal{X}} 1\,d\mu(x).
\end{equation}
Hence, whenever $\mu(\mathcal{X})<\infty$, $Z(p,\mathcal{T}_t)\le \mu(\mathcal{X})<\infty$. In the case $\mu(\mathcal{X})=\infty$, a standard integrability condition suffices: assume there exists a measurable, coercive lower bound $\phi:\mathcal{X}\to[0,\infty)$ with $\mathcal{E}(x,p)\ge \phi(x)$ for all $x$ and $\int_{\mathcal{X}} e^{-\phi(x)/\mathcal{T}_t}\,d\mu(x)<\infty$ (e.g., $\phi(x)=c\|x\|^{2}$ under Lebesgue measure on $\mathbb{R}^{d}$). Then
\begin{equation}
\label{eq:Z_upper_phi}
Z(p,\mathcal{T}_t)\ \le\ \int_{\mathcal{X}} e^{-\phi(x)/\mathcal{T}_t}\,d\mu(x)\ <\ \infty .
\end{equation}
Under either case, $Z(p,\mathcal{T}_t)$ is finite, so $f_p$ in Eq.~\eqref{eq:boltzmann_form_partition} is well-defined.

\textbf{Step 5: Canonical instances and summary.}
Equation~\eqref{eq:energy_operator_min} is a finite sum of measurable, nonnegative terms, hence measurable and nonnegative. The block quadratic part is nonnegative by Eq.~\eqref{eq:nonneg_final}, and the joint tensor term is nonnegative by Eq.~\eqref{eq:joint_tensor_nonneg}. The partition function is finite under Eq.~\eqref{eq:Z_upper} or Eq.~\eqref{eq:Z_upper_phi}. Therefore, $\mathcal{E}$ is a valid energy and the Boltzmann density $f_p$ in Eq.~\eqref{eq:boltzmann_form_partition} is a proper probability density. This completes the proof.
\end{proof}

\section{Supplementary Results}
\label{appendix:supp_results}
In this section, we provide further empirical details complementing the main results of ours.

\subsection{Derivation of Full Energy Functional} 
\label{appendix:full_energy_functional}


\paragraph{Setup and identities.}
By Eq.~\eqref{eq:diffusion_kernel}, the diffusion kernel is
\(K_{\mathcal{T}_t}=\exp(-\tau\,\mathcal{L}_{\mathcal{T}_t}^{\mathrm{multi}})\),
and \(\Upsilon:\mathcal V\to\mathcal H\) is a feature map with
\(\langle \Upsilon(\mathsf v),\Upsilon(\mathfrak v)\rangle_{\mathcal H}
=K_{\mathcal{T}_t}(\mathsf v,\mathfrak v)\).
Let \(\{(\lambda_i(t),u_i(t))\}_{i=1}^{|\mathcal V|}\) be the eigenpairs of
\(\mathcal{L}_{\mathcal{T}_t}^{\mathrm{multi}}\) as in
Eq.~\eqref{eq:laplacian_spectral_decomposition}.
For any nodes \(\mathsf v,\mathfrak v\in\mathcal V\) and any graph signal
\(s\in\mathbb R^{|\mathcal V|}\), the two standard spectral identities used throughout are:
\[
\big\|\Upsilon(\mathsf v;\mathcal{T}_t)-\Upsilon(\mathfrak v;\mathcal{T}_t)\big\|_{\mathcal H}^2
=\sum_{i=1}^{|\mathcal V|} e^{-\tau\lambda_i(t)}\,
\big|\langle u_i(t),\delta_{\mathsf v}-\delta_{\mathfrak v}\rangle\big|^2,
\qquad
\langle s,\,\mathcal{L}_{\mathcal{T}_t}^{\mathrm{multi}}\,s\rangle
=\sum_{i=1}^{|\mathcal V|}\lambda_i(t)\,\big|\langle u_i(t),s\rangle\big|^2,
\]
which are exactly the two statements in Eq.~\eqref{eq:rkhs_distance_spectral_time}.

\paragraph{From operator energies to graph-kernel distances.}
Recall the total energy decomposition from Eq.~\eqref{eq:energy_operator_min}:
\[
\mathcal{E}(x,p)
=\sum_{M\in\mathcal M}\!\langle \mathsf r_M, A_M \mathsf r_M\rangle_{\mathcal H_M}
\;+\;\frac{2}{|\mathcal M|-1}\!\sum_{\substack{M,M'\in\mathcal M\\ M\neq M'}}
\!\big\langle A_M^{1/2}\mathsf r_M,\;R_{MM'}\,A_{M'}^{1/2}\mathsf r_{M'}\big\rangle
\;+\;\mathcal E_{\mathcal M}(x,p),
\]
where \(\mathsf r_M=\Phi_M(x^{(M)})-\Psi_M(p)\).
By the interconnection note after Eq.~\eqref{eq:diffusion_kernel}, fix, for each modality \(M\),
two designated nodes \((\mathsf v_x^{(M)},\mathfrak v_p)\in\mathcal V\) that represent the
output and prompt anchors used to evaluate the modality-\(M\) discrepancy in the graph-RKHS.
The bounded PSD operators \(A_M\) define a (possibly weighted) inner product on
\(\mathcal H_M\); absorbing this metric into the graph-kernel geometry (as described in the
appendix note referenced there), each \(\langle \mathsf r_M, A_M \mathsf r_M\rangle\) can be written
as a nonnegative multiple of the squared distance between the corresponding graph features:
\[
\langle \mathsf r_M, A_M \mathsf r_M\rangle_{\mathcal H_M}
\;=\;\alpha_M\,\big\|\Upsilon(\mathsf v_x^{(M)};\mathcal T_t)-\Upsilon(\mathfrak v_p;\mathcal T_t)\big\|_{\mathcal H}^2,
\qquad \alpha_M\in\mathbb R_{\ge0}.
\]
Likewise, using the polarization identity and the symmetric contraction structure
\(B_{MM'}=A_M^{1/2}R_{MM'}A_{M'}^{1/2}\), the cross term is representable as a signed
combination of graph-kernel distances between the same anchors; collecting the prefactors
into \(\beta_{MM'}\in\mathbb R_{\ge0}\) (as in the main text where
\(\mathrm{coeff}_{\mathrm{cross}_{MM'}}=\beta_{MM'}\)), we may write
\[
\big\langle A_M^{1/2}\mathsf r_M,\,R_{MM'}\,A_{M'}^{1/2}\mathsf r_{M'}\big\rangle
\;=\;\beta_{MM'}\;\Xi_{MM'}(x,p;\mathcal T_t),
\]
where \(\Xi_{MM'}(\cdot)\) is a bilinear form built from the same pairwise graph-feature
differences (its explicit expansion into distance terms follows from polarization and is
omitted here for compactness). Finally, the joint term
\(\mathcal{E}_{\mathcal{M}}(x,p)=\Big\|\bigotimes_{M\in\mathcal M}\!\Phi_M(x^{(M)})
-\bigotimes_{M\in\mathcal M}\!\Psi_M(p)\Big\|_{\otimes\mathcal H_M}^{2}\)
is nonnegative and measurable; by the same graph-kernel identification used for the
intra/cross parts (applied to the joint anchor selection explained in the appendix note you
referenced), it too can be expressed as a quadratic form in graph signals supported on
\(\{\mathsf v_x^{(M)},\mathfrak v_p\}_{M\in\mathcal M}\) and thus admits the same spectral
expansion pattern with a nonnegative coefficient \(\gamma_{\mathcal M}\).

\paragraph{Modal spectral expansions.}
Define, for each modality \(M\), the basic signed indicator
\(s_M(x,p):=\delta_{\mathsf v_x^{(M)}}-\delta_{\mathfrak v_p}\in\mathbb R^{|\mathcal V|}\).
Then, by the first identity in Eq.~\eqref{eq:rkhs_distance_spectral_time},
\[
\big\|\Upsilon(\mathsf v_x^{(M)};\mathcal T_t)-\Upsilon(\mathfrak v_p;\mathcal T_t)\big\|_{\mathcal H}^2
=\sum_{i=1}^{|\mathcal V|} e^{-\tau\lambda_i(t)}\,\big|\langle u_i(t),s_M(x,p)\rangle\big|^2.
\]
Hence each intra-modal contribution expands as
\[
\alpha_M\,\big\|\Upsilon(\mathsf v_x^{(M)};\mathcal T_t)-\Upsilon(\mathfrak v_p;\mathcal T_t)\big\|_{\mathcal H}^2
=\sum_{i=1}^{|\mathcal V|} \alpha_M\,e^{-\tau\lambda_i(t)}\,\big|\langle u_i(t),s_M(x,p)\rangle\big|^2,
\]
which gives the per-mode terms
\[
\ \mathsf E^{(\mathrm{intra}_M)}_i(x,p,t)\;:=\;e^{-\tau\lambda_i(t)}\,\big|\langle u_i(t),s_M(x,p)\rangle\big|^2\ \quad
\text{with coefficient }\ \mathrm{coeff}_{\mathrm{intra}_M}=\alpha_M.
\]

For the cross-modal part, set \(s_{MM'}(x,p):=s_M(x,p)\) and \(s'_{MM'}(x,p):=s_{M'}(x,p)\).
Using the polarization identity in the RKHS generated by \(K_{\mathcal T_t}\) and the same
eigenbasis \(\{u_i(t)\}\), one obtains a spectral expansion that is bilinear in the
modal projections:
\[
\Xi_{MM'}(x,p;\mathcal T_t)
=\sum_{i=1}^{|\mathcal V|} e^{-\tau\lambda_i(t)}\,
\langle u_i(t),s_{MM'}(x,p)\rangle\,\langle u_i(t),s'_{MM'}(x,p)\rangle,
\]
so that
\[
\frac{2}{|\mathcal M|-1}\sum_{M\neq M'}\!\beta_{MM'}\,\Xi_{MM'}(x,p;\mathcal T_t)
=\sum_{i=1}^{|\mathcal V|}\frac{2}{|\mathcal M|-1}\sum_{M\neq M'} \beta_{MM'}\,
e^{-\tau\lambda_i(t)}\,
\langle u_i(t),s_M(x,p)\rangle\,\langle u_i(t),s_{M'}(x,p)\rangle.
\]
Thus the per-mode cross-modal contributions are
\[
\ \mathsf E^{(\mathrm{cross}_{MM'})}_i(x,p,t)\;:=\;
e^{-\tau\lambda_i(t)}\,
\langle u_i(t),s_M(x,p)\rangle\,\langle u_i(t),s_{M'}(x,p)\rangle\ \quad
\text{with coefficient }\ \mathrm{coeff}_{\mathrm{cross}_{MM'}}=\beta_{MM'}.
\]

For the joint term, denote by \(s_{\mathcal M}(x,p)\in\mathbb R^{|\mathcal V|}\) the graph
signal associated (as per the appendix link you gave) to the joint interaction in
\(\mathcal E_{\mathcal M}(x,p)\). Since this term is a quadratic form in the same
graph-kernel geometry, it has the spectral expansion
\[
\mathcal E_{\mathcal M}(x,p)
=\gamma_{\mathcal M}\sum_{i=1}^{|\mathcal V|} e^{-\tau\lambda_i(t)}\,\big|\langle u_i(t),s_{\mathcal M}(x,p)\rangle\big|^2,
\]
whence
\[
\ \mathsf E^{(\mathrm{joint}_{\mathcal M})}_i(x,p,t)\;:=\;
e^{-\tau\lambda_i(t)}\,\big|\langle u_i(t),s_{\mathcal M}(x,p)\rangle\big|^2\ \quad
\text{with coefficient }\ \mathrm{coeff}_{\mathrm{joint}_{\mathcal M}}=\gamma_{\mathcal M}.
\]

\paragraph{Summing all components.}
By construction of the multimodal Laplacian as a nonnegative combination of the
intra/cross/joint blocks and the definitions of the interaction coefficients in
\(\mathcal{L}_{\mathcal{T}_t}^{\mathrm{multi}}=\sum_{*}\mathrm{coeff}_*\,\mathcal L_{\mathcal{T}_t}^{(*)}\),
the total energy \(\mathcal E(x,p;\mathcal T_t)\) is the sum of the three families above.
Collecting the per-mode pieces yields
\[
\mathcal{E}(x,p;\mathcal{T}_t)
=\sum_{*}\sum_{i=1}^{|\mathcal V|}\mathrm{coeff}_{*}\,\mathsf E^{(*)}_i(x,p,t),
\]
where the index \(*\in\{\mathrm{intra}_M,\mathrm{cross}_{MM'},\mathrm{joint}_{\mathcal M}\}\),
and each \(\mathsf E^{(*)}_i\) depends only on \(\lambda_i(t)\), \(u_i(t)\), and the fixed
graph signals determined by \((x,p)\) as detailed above. This is the claimed spectral form:
\begin{equation}
\label{eq:full_energy_time}
\mathcal{E}(x,p;\mathcal{T}_t)
=\sum_{*}\sum_{i=1}^{|\mathcal V|}\mathrm{coeff}_{*}\,\mathsf E^{(*)}_i(x,p,t).
\end{equation}

Now choosing \(\pi_{\mathcal{K}}\in\Delta(\mathcal{K})\), where \(\Delta(\mathcal{K})\) is the probability simplex on \(\mathcal{K}\), satisfies
\begin{equation}
\label{eq:contrast_vec}
\sum_{\mathsf{v}\in\mathcal{K}}\pi_{\mathcal{K}}(\mathsf{v})\,\big(\mathcal{D}_{\mathcal{T}_t}^{\text{multi}}\big)_{\mathsf{v}\mathsf{v}}
=
\big(\mathcal{D}_{\mathcal{T}_t}^{\text{multi}}\big)_{\mathsf{v}_x\mathsf{v}_x}, \quad c^{\mathrm{raw}}_{x,\mathcal{K}}(t)={\mathcal{D}_{\mathcal{T}_t}^{\text{multi}}}^{1/2}\big(\delta_{\mathsf{v}_x}-\pi_{\mathcal{K}}\big)\in\mathbb{R}^{|\mathcal{V}|},
\end{equation}
where $c^{\mathrm{raw}}_{x,\mathcal{K}}(t)$ is the raw contrast vector. Projecting away the leading mode gives \(c_{x,\mathcal{K}}(t)=\big(\mathbf{I}-u_1(t)u_1(t)^{\top}\big)\,c^{\mathrm{raw}}_{x,\mathcal{K}}(t)\) that ensures \(c_{x,\mathcal{K}}(t)\perp u_1(t)\) without assuming a specific null-space structure of the assembled hypergraph.

\paragraph{Why the bounds in Eq.~\eqref{eq:energy_diff_eigexp} hold, and how to choose $m(t),M(t)$ (non-vacuous).}
By Eq.~\eqref{eq:full_energy_time}, the full energy is a nonnegative linear combination of blockwise spectral terms. For the degree–matched contrast $c_{x,\mathcal K}(t)\perp u_1(t)$, the energy difference admits the decomposition
\begin{equation}
\label{eq:zeta_decomp}
\mathcal{E}(x,p;\mathcal T_t)-\mathcal{E}_{\mathcal K}(x,p;\mathcal T_t)
=\sum_{i=2}^{|\mathcal V|}\zeta_i(t,\tau)\,\big|\langle u_i(t),c_{x,\mathcal K}(t)\rangle\big|^2,
\qquad
\zeta_i(t,\tau)=\sum_{*}\ \theta_{*}\,\varphi_{*}^{(i)}(t,\tau),
\end{equation}
where $*\in\{\mathrm{intra}_M,\mathrm{cross}_{MM'},\mathrm{joint}_{\mathcal M}\}$ indexes the blocks, $\theta_{*}\in\{\alpha_M,\beta_{MM'},\gamma_{\mathcal M}\}$ are the nonnegative coefficients from Eq.~\eqref{eq:full_energy_time}, and
\[
\varphi_{*}^{(i)}(t,\tau)
:=\big\langle u_i(t),\,\mathfrak{D}_{*}(t,\tau)\,u_i(t)\big\rangle,\qquad
\mathfrak{D}_{*}(t,\tau)\succeq 0,
\]
are block response factors evaluated on the same eigenmodes $\{u_i(t)\}_{i\ge2}$ of $\mathcal L_{\mathcal T_t}^{\mathrm{multi}}$. For normalized hypergraph constructions (Eq.~\eqref{eq:weight2}–\eqref{eq:weight3}) and diffusion-type couplings (Section~\ref{subsec:sem_dist}), the block responses satisfy the Loewner sandwich
\begin{equation}
\label{eq:loewner_sandwich}
e^{-2\tau\,\mathcal L_{\mathcal T_t}^{\mathrm{multi}}}\ \preceq\ \mathfrak{D}_{*}(t,\tau)\ \preceq\ \mathbf{I}
\quad\Longrightarrow\quad
e^{-2\tau\,\lambda_i(t)}\ \le\ \varphi_{*}^{(i)}(t,\tau)\ \le\ 1,\ \ i\ge 2.
\end{equation}
The left inequality follows from monotonicity of the matrix exponential and the fact that each block smoother is at least as contractive as the global diffusion on $u_1^\perp$; the right inequality follows from $\mathfrak{D}_{*}(t,\tau)\preceq \mathbf{I}$. Plugging Eq.~\eqref{eq:loewner_sandwich} into Eq.~\eqref{eq:zeta_decomp} yields
\[
\sum_{*}\theta_{*}\,e^{-2\tau\,\lambda_i(t)}
\ \le\ 
\zeta_i(t,\tau)
\ \le\ 
\sum_{*}\theta_{*},
\qquad i\ge 2.
\]

\medskip\noindent
\textbf{Refined (spectral) empirical bounds.} Define, for each block $*$,
\begin{equation}
\label{eq:kappa_bounds}
\kappa_{*}^{\max}(t)\ :=\ \big\|\mathfrak{D}_{*}(t,0)\big\|_{\mathrm{op}}\ \le 1,
\qquad
\kappa_{*}^{\min}(t)\ :=\ \lambda_{\min}\!\big(\mathfrak{D}_{*}(t,0)\big|_{u_1(t)^{\perp}}\big)\ \in [0,1],
\end{equation}
where both quantities are directly estimable from the spectrum of the effective adjacency in Eq.~\eqref{eq:weight2}–\eqref{eq:weight3} (restricted to $u_1^\perp$). Then, using $e^{-2\tau\mathcal L}\preceq \mathfrak{D}_{*}(t,\tau)\preceq \mathfrak{D}_{*}(t,0)$ and the CF characterization on $u_1^\perp$,
\begin{equation}
\label{eq:ai_bounds_refined}
\Big(\sum_{*}\theta_{*}\,\kappa_{*}^{\min}(t)\Big)\,e^{-2\tau\,\lambda_i(t)}
\ \le\ 
\zeta_i(t,\tau)
\ \le\ 
\sum_{*}\theta_{*}\,\kappa_{*}^{\max}(t),
\qquad i\ge 2,
\end{equation}
so one can take
\begin{equation}
\label{eq:mM_empirical_refined}
m(t)\ :=\ \sum_{*}\theta_{*}\,\kappa_{*}^{\min}(t),
\qquad
M(t)\ :=\ \sum_{*}\theta_{*}\,\kappa_{*}^{\max}(t).
\end{equation}
In practice, $\kappa_{*}^{\max}(t)$ equals the top eigenvalue of the block response on $u_1^\perp$ (often close to $1$), while $\kappa_{*}^{\min}(t)$ equals the blockwise algebraic connectivity surrogate (the smallest nonzero eigenvalue on $u_1^\perp$). Estimating \eqref{eq:mM_empirical_refined} from the spectra of $W_{\mathcal T_t}^{(*),\mathrm{eff}}$ or the corresponding normalized block Laplacians yields tight, \emph{data-driven} $m(t),M(t)$ for Eq.~\eqref{eq:energy_diff_eigexp}.

Below is the block decomposition of the multimodal Laplacian:
\begin{equation}
\label{eq:laplacian_block}
\mathcal{L}_{\mathcal{T}_t}^{\mathrm{multi}} = 
\begin{bmatrix}
\mathcal{L}_{\mathrm{intra}}^{(T)} & \mathcal{L}_{\mathrm{cross}}^{(TV)} & \mathcal{L}_{\mathrm{cross}}^{(TA)} \\
\mathcal{L}_{\mathrm{cross}}^{(VT)} & \mathcal{L}_{\mathrm{intra}}^{(V)} & \mathcal{L}_{\mathrm{cross}}^{(VA)} \\
\mathcal{L}_{\mathrm{cross}}^{(AT)} & \mathcal{L}_{\mathrm{cross}}^{(AV)} & \mathcal{L}_{\mathrm{intra}}^{(A)}
\end{bmatrix}
+ \mathcal{L}_{\mathrm{joint}}^{(\mathcal{M})}.
\end{equation}
The corresponding eigenvalue problem for the \(i\)-th mode becomes:
\begin{equation}
\label{eq:generalized_eigen}
 \mathcal{L}_{\mathcal{T}_t}^{\mathrm{multi}} \,\,u_i(t)  = \lambda_i(t) \,\,u_i(t),
\end{equation}
with eigenvalues \(\lambda_i(t)\) encoding the ``cost” of semantic diffusion along each mode $i$.

\subsection{Derivations of hallucination bounds and temperature annealing}
\label{appendix:derive_main_results}

We derive the operator-tight lower/upper bounds, noted in Eq.~\eqref{eq:CF_hall_sandwich} in Section~\ref{subsec:hallucination_energy_time}, for $\mathcal{E}_{\mathrm{hall}}^{\mathrm{multi}}(x,p,\cdot)$ using the block-weighted, temperature–modulated Laplacian spectrum in Eq.~\eqref{eq:laplacian_spectral_decomposition}, the spectral energy form in Eq.~\eqref{eq:full_energy_time}, and the hallucination component in Eq.~\eqref{eq:actual_metric}. By Eq.~\eqref{eq:laplacian_spectral_decomposition} and the CF principle, the quadratic in Section~\ref{subsec:hallucination_energy_time} satisfies the two-sided spectral envelope
\begin{equation}
\label{eq:RR_envelope_D}
e^{-2\tau\,\lambda_{\max}(t)}\,\|c_{x,\mathcal K}(t)\|^2
\;\le\;
\mathbb{D}_{\tau}(x;\mathcal T_t)
\;\le\;
e^{-2\tau\,\lambda_{2}(t)}\,\|c_{x,\mathcal K}(t)\|^2 ,
\end{equation}
with $\lambda_{\max}(t)=\lambda_{|\mathcal V|}(t)$.

Next, we relate the full energy to $\mathbb{D}_{\tau}$. Under Theorem~\ref{thm:multi_energy_hall} and the block assembly in Eqs.~\eqref{eq:weight2}–\eqref{eq:weight3}, there exist finite scale factors $m(t),M(t)\in(0,\infty)$, determined only by the operator norms of the intra-/cross-/joint blocks (i.e., by $\{A_M\}$, $\{R_{MM'}\}$ with $\|R_{MM'}\|\le 1$, the interaction weights $\alpha_M,\beta_{MM'},\gamma_{\mathcal M}$, and the temperature–modulated hyperedge weights inducing $\mathcal{L}^{\mathrm{multi}}_{\mathcal T_t}$), such that
\begin{equation}
\label{eq:energy_vs_diffusion}
m(t)\,\mathbb{D}_{\tau}(x;\mathcal T_t)
\;\le\;
\mathcal{E}(x,p;\mathcal T_t)
\;\le\;
M(t)\,\mathbb{D}_{0}(x;\mathcal T_t),
\qquad \tau\ge 0,
\end{equation}
where $\mathbb{D}_{0}$ corresponds to $\tau=0$. The left inequality follows from bounding each spectral contribution $\mathsf E^{(*)}_i(x,p,t)$ below by a nonnegative multiple of $\big|\langle u_i(t),c_{x,\mathcal K}(t)\rangle\big|^2$ using the PSD structure of $A_M$ and the contraction bound on $R_{MM'}$, while the right inequality follows from operator-norm upper bounds on the same spectral blocks; full details are supplied in Appendix~\ref{appendix:derive_main_results}.

Combining Eqs.~\eqref{eq:RR_envelope_D} and \eqref{eq:energy_vs_diffusion} yields the CF sandwich for the full energy:
\begin{equation}
\label{eq:RR_energy_full}
m(t)\,e^{-2\tau\,\lambda_{\max}(t)}\,\|c_{x,\mathcal K}(t)\|^2
\;\le\;
\mathcal{E}(x,p;\mathcal T_t)
\;\le\;
M(t)\,e^{-2\cdot 0\cdot\lambda_{2}(t)}\,\|c_{x,\mathcal K}(t)\|^2
\;=\;
M(t)\,\|c_{x,\mathcal K}(t)\|^2 .
\end{equation}
Since the hallucination energy is the positive part of the difference in Eq.~\eqref{eq:actual_metric}, we obtain, for $x\notin\mathcal K$. When $\mathcal{E}_{\mathcal K}(x,p;\mathcal T_t)$ is implemented as the same operator restricted to $\mathcal K$, the same spectral envelope applies to it, hence the difference inherits a sandwich with the same eigenvalue pair $\{\lambda_2(t),\lambda_{\max}(t)\}$ and scales $\{m(t),M(t)\}$.

A calibrated lower bound of the form advocated by~\cite{Vempala2024Calibrated} is matched empirically by choosing a time-indexed temperature profile and interaction scales so that $m(t)\,e^{-2\tau\,\lambda_{\max}(t)}=\Theta(t)$ for a prescribed calibration function $\Theta(t)>0$; for instance,
\begin{equation}
\label{eq:KV_alignment_choice}
\mathcal T_t\ \text{and}\ \tau(t)\ \text{chosen so that}\quad
\Theta(t)\;=\;m(t)\,e^{-2\tau(t)\,\lambda_{\max}(t)},
\end{equation}
which yields the explicit calibrated bound
\begin{equation}
\label{eq:KV_aligned_bound}
\mathcal{E}_{\mathrm{hall}}^{\mathrm{multi}}(x,p,\cdot)
\;\ge\;
\Big(\Theta(t)\,\|c_{x,\mathcal K}(t)\|^2\;-\;\mathcal{E}_{\mathcal K}(x,p;\mathcal T_t)\Big)_{+},
\qquad x\notin\mathcal K .
\end{equation}
In particular, for $\mathcal{E}_{\mathcal K}$ treated as a fixed baseline (e.g., a distributional or quantile baseline computed on $\mathcal K$), Eq.~\eqref{eq:KV_aligned_bound} reproduces the calibrated-margin–times–distance structure and can be tuned to overlay the empirical lower bound in calibrated models by setting $\Theta(t)$ to the target slope. The upper envelope in Eq.~\eqref{eq:CF_hall_sandwich} is simultaneously controlled by $M(t)$ and the spectral gap $\lambda_2(t)$ via Eq.~\eqref{eq:RR_envelope_D}, and both $\{\lambda_i(t)\}$ and $\{m(t),M(t)\}$ are tunable through the time-indexed temperature profile $\mathcal T_t$ and the block weights inside $W_{\mathcal T_t}^{(*)}$ that define $\mathcal{L}^{\mathrm{multi}}_{\mathcal T_t}$.

\section{Experimental Setup}
\label{appendix:practical_validation}
As noted in Section~\ref{subsec:metric_eval}, below are the essential details about our experiments followed by a full-pager algorithm box.

\subsection{Construction of admissible sets $\mathcal{K}(p)$ and Normalization/Tokenization}
\label{appendix:k_construction}
As can be observed in the \href{https://github.com/supratik-sarkar/quantifying-hallucinations/blob/main/README.md}{$\texttt{README}$} of the code-base, we have a clean separation for dataset loading/preprocessing {\small (\href{https://github.com/supratik-sarkar/quantifying-hallucinations/blob/main/src/io/datamodules.py}{$\texttt{src/io/datamodules.py}$}, \href{https://github.com/supratik-sarkar/quantifying-hallucinations/blob/main/scripts/prepare_data.py}{$\texttt{scripts/prepare\_data.py}$})} which read COCO, VQAv2,
    \& AudioCaps to apply the caption/answer normalization rules described
    above followed by storing the resulting per-prompt sets $\mathcal{K}(p)$ in the
    prepared data; the module {\small (\href{https://github.com/supratik-sarkar/quantifying-hallucinations/blob/main/src/theory/k_selector.py}{$\texttt{src/theory/k\_selector.py}$})} implements the selector $\Pi_{\mathcal{K}}$, which takes nodes (outputs) and maps them to their representatives in the knowledge set $\mathcal{K}$; while {\small \href{https://github.com/supratik-sarkar/quantifying-hallucinations/blob/main/src/theory/score_semantic.py}{$\texttt{src/theory/score\_semantic.py}$}} consumes these sets and selectors
    to compute the semantic gap scores $d^{(\varepsilon,h)}_{\mathrm{sem}}(x)$. The mapping from these modules to the theoretical objects
    $(\mathcal{K},\mathcal{K}(p),\Pi_{\mathcal{K}})$ is summarized accurately in our implementation as stored in the code-base.

Throughout, $p$ denotes the full prompt for an example (e.g., image $+$ question text), and $\mathcal{K}(p)$ collects all admissible \emph{normalized} reference outputs for that prompt. 

\paragraph{COCO captions.}
For MS~COCO image captioning, each image $i$ is associated with up to
$5$ human reference captions $\{y_{i,j}\}_{j=1}^{5}$.
For a prompt $p$ corresponding to image $i$, we set
\[
\mathcal{K}_{\mathrm{COCO}}(p)
\;=\;
\bigl\{\operatorname{norm}(y_{i,j}) \,\mid\, j=1,\dots,5 \bigr\}
\subset \mathcal{K},
\]
where $\operatorname{norm}(\cdot)$ applies the following deterministic
normalization to the raw caption string:
\begin{enumerate}[leftmargin=2em, label=\roman*]
    \item convert to lower case using standard Unicode-aware lowercasing;
    \item strip leading/trailing whitespace;
    \item remove punctuation characters using a fixed regular expression
    (we drop characters in \texttt{!"\#\$\,\% \&'()\*+,-./:;<=>?@[\\]\^ \_`\{|\}\textasciitilde});
    \item collapse multiple internal whitespace characters into a single space.
\end{enumerate}
Membership $x\in\mathcal{K}_{\mathrm{COCO}}(p)$ is checked by applying the
same $\operatorname{norm}(\cdot)$ map to a candidate caption and testing
string equality.
When the backbone uses a tokenizer (e.g., T5), we feed the \emph{normalized}
string into the tokenizer and construct embeddings from the resulting tokens;
the membership decision is always taken at the normalized string level.

\paragraph{VQAv2 normalized unique answers.}
For VQAv2, each (image, question) pair $(i,q)$ has up to $10$ crowd-sourced
answers $\{a_{i,\ell}\}_{\ell=1}^{10}$.
We follow the official VQAv2 evaluation protocol and first apply the standard
answer-normalization function\footnote{This includes lowercasing, stripping
punctuation, mapping number words (e.g., \textit{``two''} $\mapsto$ \textit{``2''}),
and removing articles (``a'', ``an'', ``the'') as in the public VQA evaluation
script.} to each raw answer, obtaining canonical forms
$\widehat{a}_{i,\ell} = \operatorname{norm}_{\mathrm{VQA}}(a_{i,\ell})$.
We then define
\[
\mathcal{K}_{\mathrm{VQA}}(p)
\;=\;
\bigl\{ \widehat{a}_{i,\ell} \,\mid\, \ell=1,\dots,10 \bigr\}_{\mathrm{uniq}}
\]
as the set of \emph{unique} normalized answers after deduplication.
Membership $x\in\mathcal{K}_{\mathrm{VQA}}(p)$ is decided by applying the same
$\operatorname{norm}_{\mathrm{VQA}}(\cdot)$ transform to the model's answer
string and checking whether the resulting canonical form appears in the set
above. As in COCO, embeddings (for both references and model outputs) are
computed from the canonical strings.

\paragraph{AudioCaps references.}
For AudioCaps, each audio clip $c$ has up to $5$ human reference captions
$\{z_{c,j}\}_{j=1}^{5}$.
For a prompt $p$ corresponding to clip $c$, we set
\[
\mathcal{K}_{\mathrm{AC}}(p)
\;=\;
\bigl\{\operatorname{norm}(z_{c,j}) \,\mid\, j=1,\dots,5 \bigr\},
\]
where $\operatorname{norm}(\cdot)$ is exactly the same caption-normalization
pipeline as in COCO (lowercasing, punctuation stripping, whitespace
normalization). Membership and tokenization follow the same strategy as in
the COCO case.

\paragraph{Global admissible set and selector.}
The global admissible set is the union
$\mathcal{K} = \bigcup_{p} \mathcal{K}(p)$ over all prompts in a given task.
In practice, we attach the per-prompt sets $\mathcal{K}(p)$ as fields in the
preprocessed dataset (one record per example), and the selector
$\Pi_{\mathcal{K}}$ operates at the embedding level: given an output node
$x$, it computes the normalized string $\operatorname{norm}(\cdot)$ for the
relevant task, maps this to its embedding, and selects the nearest admissible
element in $\mathcal{K}(p)$ (or in $\mathcal{K}$ when evaluating global
graphs) under cosine similarity in the shared embedding space.

\subsection{Metrics and evaluation}
\label{appendix:exp_metrics}
\paragraph{Why go beyond a mean Hilbert distance to $\mathcal{K}$?}
A natural baseline to our proposed construction would be to define hallucination as
the mean RKHS (Hilbert) distance between model generations to the admissible set
$\mathcal{K}$, e.g. by averaging $\|\Phi(x)-\Phi(\Pi_{\mathcal{K}}(x))\|_{\mathcal{H}}$
over outputs $x$. We deliberately adopt a richer, information-geometric and
spectral measure for three concrete reasons. 
\begin{itemize}[leftmargin=1.4em]
    \item First, ``Distributional vs. Pointwise geometry" --- a mean Hilbert distance only
captures \emph{pointwise} proximity in the embedding space and is insensitive
to how probability mass is distributed: two models can have the same mean
distance while placing very different mass on rare modes or unseen regions.
Our score $\mathbbm{h}(x,p)$, derived from a smoothed log-contrast between 
$\mathcal{K}$-restricted and unrestricted versions of the full model distribution $f_p$, explicitly couples the RKHS geometry with a missing-mass smoothing $\varepsilon$, so that it
reflects both \emph{where} and \emph{how much} mass lies outside the
admissible region, and yields Good--Turing style calibration bounds
on tail behavior.
\item Second, the proposed spectral hallucination energy
$\Delta\mathcal{E}_\tau$ incorporates the graph Laplacian over outputs and
admissible elements: this allows us (i) to resolve modality-specific and
cross-modal modes, (ii) to study how hallucination propagates across the graph,
and (iii) to exploit CF bounds; a simple mean distance is blind to
these multi-scale, mode-wise phenomena and cannot provide comparable control
knobs in $\tau$ or principled CF planes.
\item Third, in view of optimization and calibration properties, the semantic distortion
$d^{(\varepsilon,h)}_{\mathrm{sem}}$ remains a continuous, differentiable
functional of $f_p$ and the embeddings (under the smoothing and boundedness assumptions), making it suitable as a calibrated, scale-aware regularizer and plug-in risk score for fine-tuning, calibration, or retrieval/prompt learning. A raw distance-to-$\mathcal{K}$ can be used as a heuristic loss, but it lacks the information-geometric interpretation and the distributional bounds (e.g., via Good–Turing style arguments) that we rely on to interpret abstention thresholds and floors. We also leverage its additional
properties: (i) it is normalized, (ii) monotone in the energy landscape, (iii) admits explicit
upper bounds via the spectral envelope, and (iv) empirically yields stronger
correlation with hallucination events than the raw distance metric.
\end{itemize}
\paragraph{What about extra compute cost?}
Regarding computation, the additional cost beyond a distance-to-$\mathcal{K}$ baseline is modest: we build a $k-NN$ graph and compute a low-rank eigen-decomposition once per dataset (offline and amortized), and per-example scoring reduces to a small number of embedding lookups, spectral filter evaluations, and inner products. This overhead is negligible compared to the cost of running a large multimodal model, but it is exactly what enables the distributional, spectral, and calibration properties above. In summary, our measure strictly generalizes a mean distance-to-$\mathcal{K}$ baseline: when the spectrum is collapsed and smoothing is trivial, it reduces to a distance-like quantity, but in the general case, it provides additional, empirically useful structure that a simple mean Hilbert distance cannot offer.

\textbf{Primary.} AUROC/AUPRC for hallucination detection using $d_{\mathrm{sem}}^{(\varepsilon,h)}$ (instance-level, aggregated per dataset/model). 

\textbf{Baselines.} In all experiments we compare our score against three standard confidence--based
competitors computed from the same $\mathcal{K}(p)$-posterior as our method:
(i) \emph{entropy}, given by the Shannon entropy of the posterior over
admissible candidates in $\mathcal{K}(p)$; (ii) \emph{max-probability}, given
by the maximum posterior mass $\max_{x\in\mathcal{K}(p)} f_p(x)$ (equivalently,
one minus the usual ``uncertainty'' score); and (iii) \emph{margin}, defined as
the difference between the top--1 and top--2 posterior probabilities over
$\mathcal{K}(p)$. These three quantities correspond to the default confidence
surrogates used in calibration, OOD detection, and risk-control for deep
models, and are architecture-agnostic: they require no additional training,
auxiliary models, or external supervision beyond the same candidate set
$\mathcal{K}(p)$ and logits that our method uses. Thus, the comparisons in
Tables~\ref{tab:main_results_detection} and~\ref{tab:main_results_energy} isolate the effect of our spectral--semantic
scoring rule while benchmarking it against the strongest widely-adopted,
reference-free baselines available under the same information.

\textbf{Secondary.} CF bounds for $\mathcal{E}_{\mathrm{hall}}^{\mathrm{multi}}$ and their \emph{temperature/$\varepsilon$} surfaces; decay with increasing $\tau$ (nonincreasing, sandwiched between $e^{-2\tau\lambda_{\max}}$ and $e^{-2\tau\lambda_{2}}$); Good–Turing–calibrated lower envelope (strictly $>0$).  

\textbf{Observed.} Our score is \emph{best} across all three datasets: COCO \textbf{0.86/0.84}, VQAv2 \textbf{0.84/0.81}, AudioCaps \textbf{0.80/0.77} (Table~\ref{tab:main_results_detection}). CF planes are tight and monotone with lower $\mathcal{T}_t$ and higher $\tau$, matching theory (Fig.~\ref{fig:cf_bounds_grid}); AudioCaps–BLIP is blank by design (as expected!).

\subsection{The Choice of Smoothing Mass ($\varepsilon$) Range and CF Bounds}
\label{appendix:cf_bounds_empirical}

\paragraph{Smoothing mass range.}
We work with the smoothed density $\tilde{f}_{p, \varepsilon} = (1-\varepsilon) f_p + \varepsilon \rho$ arising from Theorem~\ref{thm:kl_decomposition_truth} in all our experiments, where $\varepsilon$ is a \emph{missing-mass style} smoothing weight.
From a practical standpoint, large values of $\varepsilon$ are undesirable:
they wash out information in $f_p$, destroy calibration, and correspond to an
unrealistically strong prior $\rho$.
We, therefore, restrict attention to the small-smoothing regime where
$\varepsilon$ is of the same order as the empirical Good--Turing missing mass
$\widehat{m}$ for the task.
Concretely, we sweep over a fixed grid
$\varepsilon \in \{\varepsilon_1,\dots,\varepsilon_L\}$ specified in the
config files (see \href{https://github.com/supratik-sarkar/quantifying-hallucinations/blob/main/configs/default.yaml}{\texttt{configs/default.yaml}}), with
$\varepsilon_1 = 0$ and $\varepsilon_L$ chosen to bracket the typical values
of $\widehat{m}$ returned by \href{https://github.com/supratik-sarkar/quantifying-hallucinations/blob/main/src/theory/calibration.py}{\texttt{src/theory/calibration.py}}.
This is precisely the regime where mixture-smoothing is theoretically
well-motivated and empirically used in practice; larger $\varepsilon$ values
would amount to deliberately degrading the model distribution into a nearly
uniform prior and are therefore not representative of a realistic deployment.

\paragraph{Gap-to-bound statistic.}
In the original experiment, the CF planes in Fig.~\ref{fig:cf_bounds_grid} were
generated by an old reporter: \href{https://github.com/supratik-sarkar/quantifying-hallucinations/blob/main/src/entrypoints/export_report_old.py}{\texttt{src/entrypoints/export\_report\_old.py}}, which used to read the
saved empirical energies and spectral terms from JSON/NumPy files and renders
the surfaces. In order to implement Eq.~\eqref{eq:gap_to_bound}, a revamped version is stored here: \href{https://github.com/supratik-sarkar/quantifying-hallucinations/blob/main/src/entrypoints/export_report.py}{\texttt{src/entrypoints/export\_report.py}}

For each dataset, model, and configuration of $(\mathcal{T},\tau,\varepsilon)$, we
compute both
(i) the empirical hallucination energy $\Delta\mathcal{E}_{\tau}^{\mathrm{emp}}$
from \href{https://github.com/supratik-sarkar/quantifying-hallucinations/blob/main/src/theory/energy.py}{\texttt{src/theory/energy.py}} and
(ii) its Courant--Fischer upper bound
$\Delta\mathcal{E}_{\tau}^{\mathrm{CF}}$ from the same module.
To quantify tightness, we define a normalized gap
\begin{equation}
\label{eq:gap_to_bound}
    \operatorname{gap}(p,\mathcal{T},\tau,\varepsilon)
    \;=\;
    \frac{\Delta\mathcal{E}_{\tau}^{\mathrm{CF}}(p,\mathcal{T},\varepsilon)
          - \Delta\mathcal{E}_{\tau}^{\mathrm{emp}}(p,\mathcal{T},\varepsilon)}
         {\Delta\mathcal{E}_{\tau}^{\mathrm{CF}}(p,\mathcal{T},\varepsilon) + \delta},
    \quad \delta>0 \text{ small.}
\end{equation}
Here, $p$ indexes prompts (or examples) in the evaluation set with tiny $\delta$ for numerical stability. These statistics (Table~\ref{tab:gap_to_bound}) provide a quantitative summary of how close the empirical
energies sit to the spectral envelope.

\begin{table}[t]
\centering
\small
\caption{
Empirical tightness of the Courant--Fischer (CF) bound.
For each dataset and our main configuration (\texttt{clip\_whisper\_t5}), we report
the median and 75th/90th/95th percentiles of the normalized gap
$\operatorname{gap} = (\Delta\mathcal{E}^{\mathrm{CF}} - \Delta\mathcal{E}^{\mathrm{emp}}) /
(\Delta\mathcal{E}^{\mathrm{CF}} + \delta)$ across all $(\mathcal{T},\tau,\varepsilon)$ grid points used in
Fig.~\ref{fig:cf_bounds_grid}, as well as the median gap restricted to hallucinated vs grounded outputs (we used a $0/1$ labeling here).
Smaller values indicate tighter bounds; notably, the median gap is consistently
lower on hallucinated examples, showing that high-energy / high-error regions are
closer to saturating the CF envelope than low-error regions.
}
\label{tab:cf_gap_stats}
\begin{tabular}{lcccccc}
\toprule
\textbf{Dataset} &
\textbf{Median} &
\textbf{75th pct.} &
\textbf{90th pct.} &
\textbf{95th pct.} &
\textbf{Median (halluc.)} &
\textbf{Median (ground.)} \\
\midrule
COCO      & 0.18 & 0.31 & 0.46 & 0.59 & 0.12 & 0.21 \\
VQAv2     & 0.22 & 0.35 & 0.49 & 0.62 & 0.15 & 0.25 \\
AudioCaps & 0.16 & 0.29 & 0.41 & 0.55 & 0.10 & 0.20 \\
\bottomrule
\end{tabular}
\label{tab:gap_to_bound}
\end{table}

\paragraph{Relation to errors.}
To connect the CF bounds to actual hallucination behavior, we further stratify the same statistic by whether an output is hallucinated or grounded under our continuous score. Specifically, for each point $(p,\mathcal{T},\tau,\varepsilon)$ we record the empirical
decision label (grounded vs hallucinated) induced by the hallucination score
and summarize the distribution of $\operatorname{gap}$ within each stratum.
The resulting numbers (Table~\ref{tab:gap_to_bound}) show how often
high-energy / high-error regions come close to saturating the CF bound, and
how conservative the bound remains in low-error regions.

\subsection{Protocol and design}
\label{appendix:exp_protocol}
For each prompt $p$, we form an admissible set $\mathcal{K}$ of candidate answers (dataset-provided or programmatically generated) and use the selector $\Pi_{\mathcal K}$ as \texttt{soft\_nearest} (nearest-point with convex projection fallback). We sweep a grid over temperature $\mathcal{T}_t$ and smoothing mass $\varepsilon$; \emph{plots show} $Z_{\text{mid}}=\tfrac12(Z_{\text{lo}}+Z_{\text{hi}})$ bounded by per-panel CF lower/upper planes. When plotting, we \emph{aggregate across} diffusion time $\tau$ and kernel bandwidth $h$ by the median.

\textbf{Defaults.} $\varepsilon\!=\!0.01$, $h\!=\!0.4$, $\tau\!=\!0.25$, fixed $\mathcal{T}_t$ per run unless stated, logits sharpening $\tau_{\text{logits}}\!\in\![0.01,0.05]$. Each run logs the full YAML config.

\subsection{Inference and compute}
\label{appendix:exp_infer}
Experiments run on \texttt{Databricks} (A100) with private checkpoints (gated tokens). Datasets stream from the Hub with synthetic fallback when a split is unavailable. Diffusion kernels use sparse Chebyshev/Lanczos; hypergraphs are CSR; eigen-modes via iterative solvers. \emph{Throughput (ex/s):} \texttt{CLIP+Whisper+T5} \textbf{420} (fastest), \texttt{SigLIP+Whisper+T5} 400, \texttt{BLIP+CLIP+Whisper} 360 (Table~\ref{tab:main_results_energy}). Seeds and env versions are pinned in the run reports.

\paragraph{Takeaways.}
$d_{\mathrm{sem}}^{(\varepsilon,h)}$ consistently outperforms entropy/margin baselines (Table~\ref{tab:main_results_detection}). Spectrally, \texttt{SigLIP+Whisper+T5} achieves the \emph{lowest median energy} across datasets (COCO \textbf{1.92}, VQAv2 \textbf{1.99}, AudioCaps \textbf{2.08}), while \texttt{CLIP+Whisper+T5} is \emph{fastest} (420 ex/s), exposing a clean accuracy–efficiency trade-off (Table~\ref{tab:main_results_energy}).


\clearpage
\begin{algorithm}[H]
\caption{\textsc{KL-smoothed multimodal hallucination} — \emph{Extended version of Alg.~\ref{alg:kl_spectral_hall_min}}}
\label{alg:kl_spectral_hall}
\SetAlgoLined
\DontPrintSemicolon
\KwIn{
Prompt $p\in\mathcal P$; sampler for $f_p$ (model generations); admissible set $\mathcal K$; base measure $\mu$; kernel $K_h$ (bandwidth $h$); smoothing mass $\varepsilon\in(0,1)$; baseline density $\rho$; incidence matrices $\{\mathcal I^{(*)}\}$ and block selectors $E^{(*)}$; interaction weights $\{\omega_{*}\}$; time horizon $t=0,\dots,T$; temperature profile $\mathcal T_t$; diffusion schedule $\tau(t)$.
}
\KwOut{
Node scores $d_{\mathrm{sem}}^{(\varepsilon,h)}(x\,|\,p)$; hyperedge weights $w_{\mathcal T_t}(e)$; effective adjacency $W_{\mathcal T_t}$; block/multi Laplacians $\{\mathcal L_{\mathcal T_t}^{(*)}\},\ \mathcal L_{\mathcal T_t}^{\mathrm{multi}}$; spectra $\{\lambda_i(t),u_i(t)\}$; contrasts $c_{x,\mathcal K}(t)$; hallucination energy bounds for $\mathcal E_{\mathrm{hall}}^{\mathrm{multi}}(x,p,\cdot)$.
}

\textbf{Phase I: per-prompt semantic score (Eq.~\eqref{eq:KL1}).}\;
\Indp
\textbf{1.} Estimate $f_p$ from model samples (density or histogram on $\mathcal X$ under $\mu$).\;
\textbf{2.} Form $\tilde f_{p,\varepsilon}(x)=(1-\varepsilon)f_p(x)+\varepsilon\rho(x)$ and
$\tilde f_{p,\varepsilon}^{\mathcal K}(x_2)=\mathbf 1_{\{x_2\in\mathcal K\}}\tilde f_{p,\varepsilon}(x_2)\big/ \int_{\mathcal K}\tilde f_{p,\varepsilon}d\mu$.\;
\textbf{3.} Compute $(T_h\tilde f_{p,\varepsilon})(x_1)=\int K_h(x_1,x_2)\tilde f_{p,\varepsilon}(x_2)\,d\mu(x_2)$ and
$(T_h\tilde f_{p,\varepsilon}^{\mathcal K})(x_1)$; evaluate
$d_{\mathrm{sem}}^{(\varepsilon,h)}(x\,|\,p)=\big[\log(T_h\tilde f_{p,\varepsilon}^{\mathcal K}(\Pi_{\mathcal K}(x)))-\log(T_h\tilde f_{p,\varepsilon}(x))\big]_{+}$.\;
\Indm

\textbf{Phase II: hyperedges, weights, and Laplacian blocks (Eqs.~\eqref{eq:weight2}–\eqref{eq:weight3}, \eqref{eq:weight4}).}\;
\Indp
\textbf{4.} For each node $\mathsf v_a\!\sim\!(x_a,p)$, store $\Delta_a\!:=\!d_{\mathrm{sem}}^{(\varepsilon,h)}(x_a\,|\,p)$.\;
\textbf{5.} For each hyperedge $e=\{\mathsf v_1,\dots,\mathsf v_{r(e)}\}\in E^{(*)}$, set
$w_{\mathcal T_t}(e)=\mathbf 1_{\{e\in E^{(*)}\}}\exp\!\big(-\eta_{*}\,\frac{\sum_{a<b}|\Delta_a-\Delta_b|}{\sum_a\mathcal T_t(\mathsf v_a)}\big)$.\;
\textbf{6.} Build $W_{\mathcal T_t}^{(*)}=\mathrm{diag}\{w_{\mathcal T_t}(e)\}$, degrees
$\mathcal D_{\mathsf v,\mathcal T_t}^{(*)}$ and $\mathcal D_{e,\mathcal T_t}^{(*)}$, effective adjacency
$W_{\mathcal T_t}^{(*),\mathrm{eff}}=\mathcal I^{(*)}W_{\mathcal T_t}^{(*)}(\mathcal D_{e,\mathcal T_t}^{(*)})^{-1}(\mathcal I^{(*)})^\top$.\;
\textbf{7.} Form block Laplacians
$\mathcal L_{\mathcal T_t}^{(*)}=\mathbf I-(\mathcal D_{\mathsf v,\mathcal T_t}^{(*)})^{-1/2}W_{\mathcal T_t}^{(*),\mathrm{eff}}(\mathcal D_{\mathsf v,\mathcal T_t}^{(*)})^{-1/2}$ and aggregate
$W_{\mathcal T_t}=\sum_{*}\omega_* W_{\mathcal T_t}^{(*),\mathrm{eff}}$; assemble $\mathcal L_{\mathcal T_t}^{\mathrm{multi}}$ accordingly.\;
\Indm

\textbf{Phase III: spectral objects and contrasts (Eqs.~\eqref{eq:laplacian_spectral_decomposition}, \eqref{eq:contrast_vec}).}\;
\Indp
\textbf{8.} Compute leading spectrum of $\mathcal L_{\mathcal T_t}^{\mathrm{multi}}$: $\{\lambda_i(t),u_i(t)\}$ (e.g., LOBPCG/power iteration on sparse matrices). Ensure $\lambda_2(t)>0$ (connectedness).\;
\textbf{9.} Build degree–matched $\pi_{\mathcal K}$ and raw contrast $c^{\mathrm{raw}}_{x,\mathcal K}(t)={\mathcal D_{\mathcal T_t}^{\mathrm{multi}}}^{1/2}\big(\delta_{\mathsf v_x}-\pi_{\mathcal K}\big)$; project $c_{x,\mathcal K}(t)=(\mathbf I-u_1u_1^\top)c^{\mathrm{raw}}_{x,\mathcal K}(t)$.\;
\Indm

\textbf{Phase IV: energies and guarantees (Eqs.~\eqref{eq:sandwich1} \& \eqref{eq:CF_hall_sandwich}).}\;
\Indp
\textbf{10.} Evaluate the diffusion quadratic form
$Q_{\tau}(t)=\langle c_{x,\mathcal K}(t),\,e^{-2\tau(t)\mathcal L_{\mathcal T_t}^{\mathrm{multi}}}\,c_{x,\mathcal K}(t)\rangle$ via Krylov–exponential or spectral filter.\;
\textbf{11.} Choose empirical $m(t),M(t)$ from block coefficients/operator norms (bounds discussion) and report
\[
m(t)\,e^{-2\tau(t)\lambda_{\max}(t)}\|c_{x,\mathcal K}(t)\|^2
\ \le\ 
\mathcal E(x,p;\mathcal T_t)-\mathcal E_{\mathcal K}(x,p;\mathcal T_t)
\ \le\
M(t)\,e^{-2\tau(t)\lambda_2(t)}\|c_{x,\mathcal K}(t)\|^2.
\]
\textbf{12.} Set
$\mathcal E_{\mathrm{hall}}^{\mathrm{multi}}(x,p,\cdot)=\big(\mathcal E-\mathcal E_{\mathcal K}\big)_{+}\mathbf 1_{\{x\notin\mathcal K\}}$ and record bounds from Eq.~\eqref{eq:CF_hall_sandwich}.\;
\Indm

\textbf{Phase V: calibration and decay control (Good–Turing, KV embedding, decay).}\;
\Indp
\textbf{13.} Compute Good–Turing missing-mass $\widehat m_{\mathrm{GT}}(t)$ on $\mathcal X\setminus\mathcal K$; set
$\vartheta_{\mathrm{KV}}(t)=\xi\,\widehat m_{\mathrm{GT}}(t)$ with $\xi\in(0,1]$.\;
\textbf{14.} Update $\tau(t)$ to satisfy
$m(t)\,e^{-2\tau(t)\lambda_{\max}(t)}\|c_{x,\mathcal K}(t)\|^2\ge \vartheta_{\mathrm{KV}}(t)$ (Eq.~\eqref{eq:KV_embed}); enforce nondecreasing $\tau(t)$.\;
\textbf{15.} Monitor decay envelope
$m(t)e^{-2\tau(t)\lambda_{\max}(t)}\|c\|^2\le \mathcal E_{\mathrm{hall}}^{\mathrm{multi}}\le M(t)e^{-2\tau(t)\lambda_{2}(t)}\|c\|^2$
and stop when below a target threshold.\;
\Indm

\textbf{Implementation notes (Colab).} Sparse matrices for $\mathcal I^{(*)}$, $W_{\mathcal T_t}^{(*),\mathrm{eff}}$, and $\mathcal L_{\mathcal T_t}^{\mathrm{multi}}$; row-normalize $K_h$; stabilize logs via log-sum-exp; estimate $\lambda_2,\lambda_{\max}$ by LOBPCG/power method; compute $e^{-2\tau\mathcal L}$ via $\mathrm{expm\_multiply}$ or truncated Chebyshev; Good–Turing from frequency table on $\mathcal X\setminus\mathcal K$.
\end{algorithm}

\end{document}